\documentclass{article}

\usepackage[preprint]{neurips_2026}
\usepackage{amsmath,amssymb,amsthm}
\usepackage{graphicx}
\usepackage{booktabs}
\usepackage{tabularx}
\usepackage{algorithm}
\usepackage{algpseudocode}
\let\REQUIRE\Require
\let\STATE\State
\let\FOR\For
\let\ENDFOR\EndFor
\let\IF\If
\let\ENDIF\EndIf
\let\WHILE\While
\let\ENDWHILE\EndWhile
\let\ELSE\Else
\let\RETURN\Return
\let\COMMENT\Comment
\usepackage{enumitem}
\usepackage{comment}

\usepackage{float}
\floatstyle{ruled}
\restylefloat{algorithm}

\usepackage{tikz}
\usetikzlibrary{arrows.meta,patterns,positioning,calc}

\makeatletter
\@ifpackageloaded{natbib}{}{\usepackage[round,authoryear]{natbib}}
\makeatother
\let\cite\citep

\usepackage{hyperref}
\hypersetup{hidelinks}
\usepackage{url}
\usepackage{xcolor}

\setlist{topsep=3pt plus 1pt, itemsep=2pt plus 1pt, parsep=0pt, partopsep=0pt}

\makeatletter
\@ifundefined{bibsep}{}{\setlength{\bibsep}{1.2pt plus 0.4pt}}
\makeatother

\newtheorem{theorem}{Theorem}[section]
\newtheorem{lemma}[theorem]{Lemma}
\newtheorem{corollary}[theorem]{Corollary}
\newtheorem{proposition}[theorem]{Proposition}
\theoremstyle{definition}
\newtheorem{definition}[theorem]{Definition}
\theoremstyle{remark}

\newtheorem{observation}[theorem]{Observation}

\title{Trivium: Temporal Regret as a First-Class Objective for Causal-Memory Controllers}

\author{%
 Edward Y. Chang \\ Computer Science \\ Stanford University
}

\begin{document}

\addtocontents{toc}{\protect\setcounter{tocdepth}{-1}}

\maketitle

\begin{abstract}
Many agentic systems and LLM pipelines correct mistakes by optimizing outcome reward. This addresses only the \emph{what} of failure: the \emph{why}
and \emph{when} of a mismatch are not systematically logged, reviewed, or corrected, so the same error can recur across episodes. This is a
structural problem, not merely a model-capacity one. We propose \emph{long-horizon temporal regret} as a first-class diagnostic objective alongside
outcome regret and epistemic regret. Temporal regret records how long an
unresolved or incorrect causal model remains action-relevant; epistemic regret records truth-relative posterior error, with entropy or information
gain used operationally when the true graph is unavailable. Modeling an agent as a stream of $E$ episodes, we prove conditional identification results
under explicit intervention, persistence, separation, and detectability assumptions. Outcome-only observation cannot distinguish an observationally
equivalent causal alternative without an intervention channel. Under commit discipline and positive per-probe information, delayed identification has
logarithmic probe complexity in the episode horizon; under $K$ detector-visible changes, the conditional bound extends segment-wise. The corrected
empirical record is narrower. On the primary CausalBench-Seq stream, a hard structural readout records $7.8\pm2.9$ misidentified episodes over
$E=500$, with zero observed stationary errors after identification, while outcome-only controllers remain misidentified throughout. Audit ablations
show that continued posterior updating, rather than CUSUM reopening or local repair, drives posterior recovery on the tested topology changes; local
repair instead reduces committed-graph dispatch exposure from $10.1$ to $7.8$ episodes. The previous empirical logarithmic-envelope verdict based on a
clipped soft score is withdrawn. A real-LLM stream remains preliminary external-validity evidence. Self-learning here means revising an external
causal model, not retraining LLM weights: a substrate-external corrective channel that can wrap any dispatch policy, RL-trained or otherwise.
\end{abstract}

\noindent{\small\textbf{Note on version 2 (closure).} \textbf{Why this version.} This version corrects errors in v1.
The use of the term ``regret'' also caused confusion because the quantities defined here are not standard comparator-based online-learning regret.
Subsequent work will use ``liability''
or ``debt'' to be clear.
\textbf{What is corrected.} We correct the epistemic quantity and scope the formal bounds to their assumptions (Definition~\ref{def:three-regret}; Theorems~\ref{thm:within}--\ref{thm:drift}); replace the clipped soft-score interpretation with a hard structural readout and withdraw the former empirical logarithmic-envelope verdict (Table~\ref{tab:e1-numbers}; RQ2); distinguish the ideal commit discipline from the executed entropy-gated posterior snapshot (Definition~\ref{def:commit}); separate posterior recovery, committed-graph freshness, CUSUM reopening, and LRCP effects (RQ3--RQ4); and rescope unexecuted studies, analytical calibrations, and pilot evidence in the appendices. Table~\ref{tab:rq-matrix} gives the final claim status. This version is the final corrected statement of Trivium, not an expansion of its claims.}\par



\section{Introduction}
\label{sec:intro}

Long-lived agentic systems must contend with environments where the most consequential causes are often unseen: fleets that misread weather
change-points (e.g., the urban-ride-share scheduler of App.~\ref{app:urs}), supply chains that miss demand shifts, healthcare systems that miss
accumulating risk factors, social and educational systems that fail to connect early warning signs into actionable interventions. In each case,
accumulated regret over unseen confounders does not announce itself in the reward channel; an outcome-only reinforcement-learning system can drive its
training-time outcome regret to zero while, in deployment, spending long stretches acting on a wrong causal
model~\cite{kallus2020confounding,namkoong2020off,levine2020offline}. By the time the unseen cost overtakes nominal performance, the corrective window
has often closed. The reason is structural: the corrective signal, which hidden cause is being mispredicted, is invisible to the reward channel. Yann
LeCun has long argued that current AI systems cannot learn the way humans learn~\cite{lecun2022path}: when a plan fails, a human asks \emph{why}; an
outcome-only learner sees only \emph{how much} reward was lost~\cite{huang2023llms}. This paper formalizes that critique through a three-regret
objective and conditional identification results on a stream of $E$ episodes sharing a persistent causal log. An observational-equivalence separation
shows outcome-only observation cannot distinguish causally distinct worlds without an intervention channel, so temporal miscalibration persists
linearly. With budgeted causal probes and persistent logging, total probe complexity is logarithmic in $E$, inducing logarithmic
delayed-identification temporal regret under the stated accounting convention.

\paragraph{Trivium in three acts.}
We propose \emph{Trivium}, a three-regret informed update mechanism that identifies, quantifies, and traces mistakes across three layers.
\emph{Outcome regret} captures \emph{what}: the conventional reinforcement-learning signal of answer accuracy. \emph{Epistemic regret} captures
\emph{why}: residual uncertainty over the working causal model. \emph{Temporal regret} captures \emph{when}: how long a miscalibrated causal model is
tolerated before correction. Across episodes and long horizons, Trivium records these three signals in a persistent log and uses accumulated
interventional evidence to revise its working causal model. Any improvement in realized outcomes requires the additional dispatch assumptions stated
later and is not inferred from structural identification alone. Concretely, a fleet that ignores weather change-points, a supply-chain controller that
ignores supplier-side shocks, or a triage system that ignores accumulating risk factors pays the same structural cost: exposure to the unseen
confounder accumulates over the horizon until corrective evidence is logged, and the length of that accumulation window is exactly the \emph{when} our
theorems bound (App.~\ref{app:helicopter-example} sketches a compact helicopter-delivery illustration of this mechanism).

\paragraph{Why this matters.}
Treated as a diagnostic objective alongside outcome regret, \emph{temporal} regret motivates a design pattern for domains where repeated structured
action, persistent causal evidence, and an intervention channel are present: tool-using LLM agents~\cite{jin2023cladder,kiciman2023causal},
multi-agent software-engineering loops, hospital-triage systems, and supply-chain controllers among others. Such systems can record not only \emph{how
much} reward was lost, but what causal uncertainty remains and how long an unresolved model is used. Whether exposure can genuinely \emph{compound},
through dependent actions, inherited memories, or downstream agents, is not measured here and is left to future work.

\paragraph{Contributions.}

\begin{enumerate}[leftmargin=1.5em,itemsep=-2pt,topsep=0pt]
\item \textbf{Three-regret functional and outcome-only separation.} A three-regret functional combining outcome, epistemic, and temporal diagnostic
objectives (Definition~\ref{def:three-regret}), paired with an observational-equivalence separation: outcome-only observation cannot distinguish
causally distinct models with identical observational distributions, allowing temporal miscalibration to persist linearly without an intervention
channel (Appendix~\ref{app:separation}).

\item \textbf{Conditional causal-probe complexity and delayed-identification bounds.}
Within an episode, exact null-versus-separated classification over $m=|G^{\mathrm{inf}}|$ candidates requires
$\mathcal{O}(m\sigma^2\kappa_{\min}^{-2}\log T)$ probes under the ideal commit discipline (Theorem~\ref{thm:within}). Across episodes, persistent
evidence gives a logarithmic-in-$E$ total-probe upper bound for informative components under a positive per-probe information condition
(Theorem~\ref{thm:ce-upper}); Proposition~\ref{thm:ce-lower} supplies an adapted horizon-order lower bound with explicitly different structural
factors. The drift extension is conditional on detector-visible changes (Theorem~\ref{thm:drift}). These are formal bounds, not a renewed empirical
fit of the withdrawn clipped soft score.

\item \textbf{Algorithmic instantiation with an implementation audit.} Trivium combines a persistent causal transaction log, budgeted probes, an
entropy-gated posterior snapshot, CUSUM-triggered reopening, and LRCP local repair. The ideal frozen-commit rule and the executed snapshot rule are
stated separately. The LRCP contraction proposition remains conditional, because the repair-disabled ablation shows that the observed posterior
contraction is produced by continued updating, not by LRCP. Direct committed-graph logging instead identifies LRCP's measured benefit as dispatch
freshness, reducing dispatch exposure from $10.1$ to $7.8$ episodes. The dispatch and grounding results remain conditional bridges under their
explicit assumptions.

\item \textbf{Controlled tests with explicit positive and negative results.} RQ1 corroborates the outcome-only observational-equivalence separation.
RQ2 records bounded finite-horizon hard exposure with no observed stationary growth, but does not estimate an asymptotic rate from a single horizon.
RQ3 observes rapid topology recovery while falsifying the designed CUSUM reopen mechanism on that change class; a separate variance stream shows that
a squared-residual detector can see a change while the mean-residual repair signal stays at its baseline rate. RQ4 falsifies the attribution of posterior contraction to LRCP
and identifies dispatch freshness as its measured role. RQ5 remains a preliminary real-LLM pilot under the \textsc{CAP-GSM8K}
protocol~\citep{ChangACL2026,cobbe2021gsm8k}.
\end{enumerate}

\noindent
Each of the four contributions is backed by a theorem (or definition / proposition / lemma / corollary) and a stated experiment or ablation; the full mapping appears in Table~\ref{tab:traceability} (App.~\ref{app:traceability}).




\section{Related Work}
\label{sec:related}

Trivium connects four literatures, each expanded in a parallel subsection of Appendix~\ref{app:related-ext}. First, offline and confounded RL show
that outcome-only learning is limited under unobserved confounding (the ``lacking-learning'' critique~\cite{lecun2022path}); our
Appendix~\ref{app:separation} separation is complementary, using observationally equivalent SCMs whose passive outcome streams match but whose
interventional distributions differ. Second, self-improving LLM methods, LLM causal-reasoning benchmarks, and recent object-centric world models such
as Causal-JEPA~\cite{nam2026causaljepa} show the value of model-internal critique and latent causal inductive biases; Trivium is complementary,
placing the corrective signal in a substrate-external causal log and asking when accumulated evidence should revise future policy. Third, causal
bandits study intervention-budgeted identification mostly in single-episode settings; Trivium extends this to episode streams with persistent causal
evidence and a total-probe lower bound (Audibert--Bübeck-style~\cite{audibert2010best}). Fourth, classical identification
(do-calculus~\cite{pearl2009causality}), CUSUM change-point detection, transactional memory, and belief revision supply the primitives used by the CTL
and the drift-recovery analysis. Appendix~\ref{app:rw-llm} additionally discusses concurrent anonymized work on the orthogonal within-episode
causal-critique side of the problem (\emph{why} within an episode; Trivium addresses \emph{when} across an episode stream).



\section{Theoretical Framework}
\label{sec:theory}

Figure~\ref{fig:atrium-arch} depicts the end-to-end architecture; symbols are defined where first introduced and collected in Table~\ref{tab:notation}
(App.~\ref{app:notation}). We formalize long-horizon scheduling as a causally-structured episodic decision process, define a three-regret functional
over outcome, epistemic, and temporal diagnostic objectives, and state the convergence and drift-recovery guarantees that the experiments of
\S\ref{sec:experiments} test. Proofs are in Appendices~\ref{app:proofs}--\ref{app:separation}; positioning against adjacent literatures (causal
bandits, change-point detection, identification-under-interventions, transactional memory) is in Appendix~\ref{app:background-ext}.


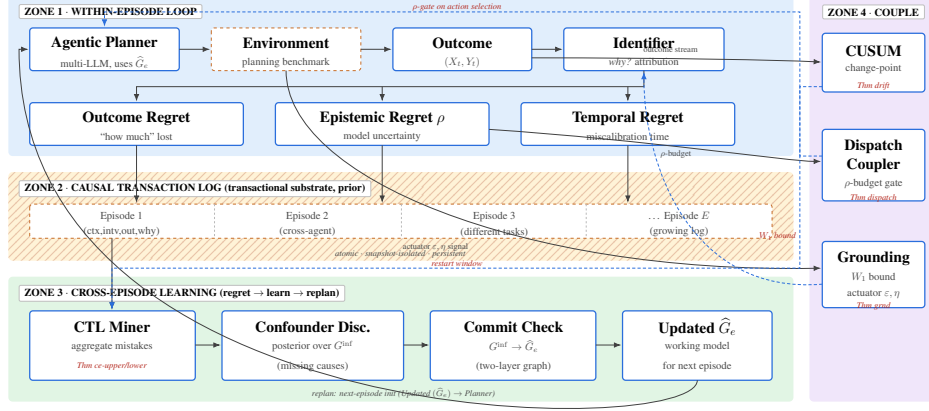
\begin{figure}[!tbp]
\centering
\definecolor{atrmAccent}{RGB}{42,110,220}        
\definecolor{atrmSubstrate}{RGB}{200,110,40}     
\definecolor{atrmThmRed}{RGB}{180,40,30}         
\definecolor{atrmSubT}{RGB}{75,75,75}            
\definecolor{zone1Bg}{RGB}{225,238,252}          
\definecolor{zone2Bg}{RGB}{255,243,215}          
\definecolor{zone3Bg}{RGB}{225,245,228}          
\definecolor{zone4Bg}{RGB}{240,230,250}          
\resizebox{0.88\textwidth}{!}{%
\begin{tikzpicture}[
 x=1pt, y=-1pt,
 >=Latex,
 ours/.style = {draw=atrmAccent, line width=1.2pt, rounded corners=3pt, fill=white},
 substrate/.style= {draw=atrmSubstrate, line width=0.9pt, dashed, rounded corners=3pt, fill=white},
 zone1bg/.style = {fill=zone1Bg, rounded corners=4pt, draw=none},
 zone2bg/.style = {fill=zone2Bg, rounded corners=4pt, draw=none,
 postaction={pattern=north east lines, pattern color=atrmSubstrate!40}},
 zone3bg/.style = {fill=zone3Bg, rounded corners=4pt, draw=none},
 zone4bg/.style = {fill=zone4Bg, rounded corners=4pt, draw=none},
 banner/.style = {fill=white, draw=gray!50, line width=0.2pt, rounded corners=1.5pt,
 inner xsep=5pt, inner ysep=3pt},
 flow/.style = {->, >=Latex, line width=0.9pt, shorten >=1pt, draw=black!75},
 gate/.style = {->, >=Latex, line width=0.9pt, densely dashed,
 draw=atrmAccent, shorten >=1pt},
 edivr/.style = {gray!55, dash pattern=on 2pt off 2pt},
 Tboxtitle/.style= {font=\Large\bfseries},
 Tboxsub/.style = {font=\normalsize, text=atrmSubT},
 Tbanner/.style = {font=\normalsize\bfseries},
 Tthm/.style = {font=\footnotesize\itshape, text=atrmThmRed},
 Tlbl/.style = {font=\footnotesize, text=atrmSubT},
]

\fill[zone1bg] (20,30) rectangle (755,180);
\node[banner, anchor=north west] at (30,38)
 {\normalsize\bfseries ZONE 1 $\cdot$ WITHIN-EPISODE LOOP};

\draw[ours] (40,60) rectangle (180,100);
\node[Tboxtitle] at (110,74) {Agentic Planner};
\node[Tboxsub] at (110,92) {multi-LLM, uses \(\widehat G_e\)};

\draw[substrate] (210,60) rectangle (350,100);
\node[Tboxtitle] at (280,74) {Environment};
\node[Tboxsub] at (280,92) {planning benchmark};

\draw[ours] (380,60) rectangle (510,100);
\node[Tboxtitle] at (445,74) {Outcome};
\node[Tboxsub] at (445,92) {$(X_t,Y_t)$};

\draw[ours] (540,60) rectangle (690,100);
\node[Tboxtitle] at (615,74) {Identifier};
\node[Tboxsub] at (615,92) {\emph{why?} attribution};

\draw[flow] (180,80) -- (210,80);
\draw[flow] (350,80) -- (380,80);
\draw[flow] (510,80) -- (540,80);

\draw[ours] (40,130) rectangle (240,170);
\node[Tboxtitle] at (140,145) {Outcome Regret};
\node[Tboxsub] at (140,162) {``how much'' lost};

\draw[ours] (270,130) rectangle (470,170);
\node[Tboxtitle] at (370,145) {Epistemic Regret $\rho$};
\node[Tboxsub] at (370,162) {model uncertainty};

\draw[ours] (500,130) rectangle (700,170);
\node[Tboxtitle] at (600,145) {Temporal Regret};
\node[Tboxsub] at (600,162) {miscalibration time};

\draw[flow] (615,100) -- (615,115) -- (140,115) -- (140,130);
\draw[flow] (370,115) -- (370,130);
\draw[flow] (600,115) -- (600,130);

\fill[zone2bg] (20,195) rectangle (755,278);
\node[banner, anchor=north west] at (30,203)
 {\normalsize\bfseries ZONE 2 $\cdot$ CAUSAL TRANSACTION LOG (transactional substrate, prior)};

\draw[substrate] (40,225) rectangle (735,257);
\draw[edivr] (213,225) -- (213,257);
\draw[edivr] (388,225) -- (388,257);
\draw[edivr] (561,225) -- (561,257);
\node[Tboxsub] at (126,237) {Episode 1};
\node[Tboxsub] at (126,251) {(ctx,intv,out,why)};
\node[Tboxsub] at (300,237) {Episode 2};
\node[Tboxsub] at (300,251) {(cross-agent)};
\node[Tboxsub] at (475,237) {Episode 3};
\node[Tboxsub] at (475,251) {(different tasks)};
\node[Tboxsub] at (648,237) {\ldots\ Episode $E$};
\node[Tboxsub] at (648,251) {(growing log)};
\node[Tboxsub, font=\footnotesize\itshape] at (387,272)
 {atomic $\cdot$ snapshot-isolated $\cdot$ persistent};

\draw[flow] (140,170) -- (140,225);
\draw[flow] (370,170) -- (370,225);
\draw[flow] (600,170) -- (600,225);

\fill[zone3bg] (20,293) rectangle (755,410);
\node[banner, anchor=north west] at (30,301)
 {\normalsize\bfseries ZONE 3 $\cdot$ CROSS-EPISODE LEARNING (regret $\to$ learn $\to$ replan)};

\draw[ours] (40,325) rectangle (195,390);
\node[Tboxtitle] at (117,340) {CTL Miner};
\node[Tboxsub] at (117,358) {aggregate mistakes};
\node[Tthm] at (117,377) {Thm ce-upper/lower};

\draw[ours] (220,325) rectangle (390,390);
\node[Tboxtitle] at (305,340) {Confounder Disc.};
\node[Tboxsub] at (305,358) {posterior over $G^{\mathrm{inf}}$};
\node[Tboxsub] at (305,377) {(missing causes)};

\draw[ours] (415,325) rectangle (570,390);
\node[Tboxtitle] at (493,340) {Commit Check};
\node[Tboxsub] at (493,358) {\(G^{\mathrm{inf}}\to \widehat G_e\)};
\node[Tboxsub] at (493,377) {(two-layer graph)};

\draw[ours] (595,325) rectangle (735,390);
\node[Tboxtitle] at (665,340) {Updated \(\widehat G_e\)};
\node[Tboxsub] at (665,358) {working model};
\node[Tboxsub] at (665,377) {for next episode};

\draw[flow] (195,357) -- (220,357);
\draw[flow] (390,357) -- (415,357);
\draw[flow] (570,357) -- (595,357);

\draw[flow] (117,257) -- (117,325);

\draw[flow] (665,390) .. controls (665,445) and (30,445) .. (30,80) -- (40,80);
\node[Tlbl, font=\footnotesize\itshape] at (390,402)
 {replan: next-episode init (Updated $(\widehat G_e) \to$ Planner)};

\fill[zone4bg] (770,30) rectangle (890,410);
\node[banner] at (830,45) {\normalsize\bfseries ZONE 4 $\cdot$ COUPLE};

\draw[ours] (782,65) rectangle (878,120);
\node[Tboxtitle] at (830,81) {CUSUM};
\node[Tboxsub] at (830,98) {change-point};
\node[Tthm] at (830,114) {Thm drift};

\draw[ours] (782,155) rectangle (878,222);
\node[Tboxtitle] at (830,171) {Dispatch};
\node[Tboxtitle] at (830,189) {Coupler};
\node[Tboxsub] at (830,206) {$\rho$-budget gate};
\node[Tthm] at (830,219) {Thm dispatch};

\draw[ours] (782,260) rectangle (878,322);
\node[Tboxtitle] at (830,276) {Grounding};
\node[Tboxsub] at (830,293) {$W_1$ bound};
\node[Tboxsub] at (830,309) {actuator $\varepsilon,\eta$};
\node[Tthm] at (830,320) {Thm grnd};

\draw[flow] (510,88) .. controls (640,88) and (740,90) .. (782,92);
\node[Tlbl] at (640,80) {outcome stream};

\draw[flow] (470,155) .. controls (620,165) and (750,185) .. (782,185);
\node[Tlbl] at (645,178) {$\rho$-budget};

\draw[flow] (280,100) .. controls (280,200) and (500,290) .. (782,285);
\node[Tlbl] at (420,265) {actuator $\varepsilon,\eta$ signal};

\draw[gate] (782,115) -- (761,115) -- (761,285) -- (117,285) -- (117,325);
\node[Tthm] at (440,280) {restart window};

\draw[gate] (782,180) -- (761,180) -- (761,45) -- (110,45) -- (110,60);
\node[Tthm] at (440,40) {$\rho$-gate on action selection};

\draw[gate] (782,300) .. controls (660,310) and (615,190) .. (615,100);
\node[Tthm] at (740,255) {$W_1$ bound};

\end{tikzpicture}%
}
\caption{\textbf{Computational structure of temporal-regret minimization.}
Solid-bordered boxes are Trivium contributions; dashed-bordered boxes are load-bearing prior-work substrate (planning environment and CTL realization). Zone~2 (hatched) is a persistent transactional \emph{Causal Transaction Log}.
Information flows \emph{regret $\to$ log $\to$ learning $\to$ replan}: Zone~1 emits three regret signals; Zone~2 logs one cross-agent entry per
episode; Zone~3 updates the two-layer graph $G^{\mathrm{inf}} \to \widehat G_e$ for the next episode (as executed, $\widehat G_e$ is an entropy-gated
posterior snapshot; see the note at Def.~\ref{def:commit}); Zone~4 optionally couples to outcome regret via CUSUM, $\rho$-budgeted dispatch, and
$W_1$-bounded grounding.}
\label{fig:atrium-arch}
\end{figure}



\subsection{Setup and Three-Regret Functional}
\label{sec:three-regret}

An \emph{episode} \(e \in \{1,\dots,E\}\) is a horizon-\(T\) dispatch
window; at time \(t\), a population of \(M\) agents observes context
\(X_t\), jointly selects actions
\(A_t=(A_t^{(1)},\dots,A_t^{(M)})\), and realizes outcome \(Y_t\).
Setup-only auxiliaries \((M,A_t)\) appear only here; subsequent rate
results use the scalar quantities collected in Table~\ref{tab:notation}.

Three graph objects are distinguished. The \emph{influence graph}
\(G^{\mathrm{inf}}\) is the candidate causal family over which epistemic
uncertainty lives. The true local interventional structure during episode
\(e\) is \(G^{\mathrm{true}}_e\subseteq G^{\mathrm{inf}}\); in stationary
sections we write \(G^{\mathrm{true}}\) for the common value. The committed
working graph at episode \(e\) is \(\widehat G_e\subseteq
G^{\mathrm{inf}}\), with identified edge weights; this is the graph
against which the dispatch policy plans once sufficient causal evidence
has accumulated.

All episodes write to a persistent \emph{Causal Transaction Log} (CTL),
the epistemic memory that makes cross-episode evidence accumulation possible. An
\emph{intervention} \(\mathrm{do}(Z=z)\) is expensive and budget-bounded
per episode at \(B\); the scheduler allocates \(B\) across nodes in
\(G^{\mathrm{inf}}\).

\paragraph{Instance parameters (used throughout).}
\(m := |G^{\mathrm{inf}}|\) is the number of candidate components tested by the exact structural decision, including null components. \(m_0 := |\{Z
\in G^{\mathrm{inf}} : |\kappa_Z| \geq \kappa_{\min}\}|\) is the informative-component cardinality. The separated-effects condition is \(\kappa_Z=0\)
or \(|\kappa_Z|\geq\kappa_{\min}>0\).
\(\sigma^2\) is the sub-Gaussian noise variance on outcomes; \(\Delta_{\max}\)
is the per-step outcome-regret gap; \(d_{\mathrm{out}}^{\max}\) is the
maximum outcome-degree in \(G^{\mathrm{inf}}\) (the bandit-like fan-out
feeding Proposition~\ref{thm:ce-lower}). For embodied deployment
(Theorem~\ref{thm:grounding}), \(L\) is the outcome-map Lipschitz constant,
\(\varepsilon\) bounds actuator error, \(\eta\) bounds observation noise,
and \(p_f\) is the CTL atomic-commit failure rate with bounded failure
magnitude \(\rho_c\). These quantities are instance properties, not
hyperparameters: the scheduler never needs to know them, but their values
appear in the bound constants.

\paragraph{Assumptions and scope.}
\label{par:assumptions}
All theorems below rest on four load-bearing assumptions. (\textbf{A1})
For every episode, the true local interventional structure
\(G^{\mathrm{true}}_e\) lies within the prescribed influence graph
\(G^{\mathrm{inf}}\). (\textbf{A2}) Outcome noise is sub-Gaussian with
variance \(\sigma^2\), and the SCM admits a local-linear approximation
around committed edges. (\textbf{A3}) The CTL satisfies atomicity,
persistence, and snapshot isolation. The asymptotic analysis assumes
vanishing commit-failure rate \(p_f=o(1/\mathrm{polylog}(E))\);
finite-sample commit failures are absorbed by Theorem~\ref{thm:grounding}'s
slack term \(\varepsilon_c\), with the SagaLLM substrate reporting
\(p_f\approx0.7\%\) empirically. (\textbf{A4}) Drift is CUSUM-detectable
on the running CE-EIG stream at threshold \(h>0\). Each
falsifier in Table~\ref{tab:rq-matrix} is designed to probe one of
A1--A4 at its load-bearing boundary.

\begin{definition}[Three-regret functional]
\label{def:three-regret}
Let \(G^{\mathrm{true}}_e\subseteq G^{\mathrm{inf}}\) denote the true
local interventional structure during episode \(e\), let
\(\widehat G_e\subseteq G^{\mathrm{inf}}\) denote the graph committed for
planning at episode \(e\), and let \(P_{e,t}\) denote the scheduler's
posterior over candidate graphs in \(G^{\mathrm{inf}}\) at within-episode
step \(t\) of episode \(e\) (with \(P_e:=P_{e,0}\) at the episode start).
In stationary settings, \(G^{\mathrm{true}}_e=G^{\mathrm{true}}\) for all
\(e\); under drift, the episode index carries the changing structure.
Let \(\pi^\star\) be optimal under full knowledge of the relevant
\(G^{\mathrm{true}}_e\), and let \(\pi_t\) be the deployed policy. We
define three diagnostic regret objectives, one per failure mode of a
long-lived agent:
\[
R_{\mathrm{out}}(E,T)
:=
\sum_{e=1}^{E}\sum_{t=1}^{T}
\mathbb{E}\!\left[
Y_{e,t}^{\pi^\star}-Y_{e,t}^{\pi_t}\mid G^{\mathrm{true}}_e
\right],
\]
\[
R_{\mathrm{epi}}(E,T)
:=
\sum_{e=1}^{E}\sum_{t=1}^{T}
D_{\mathrm{KL}}\!\left(
\delta_{G^{\mathrm{true}}_e}\,\|\,P_{e,t}
\right),
\]
\[
R_{\mathrm{temp}}(E,T)
:=
\sum_{e\in \mathcal{M}}
\Delta_e\,\mathbb{E}\!\left[\min\{\tau_e,T\}\right],
\]
where
\[
\mathcal{M}:=\{e:\widehat G_e\not\equiv G^{\mathrm{true}}_e
\text{ on the subgraph induced by the task-relevant variables of episode }e\}
\]
denotes the miscalibrated-episode set, \(\Delta_e\) is the per-step
regret gap on episode \(e\), and \(\tau_e\) is the stopping time at which
the scheduler's posterior commits to the correct structure for episode
\(e\). Two roles of this quantity must not share a symbol. The truth-relative log-loss below is the \emph{evaluation} quantity, computable only with
$G^{\mathrm{true}}$ and used for analysis; the \emph{operational} quantity accumulated by the executed scheduler (Algorithm~\ref{alg:trivium}) is
posterior confidence relative to the \emph{committed} graph, $-\log P_{e,t}(\widehat G_e)$, a proxy that can be small while the commitment is wrong.
Where the running tally $\mathcal{R}^{\mathrm{epi}}_{e,t}$ gates dispatch or probing, it is this operational proxy, written
$\widehat{\mathcal{R}}^{\mathrm{epi}}_{e,t}$ hereafter. Equivalently, whenever \(P_{e,t}(G^{\mathrm{true}}_e)>0\), the
epistemic evaluation term per step is
\[
D_{\mathrm{KL}}\!\left(
\delta_{G^{\mathrm{true}}_e}\,\|\,P_{e,t}
\right)
=
-\log P_{e,t}(G^{\mathrm{true}}_e).
\]
When \(G^{\mathrm{true}}_e\) is not observable during deployment, posterior
entropy \(H(P_{e,t})\) or cross-episode expected information gain
\(\mathrm{CE\mbox{-}EIG}\) serves as the operational proxy for this
epistemic uncertainty.

Outcome regret captures what was lost, epistemic regret captures residual
uncertainty over the working causal model, and temporal regret captures
how long the scheduler acts on a miscalibrated model before correction.
The three are independently defined diagnostic objectives, not components
of a single additive scalar; we use the informal shorthand
\(R_{\mathrm{total}}:=R_{\mathrm{out}}+R_{\mathrm{epi}}+R_{\mathrm{temp}}\)
only when discussing them jointly as a composite failure-mode profile for
any deployed policy \(\pi_t\), posterior trajectory \(P_{e,t}\), and
committed working graph \(\widehat G_e\subseteq G^{\mathrm{inf}}\).
\end{definition}

The failure mode outcome-only RL cannot see is now explicit:
\(R_{\mathrm{out}}\to 0\) on training can still leave
\(R_{\mathrm{temp}}\) unbounded in deployment because
\(R_{\mathrm{temp}}\) tracks identification delay, not reward loss.
Appendix~\ref{app:separation} proves an observational-equivalence separation:
for a pair of causally distinct SCMs with identical observational transcripts,
every outcome-only learner incurs \(\Omega(T)\) worst-case temporal
miscalibration, whereas a budgeted causal probe identifies the correct
structure in \(O((\sigma^2+\kappa^2)\kappa^{-2}\log T)\) samples.

\subsection{Within-Episode and Cross-Episode Bounds}
\label{sec:cross-episode}

\begin{definition}[Commit discipline]
\label{def:commit}
The scheduler operates under \emph{commit discipline} if it commits an edge, node, or local block to \(\widehat G_e\) only after the relevant EIG drops below \(\log T/T\) and the empirical effect exceeds \(\kappa_{\min}\).
\end{definition}

\noindent\emph{As-executed note (audit).} The released implementation realizes a weaker discipline than this definition and than the frozen-commitment
description used elsewhere: its commit check re-snapshots the posterior MAP whenever posterior entropy falls below a fixed bound ($0.8$), every
episode. The committed graph is therefore an entropy-gated posterior snapshot rather than a frozen object, no absorbing committed state is reachable
under the executed code, and every theorem-to-experiment mapping in this paper should be read against this executed contract. The idealized
frozen-commitment contract remains the stated design and is not what produced the reported numbers.

\begin{theorem}[Within-episode causal-probe complexity and induced temporal regret]
\label{thm:within}
Under commit discipline, sub-Gaussian outcome noise with variance proxy $\sigma^2$, and the separated-effects condition that every one of the $m=|G^{\mathrm{inf}}|$ candidate components satisfies $\kappa_Z = 0$ or $|\kappa_Z|\ge \kappa_{\min}$, there exists a universal constant $c$ such that
\(
 n_Z \ge c\,\sigma^2\kappa_{\min}^{-2}\log(mT/\delta)
\)
causal probes per candidate component suffice. Classifying by the rule $|\hat\kappa_Z|\geq\kappa_{\min}/2$ and union-bounding over all $m$ candidates
commits the correct present/absent decision for every candidate with probability at least $1-\delta$. Thus $N_{\mathrm{probe}}^{\mathrm{within}} =
\mathcal{O}(m\sigma^2\kappa_{\min}^{-2}\log(mT/\delta))$. Under conservative calendar-time accounting, this induces $R_{\mathrm{temp}}(T) =
\mathcal{O}(\Delta_{\max}m^2\sigma^2\kappa_{\min}^{-2}\log(mT/\delta))$.
\end{theorem}

Define the cross-episode EIG potential $\mathrm{CE\text{-}EIG}(e) = \mathbb{E}[H(P_e) - H(P_{e+1}) \mid \mathcal{E}_e]$.

\begin{theorem}[Cross-episode causal-probe upper bound]
\label{thm:ce-upper}
Assume a stationary window, a persistent CTL, and a positive per-probe information condition: every active causal probe of an unresolved informative
confounder yields expected information gain at least $K_0 > 0$ after normalization by $\kappa_{\min}^2/\sigma^2$. Then, with probability at least
$1-\delta$, all $m_0$ informative confounders are identified after $N_{\mathrm{probe}}^{\mathrm{CE}} = \mathcal{O}(m_0\sigma^2 / (K_0\kappa_{\min}^2)
\cdot \log(m_0/\delta))$ total causal probes. Setting $\delta = 1/E$ gives logarithmic dependence on the episode horizon. If the effective
causal-probe budget is $B_{\mathrm{eff}}$ probes per episode, the number of pre-commit episodes is
$\mathcal{O}(N_{\mathrm{probe}}^{\mathrm{CE}}/B_{\mathrm{eff}})$, yielding delayed-identification temporal regret of the same logarithmic order under
bounded per-episode gap.
\end{theorem}

\noindent\emph{Scope of exact-structure language.} Theorem~\ref{thm:ce-upper} identifies the $m_0$ informative components under its
positive-information condition. Unlike Theorem~\ref{thm:within}, it does not by itself prove simultaneous rejection of every null candidate.
Exact-structure recovery in Table~\ref{tab:e1-numbers} is therefore an empirical hard-readout result, not a consequence of this theorem alone.

$K_0 \in (0,1]$ is a dimensionless coverage/posterior-mass constant, analogous to a gap condition in best-arm identification~\cite{audibert2010best};
Lemma~\ref{lem:k0-lower} (App.~\ref{app:ce-proof}) gives a sufficient local-Gaussian condition with per-probe information $\geq K_0 \kappa_{\min}^2 /
(2\sigma^2)$ (the factor $2$ from the Gaussian KL is absorbed into the dimensionless constant), yielding the instance-explicit rate
$\mathcal{O}(\Delta_{\max} m_0^2 \sigma^2 \log E / (B_{\mathrm{eff}} K_0 \kappa_{\min}^2))$; Appendix~\ref{sec:results-a8} reports budget sensitivity
but does not estimate $K_0$ or verify this information-rate condition.

\begin{proposition}[Cross-episode total-probe lower bound, adapted]
\label{thm:ce-lower}
For any scheduler and any $\delta \in (0, 1/2)$, there exists an instance with candidate confounder set $G^{\mathrm{inf}}$ and maximum outcome fan-out $d_{\mathrm{out}}^{\max}$ such that identifying the informative causal structure with probability at least $1-\delta$ requires at least
\(
N_{\mathrm{probe}} = \Omega\!\big( |G^{\mathrm{inf}}| \log(1/\delta) / (d_{\mathrm{out}}^{\max} \log(1 + \kappa_{\max}^2/\sigma^2)) \big)
\)
total causal probes.
\end{proposition}
\noindent\emph{Scope.} We state this as a proposition adapted from standard best-arm-identification hardness (Audibert--Bubeck) rather than as a
self-contained theorem: Appendix~\ref{app:ce-lower-proof} gives the adversarial family, the per-probe information cap, and the Fano accounting, but
leaves the transcript-level change-of-measure for adaptive schedulers implicit, and its $|G^{\mathrm{inf}}|$ target is not yet aligned with the upper
bound's $m_0$ informative components. We therefore claim only the same logarithmic dependence on the confidence choice $\delta=1/E$; the proposition
is not a matching lower bound in its structural factors or adaptive-scheduler proof detail.

\begin{corollary}[Budget-to-regret accounting]
\label{cor:budget}
If the scheduler can spend $B_{\mathrm{eff}}$ effective causal probes per episode, then the total-probe upper bound of Theorem~\ref{thm:ce-upper}
converts to $\mathcal{O}(N_{\mathrm{probe}}^{\mathrm{CE}}/B_{\mathrm{eff}})$ pre-commit episodes. With bounded per-episode temporal-regret gap, this
gives the conditional delayed-identification temporal-regret accounting used by the theory.
\end{corollary}

Theorem~\ref{thm:ce-upper} and Proposition~\ref{thm:ce-lower} share a logarithmic dependence after the confidence choice $\delta=1/E$, but they differ
in structural targets and the lower-bound proof remains adapted rather than complete for adaptive schedulers. The proposition is a total-probe
statement, not a per-episode budget lower bound. Even when $d_{\mathrm{out}}^{\max}=1$, the $|G^{\mathrm{inf}}|$ versus $m_0$ mismatch remains. The
revised RQ2 therefore does not present a tight empirical envelope or estimate an asymptotic rate from the single $E=500$ hard-readout trajectory.

\subsection{Drift-Robust Replan, Dispatch Coupling, and Physical Grounding}
\label{sec:drift}

\begin{definition}[$K$-change-point stream]
\label{def:kcp}
At most $K$ \emph{CUSUM-detectable} change-points: times $e_1 < \cdots < e_K$ and threshold $h > 0$ such that the CUSUM statistic on running CE-EIG exceeds $h$ within $\mathcal{O}(\log E)$ of each $e_k$.
\end{definition}

\begin{theorem}[Drift-robust re-identification]
\label{thm:drift}
For at most $K$ CUSUM-detectable change-points with separation $I_{\min} > 0$, threshold $h = c\log E$ gives detection delay $\mathcal{O}(\log E / I_{\min})$ at false-alarm probability $\mathcal{O}(1/E)$, and applying Theorem~\ref{thm:ce-upper} per segment yields
\(
R_{\mathrm{temp}}^{\mathrm{CE}}(E) = \mathcal{O}\!\big((K{+}1)\,\Delta_{\max}m_0^2\sigma^2 \log E / (B_{\mathrm{eff}}K_0\kappa_{\min}^2) + K\,\Delta_{\max}m_0 \log E / I_{\min}\big);
\)
the first term is segment-wise re-identification cost, the second is detection-delay cost.
\end{theorem}

The $(K{+}1)$ factor is the segment count (reducing to Theorem~\ref{thm:ce-upper} at $K{=}0$). Prediction P4 tests this with $K \in \{1, 3, 5\}$.

\paragraph{From identification rate to action quality (closed-loop bridge).}
The cross-episode rate of Theorem~\ref{thm:ce-upper} is an \emph{identification} rate; two further results lift it to action quality and embodied deployment.

\begin{theorem}[Constraint-aware dispatch coupling]
\label{thm:dispatch-coupling}
Assume that, when planning against a committed graph
\(\widehat G_e\) that agrees with \(G^{\mathrm{true}}_e\) on the
task-relevant causal structure, the base dispatch policy has
episode-level oracle-graph regret
\(R_{\mathrm{oracle}}(E)=\widetilde{\mathcal{O}}(\sqrt E)\).
Under commit discipline and an epistemic-regret gate of budget $\rho$, the gated dispatch policy satisfies $\mathcal{R}_E^{\mathrm{out}} \le
R_{\mathrm{oracle}}(E) + \mathcal{O}(\rho\,N_{\mathrm{pre}})$, where $N_{\mathrm{pre}}$ is the number of pre-commit episodes. Combining with
Corollary~\ref{cor:budget} gives $\mathcal{R}_E^{\mathrm{out}} = \widetilde{\mathcal{O}}(\sqrt{E} + \rho \log E)$ whenever $N_{\mathrm{pre}} =
\mathcal{O}(\log E)$.
\end{theorem}

\begin{theorem}[Physical grounding under bounded actuator, observation, and commit error]
\label{thm:grounding}
If the CTL supports atomic commits and snapshot isolation, actuator error is bounded by $\varepsilon$, observation noise by $\eta$, the outcome map is
$L$-Lipschitz, and commit failures occur with rate $p_f$ and bounded magnitude $\rho_c$, then $W_1\!\big(P_Y, P(Y \mid \mathrm{do})\big) \leq
L(\varepsilon + \eta) + L \rho_c p_f$ ($p_f \approx 0.7\%$ empirically~\cite{ChangGeng2025sagallm}).
\end{theorem}

Corollary~\ref{cor:e2e-deploy} (App.~\ref{app:composed-transfer}) composes the two; full proofs are in Apps.~\ref{app:dispatch-coupling-proof} and~\ref{app:grounding-proof}.

\subsection{Algorithm: Trivium}
\label{sec:algorithm}

Trivium is a cross-episode meta-controller (full pseudocode in App.~\ref{app:algorithm-ext}). Each episode: (i) compute per-episode budget $B_e = m_0
\log(\max(e, w_e))$ from Corollary~\ref{cor:budget}; (ii) run the inner dispatch--intervention loop with the dispatch head \emph{gated} on a running
epistemic-regret tally (Theorem~\ref{thm:dispatch-coupling}); (iii) repair constraint violations via LRCP; (iv) update posterior and commit via
\textsc{CommitCheck}; (v) step the CUSUM statistic and, on detection, double the active-learning window (Theorem~\ref{thm:drift}) with a partial
posterior reset. Sub-routines \textsc{IGScore}, \textsc{CommitCheck}, \textsc{PartialReset}, \textsc{PosteriorUpdate}, and the per-component
theorem-by-theorem correctness map are in Appendix~\ref{app:algorithm-ext}.

\paragraph{LRCP local repair.}
\label{sec:lrcp}
LRCP (Algorithm~\ref{alg:lrcp} in App.~\ref{app:algorithm-ext}) repairs constraint violations by bounded-radius edits on the committed DAG; it is the
within-episode primitive that makes Theorem~\ref{thm:within}'s $\log T$ rate operational. Proposition~\ref{prop:lrcp-conv} gives a conditional local
contraction certificate; the post-disruption trace of Exp A.1 has geometric shape (per-seed mean $\hat\kappa = 0.866$), but the audit's
repair-disabled ablation shows that shape belongs to continued posterior updating, and LRCP's measured role is dispatch freshness rather than
posterior identification (RQ4). Corollaries~\ref{cor:lrcp-transfer}--\ref{cor:lrcp-robust}
(Apps.~\ref{app:lrcp-transfer-proof}--\ref{app:lrcp-robust-proof}) extend the same primitive to cross-domain transfer and random-disruption robustness
under projection-residual and CUSUM-detectability assumptions.

\begin{proposition}[LRCP contraction under a local repair certificate]
\label{prop:lrcp-conv}
Suppose LRCP operates in a neighborhood where each radius-$r$ repair step satisfies a local contraction certificate: $\mathbb{E}[\epsilon_{k+1} \mid
\epsilon_k] \le \beta\,\epsilon_k + \gamma$, with $\beta \in (0,1)$. Then $\epsilon_k \le \beta^k \epsilon_0 + \gamma/(1-\beta)$; if $\gamma \le
(1-\beta)/(2T)$, LRCP reaches $\epsilon_k \le 1/T$ in $\mathcal{O}(\log T)$ iterations. A graph spectral gap with bounded influence decay is one
sufficient way to certify the local contraction condition, but the certificate is the load-bearing assumption.
\end{proposition}



\section{Experiments}
\label{sec:experiments}
\label{sec:results}

\paragraph{Validation stack.}
Trivium uses a staged validation stack, classified by role. CausalBench-Seq is the primary controlled testbed; the executed entropy-gated snapshot
differs from the ideal frozen-commit assumption used by parts of the formal analysis, so the experiments are not described as fully
assumption-matched; SagaLLM~\citep{ChangGeng2025sagallm} supplies the transactional-memory substrate required by the CTL abstraction; the
\textsc{CAP-GSM8K} real-LLM stream~\citep{ChangACL2026} is an external-validity probe, not an independent confirmation of all assumptions; the
REALM-Bench~\citep{GengChang2026realm} extension is specified but not executed in this version. Concurrent anonymized within-episode causal-critique
work is complementary (\emph{why} within an episode; Trivium addresses \emph{when} across an episode stream; App.~\ref{app:rw-llm}). Deployment
examples (triage, supply chains, embodied control) are motivating domains whose full validation requires live intervention channels and
domain-specific causal logs. We report assumption-targeted tests and implementation audits, not independent end-to-end validation or
contract-identical validation of every theorem.

We evaluate structural identification and mechanism attribution on a controlled cross-episode testbed (\textsc{CausalBench-Seq}). RQ1--RQ4 organize
the main-body evidence; RQ5 reports a preliminary real-LLM pilot in App.~\ref{app:llm-bridge}. Each RQ is paired with a formal statement, a stated
falsifier, and a final empirical status (Table~\ref{tab:rq-matrix}). CausalBench-Seq is a confounded linear-Gaussian SCM stream with candidate edges
$C\to X$, $C\to Y$, and $X\to Y$. The $X\to Y$ edge is absent for episodes $0$--$149$, present for $150$--$299$, and absent again for $300$--$499$.
The executed suite uses $E=500$ episodes and 20 seeds for the headline run, with separately stated seed counts for ablations; full reporting scope is
in Appendix~\ref{app:experiments-ext}. The transactional substrate is attributed to SagaLLM~\citep{ChangGeng2025sagallm}. The present experiments do
not validate the separate physical-grounding theorem or a complete multi-agent planning deployment.

\begin{table}[!tbp]
\centering
\footnotesize
\renewcommand{\arraystretch}{1.08}
\setlength{\tabcolsep}{4pt}
\begin{tabularx}{\linewidth}{@{}p{0.07\linewidth} X p{0.31\linewidth}@{}}
\toprule
\textbf{RQ} & \textbf{Formal claim and test condition} & \textbf{Final empirical status} \\
\midrule
RQ1 & \textbf{Claim:} observational-equivalence $\Omega(T)$ separation (App.~\ref{app:separation}). \textbf{Condition:} identical observational transcripts and no intervention channel. \textbf{Falsifier:} $P(X{\to}Y)_{\mathrm{RLVR}} \!\approx\! P_{\mathrm{epi}}$. & Corroborated on the controlled instance. \\
\addlinespace[2pt]
RQ2 & \textbf{Claim:} conditional horizon-order probe bounds (Thm.~\ref{thm:ce-upper}; Prop.~\ref{thm:ce-lower}). \textbf{Condition:} persistent evidence and positive information. \textbf{Falsifier:} stationary hard exposure continues to grow. & Finite-horizon support; the asymptotic empirical fit is withdrawn. \\
\addlinespace[2pt]
RQ3 & \textbf{Claim:} detector-conditional drift bound (Thm.~\ref{thm:drift}). \textbf{Condition:} CUSUM-detectable shifts. \textbf{Falsifier:} the reopen pathway misses the tested changes. & Recovery observed; the topology-detector pathway is falsified. \\
\addlinespace[2pt]
RQ4 & \textbf{Claim:} conditional LRCP contraction (Prop.~\ref{prop:lrcp-conv}). \textbf{Condition:} a local contraction certificate. \textbf{Falsifier:} LRCP-off changes posterior recovery. & Posterior attribution falsified; a dispatch-freshness benefit is measured. \\
\bottomrule
\end{tabularx}
\caption{Claim-and-test matrix. Each row states the formal claim, its load-bearing condition and falsifier, and the final CausalBench-Seq status ($E=500$, 20 seeds).}
\label{tab:rq-matrix}
\end{table}

\begin{table}[!tbp]
\centering
\footnotesize
\begin{tabular}{@{}lcccc@{}}
\toprule
\textbf{Controller} & \textbf{init (0--49)} & \textbf{recovery (100)} & \textbf{stable (350)} & \textbf{Total hard exposure} \\
\midrule
Trivium (ours)    & $2.2$  & $5.5$   & $0.0$   & $7.8 \pm 2.9$ \\
Epistemic-reset   & $0.3$  & $0.5$   & $1.6$   & $2.5 \pm 2.0$ \\
RLVR              & $50.0$ & $100.0$ & $350.0$ & $500.0 \pm 0.0$ \\
RLVR+memory       & $50.0$ & $100.0$ & $350.0$ & $500.0 \pm 0.0$ \\
Reactive          & $50.0$ & $100.0$ & $350.0$ & $500.0 \pm 0.0$ \\
\bottomrule
\end{tabular}
\caption{E1, revised RQ2, under the hard identification readout ($\mathbf 1[\text{any edge MAP} \neq \text{truth}]$, per episode, from archived
episode-level logs; mean episodes wrong over 20 seeds, sd on the total). Trivium has zero observed stationary errors in this finite archive;
outcome-only controllers are misidentified in every episode. Phase means and total means are computed and rounded independently and therefore need not
sum exactly. The Epistemic-reset row is not probe-matched and supports the intervention channel, not a memory comparison.}
\label{tab:e1-numbers}
\end{table}

\paragraph{RQ1: Does the epistemic signal separate from the outcome signal?}
This is the \emph{necessary condition} for any three-regret functional: if outcome-only updates already recovered the true causal model, epistemic
regret would be redundant. We ask whether there is a regime in which outcome-only RL drives training-time outcome regret to zero while still
converging on a \emph{wrong} causal model. Experiment~A.0 (Appendix~\ref{app:separation}) runs a confounded linear-Gaussian SCM ($C \to X$, $C \to Y$,
no $X \to Y$) with $T = 500$ steps over $N_{\mathrm{seeds}} = 20$ seeds, under two controllers: \textsc{rlvr} (outcome-only) and \textsc{epistemic} (a
do-intervention-gated updater). RLVR converges to $P(X{\to}Y) = 0.989 \pm 0.003$ (retains the spurious edge); the Epistemic controller converges to
$P(X{\to}Y) = 0.015 \pm 0.005$ (eliminates it). Mann--Whitney $p = 9.64 \times 10^{-9}$; the effect size is decisive. The separation is real and
large, not a calibration artifact; RQ1 is corroborated on the controlled testbed.


\paragraph{RQ2: Does hard structural exposure continue to grow during stationary operation?}
Theorem~\ref{thm:ce-upper} gives a conditional logarithmic probe bound, but a single $E=500$ trajectory cannot identify an asymptotic rate. The
revised empirical question is narrower and directly measurable: after the controller identifies the current structure, does hard structural exposure
keep accumulating during stationary operation? We run five controllers over 20 seeds. Table~\ref{tab:e1-numbers} shows
$7.8\pm2.9$ wrong-model episodes for Trivium, consisting of initialization ($2.2$) and recovery around the two changes ($5.5$), with zero observed
stationary errors in the archived run. Every outcome-only controller remains misidentified in every episode. This finite-horizon result is consistent
with the conditional delayed-identification theorem but does not re-establish the withdrawn empirical $\mathcal{O}(\log E)$ fit.

The remainder of this paragraph preserves the audit trail for the previously reported soft score. That score, under which Trivium recorded
\(8.82\pm0.14\) against Epistemic-reset's $66.99$ and RLVR's $333.33$, is mean absolute belief error, $1-\mathrm{acc}$ with
$\mathrm{acc}=1-\frac{1}{m}\sum_j|p_j-y_j|$, computed on beliefs that the implementation clips to $[0.01,0.99]$; a correctly identified true edge held
at the clip boundary is charged exactly $0.01$ per edge per episode. The apparent post-commit residual rate $\varepsilon_{\mathrm{mis}}=0.010$/ep is
therefore the clip bound, not a structural-error rate: in the archived run every post-commit, non-drift record holds the true-edge beliefs at $0.99$
with standard deviation $0.0000$, and under a hard identification readout ($\mathbf 1[p\ge 0.5]$ per edge) zero stationary structural errors are
observed across the archived 20-seed run. Second, the agreement between the cumulative slope and $1-0.990$ from the model-accuracy column is not
independent confirmation; both quantities are functions of the same clipped belief vector. The conditional logarithmic theorem concerns delayed
identification; this one-horizon audit does not estimate its asymptotic rate. The $\Theta(E)$ characterization of the cumulative score describes the
belief-error metric under clipping, not hard misidentification. The separation from outcome-only control is understated rather than overstated by the
soft score: under the hard readout, outcome-only baselines carry a misidentified structure in every stationary episode (edge error $0.905$,
exact-structure error $1.000$ per episode) while Trivium carries none, so no finite ratio summarizes it.

\paragraph{Conditional identification-to-action bridge.}
Theorem~\ref{thm:dispatch-coupling} provides a conditional translation from committed-graph correctness to outcome regret under its oracle-graph
dispatch and policy-sensitivity assumptions. Direct committed-graph logging shows that LRCP reduces dispatch-graph exposure from $10.1\pm3.2$ to
$7.8\pm2.9$ episodes and reduces posterior-to-commit lag from $2.8$ to $0.1$ episodes. These are structural freshness measurements. The experiments do
not dispatch environment-changing actions from the decided graph and do not measure outcome-valued consequences, so no empirical end-to-end
action-quality claim is made.

\paragraph{The Epistemic-reset ablation, rescoped.}
The Epistemic-reset ablation uses the \emph{same} intervention channel and the \emph{same} epistemic update rule as Trivium but discards cross-episode
memory, and its cumulative score of $66.99 \pm 0.11$ against Trivium's $8.82 \pm 0.14$ was previously read as isolating the persistent CTL as the
load-bearing mechanism. Two confounds limit that reading. Under the belief-error score, most of the reset arm's total is an under-confidence charge
rather than misidentification: its beliefs sit on the correct side of every threshold (for example $0.80$ on a present edge) and are billed the
distance to the clip ceiling each episode. And the two arms differ in per-episode probe volume as well as in memory, so the comparison does not hold
evidence acquisition fixed. This ablation therefore tests the value of the intervention channel and of confidence accumulation; a memory-specific
claim requires a hard identification readout at a matched probe budget and is outside the empirical scope of this closure paper.

\paragraph{RQ3: What recovers after topology change, and is CUSUM reopening load-bearing?}
Across 20 seeds, the posterior crosses the new $0.5$ decision threshold after a median of $4$ episodes at $e=150$ and $5$ episodes at $e=300$; the
$K$-sweep shows approximately constant recovery cost per injected change. Those observations are consistent with a bounded per-change recovery cost.
They do not validate the detect--reopen mechanism of Theorem~\ref{thm:drift}. At the shipped threshold, CUSUM reopens at only 1 of 40 true
topology-change opportunities; the most sensitive tested threshold reaches 8 of 40 while producing 46 stationary false alarms. Posterior exposure is
essentially unchanged across the threshold sweep, and disabling LRCP also leaves posterior recovery unchanged. Continued posterior updating after
commitment is therefore the load-bearing recovery mechanism on this topology class. A separately configured CUSUM-style squared-residual detector
succeeds on the variance-inflation stream, detecting 20 of 20 shifts while the mean-residual repair signal stays at its baseline rate. That detector is not identical
to the CE-EIG statistic assumed by Theorem~\ref{thm:drift}. The empirical verdict is mixed: recovery is observed, the topology-change reopen pathway
is falsified, and detector usefulness is change-class dependent.

\begin{table}[!tbp]
\centering
\footnotesize
\begin{tabular}{@{}lccc@{}}
\toprule
\textbf{Scenario (disruption)} & \textbf{Baseline} & \textbf{+ CTL} & \textbf{Observed result} \\
\midrule
CausalBench-Seq toggle $e{=}150$    & RLVR: no recovery        & Trivium: $T_{\mathrm{rec}} = 4$ (med.)  & recovery observed \\
CausalBench-Seq toggle $e{=}300$    & RLVR: no recovery        & Trivium: $T_{\mathrm{rec}} = 5$ (med.)  & recovery observed \\
CausalBench-Seq $K$-sweep ($K{=}5$) & RLVR: no recovery        & Trivium: slope $0.95$/cp ($R^2{=}0.999$) & recovery observed \\
\bottomrule
\end{tabular}
\caption{E4, revised RQ3. Recovery is observed on CausalBench-Seq, but the CUSUM audit shows that the intended reopen mechanism is not load-bearing for the tested topology changes. Toggle rows use 20 seeds; the $K$-sweep row uses 10 seeds.}
\label{tab:e4-numbers}
\end{table}

\begin{figure}[!tbp]
\centering
\includegraphics[width=0.98\linewidth, height=0.16\textheight]{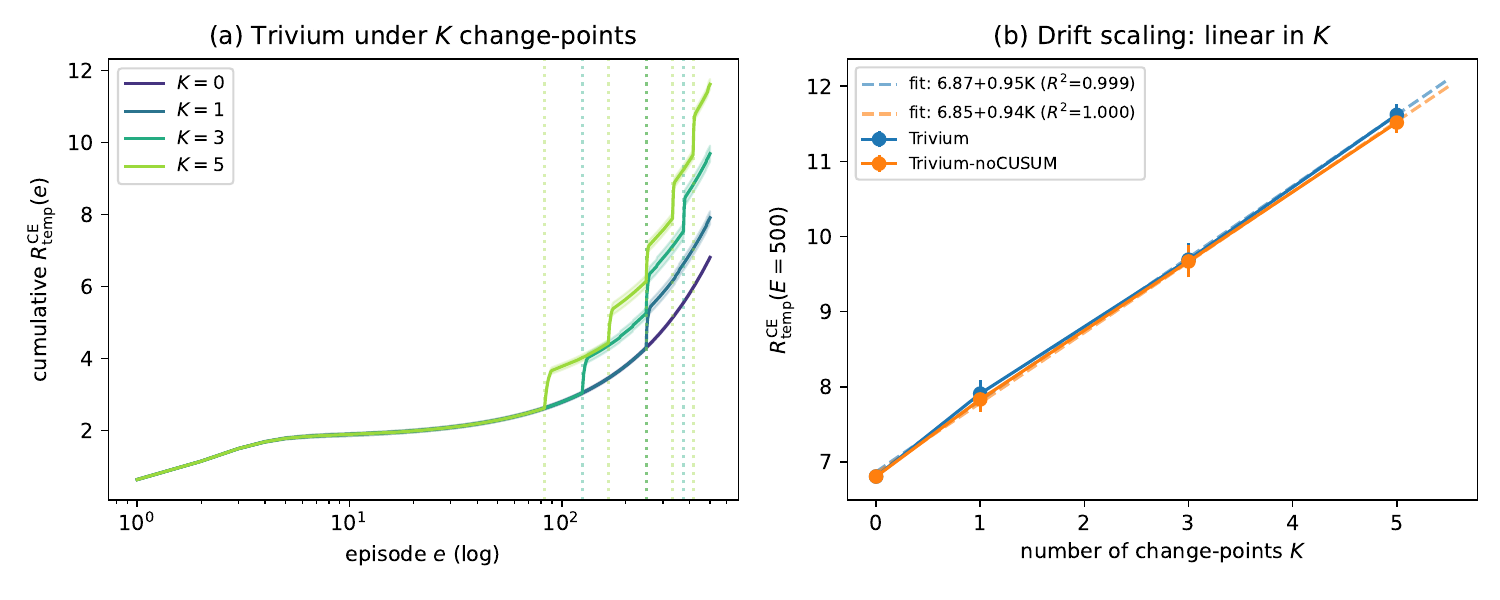}
\caption{\textbf{RQ3: drift-robust regret with $K$ change-points.} \emph{Left:} cumulative regret under $K \in \{0,1,3,5\}$ change-points injected
into CausalBench-Seq. \emph{Right:} linear-in-$K$ scaling at $E{=}500$ with slope $0.95$ per change-point ($R^2{=}0.999$), descriptively linear in $K$
over the tested range; the noCUSUM ablation overlaps, and the later threshold audit attributes topology recovery to continued posterior updating
rather than reopening (App.~\ref{sec:results-a7}).}
\label{fig:drift}
\end{figure}

\begin{figure}[!tbp]
\centering
\includegraphics[width=0.98\linewidth, height=0.15\textheight]{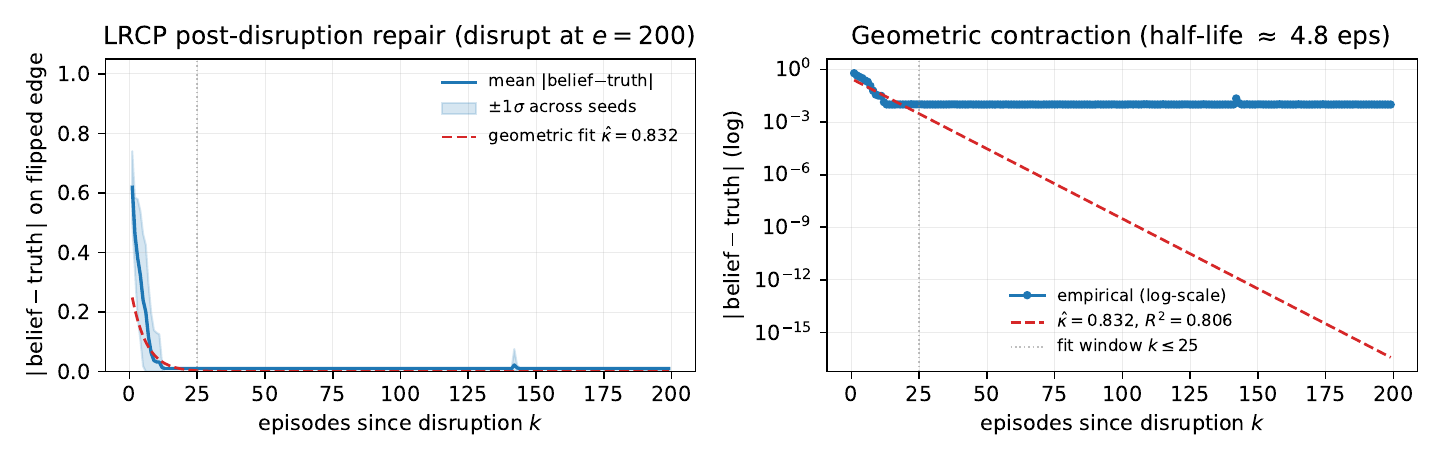}
\caption{Exp A.1, revised RQ4. Post-disruption posterior error has an approximately geometric descriptive fit ($\hat\kappa=0.832$, aggregate
$R^2=0.806$; per-seed median $\hat\kappa=0.866$). The LRCP-disabled ablation leaves this posterior trajectory unchanged, so the figure does not
corroborate LRCP's contraction certificate. It describes continued posterior updating. LRCP's measured effect appears instead in committed-graph
freshness.}
\label{fig:lrcp-contraction}
\end{figure}

\paragraph{RQ4: What system state does local repair improve?}
Proposition~\ref{prop:lrcp-conv} remains a conditional statement about an LRCP residual when a local contraction certificate holds. The original
experiment instead measured posterior edge error after a topology flip, which need not be the LRCP residual. We therefore ask whether disabling LRCP
changes posterior identification or the freshness of the committed dispatch graph. Figure~\ref{fig:lrcp-contraction} presents the trajectory together
with the OLS geometric-decay fit: over the post-disruption window ($e \in [200, 260]$), aggregate $\hat\kappa = 0.832$ ($R^2 = 0.806$) and per-seed
$\hat\kappa = 0.866 \pm 0.045$, well below $1$, with a per-seed-median half-life of $\approx 4.8$ episodes (aggregate $3.8$). The post-submission
audit shows this trace cannot be credited to LRCP: disabling LRCP entirely removes 4,317 repair calls per seed without changing posterior recovery
(all seeds still recover, at unchanged latency), so the geometric shape belongs to continued posterior updating, and Proposition~\ref{prop:lrcp-conv}
is retained only as a conditional contraction statement whose certificate this experiment does not isolate. Logging the committed graph directly
identifies LRCP's actual, measured role: it shortens the interval in which the committed dispatch graph trails the already-recovered posterior,
reducing dispatch-graph exposure from $10.1 \pm 3.2$ to $7.8 \pm 2.9$ episodes per seed (commit lag $2.8$ to $0.1$ episodes; recovery latency median
$4$ to $2$--$2.5$). LRCP is a dispatch-freshness mechanism, not a posterior-identification engine; Corollaries~\ref{cor:lrcp-transfer}
and~\ref{cor:lrcp-robust} inherit the conditional status of the proposition.

\paragraph{Robustness beyond the linear-Gaussian regime.}
Theorem~\ref{thm:ce-upper} and Proposition~\ref{prop:lrcp-conv} rest on a local-linear SCM approximation. Ablation~A-NL (App.~\ref{sec:results-anl})
stress-tests this with tanh-coupled SCMs at nonlinearity strength $\nu \in \{0, 0.5, 1, 2\}$: Trivium's log-fit $R^2 \in [0.94, 0.99]$ across all four
levels, with $\Theta(E)$ separation against RLVR at $16\text{--}25\times$. Ablation~A-JSSP (App.~\ref{sec:results-ajssp}) tests a different topology,
a $M{=}5$-machine confounded scheduling SCM with discrete actions and a makespan-style outcome (a job-shop-flavored stream); on this stream Trivium's
log-fit $R^2 = 0.992$ versus its linear-fit $R^2 = 0.887$, while RLVR's linear-fit $R^2 = 0.9999$ confirms the $\Theta(E)$ failure mode, with a
$7.3\times$ separation at $E = 500$. These auxiliary stress tests produce log-shaped cumulative trajectories outside the primary three-variable
stream. They are qualitative shape checks, not a replacement for the withdrawn primary soft-score fit and not a proof of an asymptotic law under the
full executed controller.

\paragraph{RQ5: Does the diagnostic signal transfer to a real-LLM pilot?}
Under the \textsc{CAP-GSM8K} adversarial-hint protocol of~\citet{ChangACL2026} layered on \textsc{GSM8K}~\citep{cobbe2021gsm8k}, Trivium provides a
preliminary external-validity pilot beyond the controlled CausalBench-Seq setting. Across the completed model-family runs, one full $E\!=\!500$ run on
Llama-3.3-70b and three $E\!=\!100$ pilot runs on GPT-4o, Claude-Sonnet-4.5, and GPT-3.5-Turbo, Trivium attains a $103\times$ peak reduction at
$E\!=\!500$ and an $\approx\!24\times$ geometric-mean reduction across completed runs (Table~\ref{tab:rq5-mini}). The small number of model families,
unequal horizons, and one full-length run prevent an asymptotic or general model-family conclusion. Full protocol, model-version notes, and
per-episode evidence are in App.~\ref{app:llm-bridge}.

\begin{table}[!ht]
\centering
{\footnotesize
\setlength{\tabcolsep}{6pt}
\begin{tabular}{@{}lccccc@{}}
\toprule
\textbf{Model} & $\boldsymbol{E}$ & \textbf{seeds} & \textbf{RLVR $R_{\mathrm{temp}}^{\mathrm{CE}}$} & \textbf{Trivium $R_{\mathrm{temp}}^{\mathrm{CE}}$} & \textbf{Reduction} \\
\midrule
Llama-3.3-70b      & $500$ & $10$ & $248.8$ & $2.4$ & $103\times$ \\
GPT-4o             & $100$ & $3$  & $58.0$  & $3.0$ & $19\times$  \\
Claude-Sonnet-4.5  & $100$ & $3$  & $21.3$  & $2.3$ & $9\times$   \\
GPT-3.5-Turbo      & $100$ & $3$  & $68.0$  & $3.7$ & $18\times$  \\
\bottomrule
\end{tabular}
\caption{RQ5: cumulative temporal regret on \textsc{CAP-GSM8K} (App.~\ref{app:llm-bridge}); lower is better.}
\label{tab:rq5-mini}}
\end{table}




\section{Conclusion, Impact, and Limitations}
\label{sec:discussion}\label{sec:limitations}

\paragraph{Closure.} Trivium establishes that explicit interventions and persistent causal evidence can bound delayed structural identification under
stated assumptions. Its original evaluation did not isolate a memory-specific advantage, and its detector and repair mechanisms act on different
system states than the posterior metric initially used to evaluate them. The corrected implementation audit shows that local repair improves dispatch
freshness rather than posterior identification, while CUSUM reopening is not load-bearing on the tested topology changes. These boundaries close the
claims of this paper.

\emph{Temporal regret} promotes a previously-implicit dimension of long-horizon
learning to a first-class objective. \emph{Self-learning} acquires a precise,
falsifiable meaning: the controller asks not only \emph{how much} reward was
lost but \emph{why} an outcome diverged from prediction, and the time spent on
the wrong answer is itself penalized. The path from identification rate to
action quality to downstream deployment is stated as a conditional bridge,
with the experiments here measuring posterior and committed-graph structure. The
\textsc{CAP-GSM8K} real-LLM stream (App.~\ref{app:llm-bridge}) is an
external-validity pilot rather than the main rate test: it suggests the
temporal-regret signal transfers beyond the controlled SCM setting, while
broader benchmark coverage and live intervention channels remain necessary
next steps.

\textbf{Impact.}\enspace We see this as a step toward explicit System-2-style
control for long-lived agentic systems: a learner that asks not only
\emph{what} went wrong, but \emph{why} and \emph{when}. The architectural
option this opens, a scheduler above the dispatch policy operating on a
substrate-external causal log, applies wherever an episode is structured,
an intervention channel is available, and outcomes can be re-attributed:
tool-using LLM agents, hospital triage, multi-agent SWE loops, and
supply-chain control. Evaluation correspondingly shifts from ``did outcome
reward improve'' to ``did the system identify why it failed, and how
fast'', targeting structural correction rather than reactive
outcome repair.

\subsection*{Limitations and Future Work}

As a Concept \& Feasibility contribution, this paper focuses on establishing
the temporal-regret objective, proving separation and rate results, and
validating feasibility through stated falsifiers. Broader live deployment,
deferred Gemini-2.5-Flash replication and a $\kappa_{\min}$-sweep
testing the predicted effect-to-noise dependence
(Lemma~\ref{lem:k0-lower}) are natural next steps.

\emph{Model-class and graph scope.}\enspace Bounds operate on a two-layer
graph (App.~\ref{app:dag}) with prescribed $G^{\mathrm{inf}}$, local-linear
SCM, and CUSUM-detectable change-points; confounders outside
$G^{\mathrm{inf}}$ are not identified. Extensions to hierarchical structure
learning, kernel/normalizing-flow LRCP, variance-aware change-point detectors,
and online structural-surprise detection broaden the reach
(App.~\ref{app:limitations-ext}).

\emph{Substrate reliability and LLM--graph alignment.}\enspace
Theorem~\ref{thm:grounding}'s $\varepsilon_c$ slack covers stochastic commit
failure (App.~\ref{app:warehouse}); Byzantine-robust commit strengthens it
under adversarial CTL entries. Per-agent speedup saturates at $M = m_0$
(Prop.~\ref{prop:multi-agent}); $M \gg m_0$ requires coalition design
(App.~\ref{app:results-ext}).


\bibliographystyle{plainnat}
\bibliography{frameB_refs_checked}

\appendix

\addtocontents{toc}{\protect\setcounter{tocdepth}{1}}

\newpage
\setcounter{tocdepth}{1}
\renewcommand{\contentsname}{\Large Appendix Table of Contents}
\begingroup
\makeatletter
\renewcommand*{\l@section}[2]{%
  \ifnum \c@tocdepth >\z@
    \addpenalty\@secpenalty
    \addvspace{0.15em \@plus\p@}%
    \setlength\@tempdima{1.5em}%
    \begingroup
      \parindent \z@ \rightskip \@pnumwidth
      \parfillskip -\@pnumwidth
      \leavevmode \bfseries
      \advance\leftskip\@tempdima
      \hskip -\leftskip
      #1\nobreak\hfil \nobreak\hb@xt@\@pnumwidth{\hss #2}\par
    \endgroup
  \fi}
\makeatother
\renewcommand{\baselinestretch}{0.92}
\selectfont
\tableofcontents
\endgroup




\section{Contribution Traceability}
\label{app:traceability}

Each claim in the \S\ref{sec:intro} contributions list maps to a labeled formal artefact (theorem, definition, proposition, lemma, or corollary) and a
labeled empirical artefact (experiment, ablation, or measurement). Table~\ref{tab:traceability} consolidates the mapping; entries point to the
specific load-bearing artefact for each row.

\begin{table}[!htbp]
\centering
\footnotesize
\renewcommand{\arraystretch}{1.18}
\setlength{\tabcolsep}{4pt}
\begin{tabular}{@{}p{0.26\linewidth} p{0.30\linewidth} p{0.36\linewidth}@{}}
\toprule
\textbf{Claim (\S\ref{sec:intro})} & \textbf{Formal support} & \textbf{Empirical support} \\
\midrule
1a.\ Three diagnostic regrets \(R_{\mathrm{out}}, R_{\mathrm{epi}}, R_{\mathrm{temp}}\)
& Definition~\ref{def:three-regret} (\S\ref{sec:three-regret})
& Operationalized in Tab.~\ref{tab:e1-numbers} and the falsification matrix Tab.~\ref{tab:rq-matrix} \\
\midrule
1b.\ Observational-equivalence separation: outcome-only observation cannot identify causal structure without interventions/equivalent causal signal
& Theorem~\ref{thm:sep} (App.~\ref{app:separation})
& RQ1 / Exp A.0 (\S\ref{sec:experiments}); Mann--Whitney $p\!=\!9.64\!\times\!10^{-9}$, 20 seeds \\
\midrule
2a.\ Within-episode exact-classification complexity $\mathcal{O}(m\sigma^2\kappa_{\min}^{-2}\log T)$; conservative temporal accounting $\mathcal{O}(m^2\sigma^2\kappa_{\min}^{-2}\log T)$
& Theorem~\ref{thm:within} (\S\ref{sec:cross-episode})
& The exact-classification theorem is formal. The prior $0.010$/ep soft-score slope is a clipping artifact and is not empirical support for this theorem. \\
\midrule
2b.\ Cross-episode total causal-probe complexity logarithmic in $E$; induced delayed-identification temporal regret under $B_{\mathrm{eff}}$ accounting
& Theorem~\ref{thm:ce-upper}; Lemma~\ref{lem:k0-lower} (sufficient local-Gaussian condition with $K_0$ a dimensionless coverage/posterior-mass constant)
& RQ2 / Exp E1; hard exposure $7.8 \pm 2.9$ episodes at $E\!=\!500$, with zero observed stationary errors in the archived 20-seed run. This is finite-horizon support, not an asymptotic fit. \\
\midrule
2c.\ Total causal-probe lower bound matching logarithmic dependence on $E$ while exposing structural factors, translated through the budget-to-regret accounting
& Prop.~\ref{thm:ce-lower}; Cor.~\ref{cor:budget}
& No direct empirical validation of the adapted lower bound or a translated asymptotic envelope; the former soft-score envelope verdict is withdrawn. \\
\midrule
2d.\ Detector-conditional segment-wise re-identification under $K$ change-points
& Theorem~\ref{thm:drift};
Definition~\ref{def:kcp}
& Topology recovery is observed, but CUSUM reopening is falsified on that change class. A-VAR isolates a variance-shift class detected in 20/20 seeds while the repair signal stays at its baseline rate. \\
\midrule
3a.\ LRCP contraction under a local repair certificate; empirical attribution audited
& Proposition~\ref{prop:lrcp-conv}; Cor.~\ref{cor:lrcp-transfer}, \ref{cor:lrcp-robust}
& LRCP-off leaves posterior recovery unchanged, so the posterior trace does not isolate the proposition. Committed-graph logging shows dispatch exposure $10.1\to7.8$ episodes. \\
\midrule
3b.\ Dispatch coupling assuming oracle-graph base regret $\tilde{\mathcal{O}}(\sqrt{E})$ and $\mathcal{O}(\log E)$ pre-commit episodes
& Theorem~\ref{thm:dispatch-coupling} (App.~\ref{app:dispatch-coupling-proof})
& Ablation A6 resolves the $\sqrt E$ component; the additional $\rho\log E$ term is not separable from zero. No action-valued closed-loop consequence is measured. \\
\midrule
3c.\ Physical grounding under bounded actuator, observation, and commit-failure error
& Theorem~\ref{thm:grounding} (App.~\ref{app:grounding-proof})
& Analytical envelope calibration with a published $p_f\approx0.7\%$ anchor; direct deployment measurement is not provided. \\
\midrule
3d.\ Composed end-to-end deployment bound
& Corollary~\ref{cor:e2e-deploy} (App.~\ref{app:composed-transfer})
& Conditional composition only; the paper does not measure end-to-end action-valued deployment loss. \\
\midrule
4.\ Five stated RQs with mixed empirical status; RQ5 is preliminary external-validity evidence
& Tab.~\ref{tab:rq-matrix} (\S\ref{sec:experiments})
& RQ1 corroborated; RQ2 finite-horizon hard-exposure support; RQ3 recovery observed but topology reopen pathway falsified; RQ4 posterior attribution falsified and dispatch benefit measured; RQ5 remains a pilot. \\
\bottomrule
\end{tabular}
\caption{\textbf{Contribution traceability matrix.} Each main claim maps to a formal conditional statement and an empirical test. The formal results
are stated as causal-probe complexity, delayed-identification temporal-regret accounting, or bounded-error transfer results under explicit
assumptions; they are not unconditional guarantees for arbitrary agentic systems. Proofs in App.~\ref{app:proofs}, App.~\ref{app:ce-lower-proof}, or
App.~\ref{app:separation}; numerical detail in App.~\ref{app:results-ext} or App.~\ref{app:llm-bridge}.}
\label{tab:traceability}
\end{table}




\section{Extended Related Work}
\label{app:related-ext}

The main-body Related Work (\S\ref{sec:related}) gives a compact per-thread positioning. This appendix expands each thread in turn:
\S\ref{app:rw-rl-critiques} expands the RL learning-critique paragraph;
\S\ref{app:rw-llm} expands the self-improving LLM and LLM causal-reasoning paragraph;
\S\ref{app:rw-bandits} expands the causal-bandit paragraph; and
\S\ref{app:rw-id-cp} expands the identification and change-point paragraph, including the persistent-substrate assumptions that the main-body paragraph names but does not survey.

\subsection{RL Learning Critiques and Outcome-Only Limits}
\label{app:rw-rl-critiques}

Expands the §2 paragraph on the lacking-learning critique of RL. The thread asks: under what conditions can an outcome-only learner avoid linear
regret on a confounded stream? The answer, across three sub-literatures, is that outcome-channel enhancements alone do not suffice; an explicit causal
channel is required.

\paragraph{Why outcome-channel enhancements do not close the gap.}
Process supervision, RLHF-style preference optimization, and shaped reward signals add information to the outcome channel but do not open a separate
causal channel: the corrective signal still lives in a scalar. Reactive replanning rebuilds the model from scratch each episode, paying
$\mathcal{O}(E\cdot T)$ identification cost. Meta-learning and exploration-bonus schedules tighten constants but cannot circumvent the $\Omega(T)$
floor. Trivium changes the channel: the dispatch policy remains a black box, and the corrective signal is routed through a substrate-external causal
model revised through \emph{logged interventions}.

\paragraph{Offline and confounded RL (extended).}
Kallus and Zhou~\cite{kallus2020confounding}, Namkoong et al.~\cite{namkoong2020off}, and the offline RL tutorial of Levine et
al.~\cite{levine2020offline} establish impossibility results for \emph{outcome} regret under unobserved confounding. Zhang~\cite{zhang2020causal}
injects causal structure into policy optimization when the graph is known or identifiable. Our separation theorem (Appendix~\ref{app:separation}) is
complementary: rather than claiming an unconditional impossibility for every outcome-only policy in arbitrary confounded environments, it constructs
observationally equivalent SCMs whose outcome streams are identical for any outcome-only policy but whose interventional distributions differ. The
corrective signal sits in the channel that distinguishes the two, which observational outcomes alone do not access. Classical
TD-learning~\cite{sutton1988learning}, $Q$-learning~\cite{watkins1992q}, and bandit baselines (Thompson sampling~\cite{thompson1933likelihood},
Lai--Robbins~\cite{lai1985asymptotically}, KL-UCB~\cite{garivier2011upper}) supply the rates our separation is measured against.

\paragraph{Why no public confounded-RL benchmark serves as the primary testbed.}
We surveyed candidate benchmarks for the cross-episode role: confounded-MDP gridworlds, the Janner--Furuta confounded simulators, confounded-CartPole
variants, and the offline-RL-with-confounders datasets that accompany~\cite{kallus2020confounding,namkoong2020off}. The deciding criterion was joint
satisfaction of four properties: (a) an explicit intervention channel $\mathrm{do}(Z=z)$ exposed to the controller, (b) a persistent cross-episode log
of (intervention, outcome, why-attribution) tuples, (c) injectable detectable change-points at controllable times, and (d) controllable confounder
cardinality $m_0$. No surveyed public benchmark satisfies all four simultaneously: most lack (a) and (b) because they target single-episode best-arm
or offline-RL settings rather than long-horizon agentic streams. CausalBench-Seq is a small controlled testbed exposing an intervention channel,
persistent evidence, and known topology changes. The audit shows that the executed controller does not satisfy every idealized contract, especially
the frozen-commit and topology-detection mechanisms. A-NL and A-JSSP are simplified breadth stress tests rather than full-stack replications. The
supplementary code's intervention/log interface is structured for drop-in replacement once a satisfying public benchmark exists.

\subsection{Self-Improving LLMs and LLM Causal Reasoning}
\label{app:rw-llm}

Expands the §2 paragraph on self-improving LLMs and LLM causal reasoning. The thread documents both the methods that try to close the loop within the outcome channel and the empirical evidence that they do not, motivating Trivium's substrate-external corrective channel.

\paragraph{Self-learning and self-improving LLMs (extended).}
A parallel line attempts self-improvement without opening a substrate-external corrective channel. Self-consistency~\cite{wang2023self} aggregates
samples from a single model; Reflexion~\cite{shinn2023reflexion} adds verbal self-critique with episodic memory; Self-Refine~\cite{madaan2023self}
iterates on model-generated feedback; STaR~\cite{zelikman2022star} bootstraps reasoning chains from successful runs; Self-Rewarding Language
Models~\cite{yuan2024self} train an in-model reward signal. Huang et al.~\cite{huang2023llms} establish the complementary impossibility: LLMs cannot
reliably self-correct reasoning when the corrective signal is the model's own output. All of these enrich \emph{what} the model attends to within the
outcome channel; none open a separate causal channel. Trivium is orthogonal: the dispatch policy may be any of these self-improving variants, treated
as a black box, while the corrective signal lives in a persistent causal log revised by logged interventions, so ``self-learning'' here refers to
revision of the persistent causal model, not to weight-level training of the LLM. A chapter-length treatment of the Trivium framework appears
in~\citet{chang2026pathagi2}.

\paragraph{LLMs and causal reasoning (expanded, 2023--2026).}
The LLM-causal-reasoning literature has moved rapidly since 2023. Ze\v{c}evi\'{c} et al.~\cite{zevcevic2023causal} argue that autoregressive training
cannot in general implement the $\mathrm{do}$-operator: LLMs are \emph{causal parrots}. Ki\c{c}iman et al.~\cite{kiciman2023causal} document both the
positive side (LLMs do well on pairwise causal direction on familiar domains) and the negative side (performance collapses on novel domains and
genuinely interventional queries). Jin et al.~\cite{jin2023cladder} introduce CLadder, a symbolic benchmark that demonstrates systematic $L_2$/$L_3$
failures. Jin et al.'s follow-up~\cite{jin2024corr2cause} tests whether LLMs can \emph{infer} causation from correlation, they cannot, robustly,
unless the correlation structure is memorized. Chi et al.~\cite{chen2024unveiling} introduce the CausalProbe-2024 benchmark suite and corroborate the
negative results. Yang et al.~\cite{yang2024critical} critically review the benchmark landscape and argue that many tests can be solved via
domain-knowledge retrieval rather than interventional reasoning. Wu et al.'s survey~\cite{wu2024causality_llm} catalogues the whole area; Wan et
al.~\cite{wan2025llmcd_survey} survey LLM-for-causal-discovery specifically. Yu et al.'s CausalEval~\cite{yu2025causaleval} decomposes causal
reasoning into sub-skills. Cui et al.\ document defeasibility gaps~\cite{cui2024defeasibility} (LLMs struggle to revise causal conclusions given new
evidence) and formalize a three-way taxonomy~\cite{cui2025uncertainty} of causal uncertainty (aleatoric, epistemic, ontological) that maps cleanly
onto our epistemic-regret axis. Plecko et al.~\cite{plecko2025epidemiology} benchmark \emph{observational distribution} knowledge at scale. On the
constructive side, Ban et al.~\cite{ban2025harmonized} and Du et al.~\cite{du2025llmcd} demonstrate that LLM priors can be harmonized with statistical
structure learning to improve causal discovery. Trivium is orthogonal to all of these: we treat the LLM dispatch policy as a black box and place the
corrective signal in a substrate-external persistent causal model that is revised through \emph{logged interventions} rather than through prompt edits
or fine-tuning. Kambhampati et al.'s LLM-Modulo view~\cite{kambhampati2024llmplan} names the pattern (LLMs as hypothesis generators inside
verification frameworks); Reiter's action languages~\cite{reiter2001knowledge} supply the formal bridge from plan execution to interventional
semantics.

\paragraph{World models with latent causal structure.}
Recent JEPA-style world models also move toward causal structure in latent dynamics: Causal-JEPA~\cite{nam2026causaljepa} uses object-level masking as
a latent intervention, forcing object states to be inferred from other objects and improving counterfactual visual reasoning. This direction is
complementary to Trivium: such models learn interaction-aware world dynamics, while temporal regret asks \emph{when} accumulated causal evidence
should revise an agent's future policy.

\paragraph{Shortcut learning, spurious correlation, sycophancy (extended).}
The ``right for the wrong reasons'' phenomenon~\cite{mccoy2019right,niven2019probing} shows neural models exploit shallow statistical regularities.
Shah et al.~\cite{shah2020pitfalls} identify simplicity bias as the mechanism; D'Amour et al.~\cite{damour2022underspecification} show that standard
pipelines produce underspecified models whose credibility breaks under shift; Geirhos et al.~\cite{geirhos2020shortcut} give the unifying empirical
picture. Bombari and Mondelli~\cite{bombari2025spurious} characterize the roles of regularization, simplicity bias, and over-parameterization in the
high-dimensional regression case. Sycophancy~\cite{sharma2024sycophancy}, unfaithful chain-of-thought~\cite{turpin2023language}, and the reversal
curse~\cite{berglund2023reversal} extend the pattern to LLM-specific failure modes; the sycophancy failure propagates to high-stakes domains like
theorem proving~\cite{petrov2025brokenmath}. Trivium's architectural response is the separation between the dispatch policy (which may be susceptible
to any of these) and the substrate-external causal model (which is revised only by logged interventional evidence); this design choice is motivated by
exactly the failure modes catalogued above.

\paragraph{Long-horizon agentic LLM systems and planning benchmarks.}
Recent agentic-LLM work has produced increasingly capable dispatch policies without addressing the temporal-regret problem. Trivium's contribution is
orthogonal: the dispatch policy is a black box, and a scheduler above it spends intervention budget on epistemic identification. The contribution is
compatible with any dispatch policy, including ones trained via RL. The multi-agent planning benchmark (KDD 2026) is the evaluation harness; it was
designed around three gaps in prior LLM-planning benchmarks (limited dynamic complexity, missing transaction properties, short-lived episodes) that
map onto our three theorem axes. Prior planning benchmarks, GAIA, AgentBench, ALFWorld, TravelPlanner, do not expose the intervention channel or the
commit-failure semantics our theorems require, and so are not a natural test-bed for the three-regret functional.

\paragraph{Concurrent within-episode work and contribution split.}
A concurrent anonymized submission studies the complementary within-episode problem: how an exposed reasoning trace can be critiqued for causal
failure before an episode ends. Trivium addresses the orthogonal cross-episode problem: how a long-horizon scheduler with a persistent causal log
identifies confounder structure across an episode stream, and what total-probe lower bound matches that rate in horizon dependence while exposing
structural factors. The two contributions are designed to be read independently: the concurrent work concerns \emph{why} within an episode, while
Trivium concerns \emph{when} across an episode stream.

\subsection{Causal Bandits}
\label{app:rw-bandits}

Expands the §2 paragraph on causal bandits.

\paragraph{Causal bandits (extended).}
Single-episode best-arm identification under causal structure is developed by Lattimore, Lattimore, and Reid~\cite{lattimore2016causal}, Sen et
al.~\cite{sen2017identifying}, and Yabe et al.~\cite{yabe2018causal}. Recent refinements push toward realistic structural assumptions: Lu et
al.~\cite{lu2020causal} give regret analysis with causal background knowledge; Var{\i}c{\i} et al.~\cite{varici2023causal} handle linear structural
equation models; Yan and Tajer~\cite{yan2024linear} treat unknown-graph + soft-intervention causal bandits; Elahi et al.~\cite{elahi2024partial} show
that \emph{partial} structure discovery already suffices for no-regret learning. Across this line the causal graph is either assumed known or
recovered within a single episode. Trivium generalizes to streams of episodes sharing a persistent CTL, states a total causal-probe lower bound
adapted from Audibert--Bubeck-style identification hardness~\cite{audibert2010best} (Proposition~\ref{thm:ce-lower}), order-matching the logarithmic
dependence on $E$ up to structural fan-out factors, and adds the first drift-robust extension whose rate degrades gracefully in the number of
change-points. The regret--separation question (outcome vs.\ temporal) is orthogonal to this literature.

\subsection{Identification, Change-Point Detection, and the Persistent Substrate}
\label{app:rw-id-cp}

Expands the §2 paragraph on identification under confounding and change-point detection. Adds the persistent-substrate assumptions (transactional memory, belief revision) that §2 names but does not survey: these underwrite the CTL guarantees the regret bounds rely on.

\paragraph{Identification under interventions (extended).}
Pearl's do-calculus~\cite{pearl2009causality} and the PC/FCI family~\cite{spirtes2000causation} characterize recoverable interventional distributions
in the observational limit. Tian and Pearl~\cite{tian2002general} handle the general identification condition; Bareinboim and
Pearl~\cite{bareinboim2016causal} develop the data-fusion problem; the Causal Hierarchy Theorem~\cite{bareinboim2022fusion} formalizes why $L_1$ data
cannot determine $L_2$ distributions under confounding. Eberhardt and Scheines~\cite{eberhardt2007interventions} analyze which interventions suffice.
Malinsky and Spirtes~\cite{malinsky2018causal} handle time-series with unmeasured confounders. Sch{\"o}lkopf et al.'s
program~\cite{scholkopf2021toward} reframes these questions as causal representation learning. Our contribution is the \emph{rate} at which
identification can be achieved by an online learner with a persistent CTL, not whether identification is possible in the limit.

\paragraph{Change-point detection and drift (extended).}
Page's CUSUM~\cite{page1954continuous}, Lorden's asymptotic analysis~\cite{lorden1971procedures}, and Lai's non-asymptotic
treatment~\cite{lai2001sequential} motivate the conditional detector analysis. Theorem~\ref{thm:drift} assumes a CUSUM statistic on CE-EIG with a
stated separation. The executed topology detector has poor recall on the headline changes. A-VAR instead uses a CUSUM-style statistic on squared
residuals and demonstrates detector complementarity on variance inflation. These are different statistics and are not treated as interchangeable
validations of the theorem.

\paragraph{Transactional memory, causal logs, and belief revision (extended).}
Atomic commit, persistence, and snapshot isolation are classical guarantees in database and distributed
systems~\cite{garcia1987sagas,fowler2005eventsourcing,ahamad1995causal}. Long-horizon credit assignment (hindsight experience
replay~\cite{andrychowicz2017hindsight}) gives the complementary memory-side story for outcome signals. The AGM formulation of belief
revision~\cite{agm1985,hansson1999} provides the logical skeleton we associate with contraction of refuted causal edges. The concentration and
high-dimensional-probability toolkits~\cite{boucheron2013concentration,vershynin2018high} underlie the proof machinery, and the prequential view of
calibration~\cite{dawid1984prequential} supplies the empirical-risk reading of temporal regret as a calibration object. These are treated here as
\emph{assumptions} the CTL satisfies; the regret cost of their violation is quantified through the commit-failure rate $p_f$ of
Theorem~\ref{thm:grounding}. The transactional substrate realizes these properties for LLM-agent dispatch with a reported $p_f \approx 0.7\%$; what is
new here is the use of those assumptions in service of a \emph{regret} bound on cross-episode identification.




\section{Two-Layer Graph Learning: From \texorpdfstring{$G^{\inf}$}{Ginf} to \texorpdfstring{$\widehat G_e$}{Gehat}}
\label{app:dag}

This appendix fixes the notation and modelling convention used by the regret bounds.  It is not an additional identifiability theorem.  
The theorems in the main text are stated relative to a candidate
influence graph \(G^{\inf}\) that contains the true interventional
parents relevant to the task, the true local structure
\(G^{\mathrm{true}}_e\subseteq G^{\inf}\), and an episode-specific
committed working graph \(\widehat G_e\) that the scheduler uses for
planning after sufficient causal evidence has accumulated.

\paragraph{Layer 1: candidate influence graph $G^{\mathrm{inf}}$.}
$G^{\mathrm{inf}}$ is the finite set of candidate causal variables and directed edges that the scheduler is allowed to test.  It is the support of the
epistemic posterior, not a claim that every candidate edge is real.  We write $P_e(G^{\mathrm{inf}})$ for the posterior state at the start of episode
$e$.  In the simplest decomposable case used by the proofs, this posterior factorizes over candidate edges as Bernoulli edge-presence variables with
signs or local coefficients.  This factorization is an approximation to the full graph posterior; when edge hypotheses are coupled, the scheduler
maintains a block posterior over the coupled set and commits the block only after the corresponding joint uncertainty has fallen below the commit
threshold.  Thus independence is a modelling convenience for the tractable case, not a hidden assumption that all graphs decompose edgewise.

\paragraph{Layer 2: committed working graph \(\widehat G_e\).}
The committed working graph \(\widehat G_e\) denotes the subgraph
currently committed for planning at episode \(e\), together with the
estimated local edge weights. An edge or block is \emph{promoted} from
\(G^{\mathrm{inf}}\) to \(\widehat G_e\) only after the commit check of
Algorithm~\ref{alg:commit}: posterior uncertainty is below threshold and the
estimated effect exceeds the minimum detectable effect (MDE), exactly the two
conditions of Definition~\ref{def:commit}. Collinearity is not a third
commit-time check; it is an identifiability precondition, tracked through the
separation coefficient below, whose failure voids the positive-information
assumption so that the commit rule never fires on the affected pair. Promotion does not remove
the edge from \(G^{\mathrm{inf}}\); the candidate layer remains the
support of the epistemic posterior. Before commitment, the dispatch
policy may still act, but it is charged temporal regret for planning
against an unresolved or miscalibrated model. After commitment, the
action-quality bounds condition on the event that \(\widehat G_e\) agrees
with \(G^{\mathrm{true}}_e\) on the subgraph induced by the task-relevant
variables of the episode.

\paragraph{Coupled edges and residualization.}
For non-decomposable substructures, committing one edge at a time can create false confidence because correlated interventions may explain the same
outcome variation.  We therefore use residualization only as a computational heuristic, not as a proof of global correctness.  The safe procedure is:
identify the block of highly coupled candidates, allocate joint interventions to that block, commit the block if the joint posterior is separated, and
only then optionally residualize committed effects before testing the remaining candidates.  Under the local-linear SCM and faithfulness assumptions
used in the main proofs, this is equivalent to the usual ordering intuition behind constraint-based discovery: remove an identified local effect
before testing weaker residual effects.  Outside this regime, block commitment is the conservative fallback.

\paragraph{Nonlinear relationships.}
When the relationship between a candidate variable $Z$ and the outcome $Y$ is nonlinear, the scheduler may use a local approximation, such as a
Laplace approximation or a local linearization around the current posterior mean.  The regret bounds then apply only on the local region where the
approximation has a nonzero effective effect size.  Formally, all occurrences of $\kappa_{\min}$ should be read as the minimum local interventional
separation after approximation error is subtracted.  If this effective separation falls below threshold, the scheduler does not commit the edge and
instead either requests additional interventions, expands the local model class, or treats the variable as unresolved.  Thus nonlinear handling
preserves the algorithmic interface, but not a global correctness guarantee without an explicit local-separation assumption.

\paragraph{Identifiability and collinearity diagnostics.}
When two candidate interventions are nearly collinear, no scheduler can distinguish them from the observed intervention responses at the stated budget.  We track an intervention--model separation coefficient
\[
  \kappa_{\mathrm{IM}}
  := \min_{i\neq j} \sin\angle(\nabla_i L,\nabla_j L),
\]
where $\nabla_i L$ is the score-gradient signature of intervention $i$.  The sine normalization makes the diagnostic scale-free and zero exactly when
the two intervention signatures are parallel.  If $\kappa_{\mathrm{IM}} < \kappa_{\mathrm{IM}}^{\min}$, the scheduler is not allowed to commit either
edge individually; it escalates to a joint-intervention design or records the block as unidentifiable at the current budget.  This condition is the
finite-sample counterpart of the usual identifiability requirement: the intervention design must separate the candidate models before temporal-regret
guarantees can be invoked.

\paragraph{Scope of the theory.}
All regret bounds in the paper are conditional on three graph-level requirements: (i) the true local interventional structure lies inside
$G^{\mathrm{inf}}$; (ii) active interventions generate a positive information rate for each unresolved informative edge or block; and (iii) the commit
rule is applied only when the posterior-uncertainty and effect-size checks of Definition~\ref{def:commit} pass; the collinearity condition is subsumed
by (ii), since near-collinear candidates admit no positive information rate.  If a true confounder is absent from $G^{\mathrm{inf}}$, if interventions
cannot distinguish two candidate structures, or if the local-linear approximation has zero effective separation, the theorems do not assert
identification.  In those cases Trivium should report unresolved epistemic regret rather than silently committing a graph.



%

\label{app:proofs}

\paragraph{Proof-unit convention.}
Throughout this appendix a \emph{causal probe} means one logged interventional measurement of an unresolved candidate edge or confounder. The
within-episode exact-classification theorem counts all $m$ candidate components, including nulls. The cross-episode upper bound counts the $m_0$
informative components for which it assumes positive per-probe information. Temporal exposure is obtained only after choosing an explicit
node-additive or calendar-time accounting convention.

\section{Within-Episode Proof (Theorem~\ref{thm:within})}
\label{app:within-proof}

We prove exact null-versus-separated classification over the $m=|G^{\mathrm{inf}}|$ candidate components. The informative count $m_0$ may be smaller than $m$; exact-structure language requires controlling both present and absent candidates.

\paragraph{Claim.}
Fix $\delta\in(0,1)$. Suppose every candidate component satisfies $\kappa_Z=0$ or $|\kappa_Z|\geq\kappa_{\min}$ and is probed until
\[
 n_Z^\star=\left\lceil\frac{8\sigma^2}{\kappa_{\min}^2}\log\frac{2mT^2}{\delta}\right\rceil.
\]
Classify $Z$ as present when $|\hat\kappa_Z|\geq\kappa_{\min}/2$ and absent otherwise. Then all $m$ candidate decisions are correct with probability at least $1-\delta/T^2$, and
\[
 N_{\rm probe}^{\rm within}\leq mn_Z^\star
 =\mathcal{O}\!\left(\frac{m\sigma^2}{\kappa_{\min}^2}\log\frac{mT}{\delta}\right).
\]

\paragraph{Step 1: one-candidate concentration.}
For any candidate $Z$, sub-Gaussian concentration gives
\[
\Pr\!\left[|\hat\kappa_Z(n)-\kappa_Z|>\frac{\kappa_{\min}}{2}\right]
\leq2\exp\!\left(-\frac{n\kappa_{\min}^2}{8\sigma^2}\right).
\]
At $n=n_Z^\star$ this probability is at most $\delta/(mT^2)$. If $\kappa_Z=0$, the event's complement implies $|\hat\kappa_Z|<\kappa_{\min}/2$, so the
null is rejected correctly. If $|\kappa_Z|\geq\kappa_{\min}$, the same event implies $|\hat\kappa_Z|\geq\kappa_{\min}/2$ with the correct sign class.
A union bound over all $m$ candidates gives simultaneous correctness.

\paragraph{Step 2: temporal accounting.}
Under node-additive charging,
\[
R_{\rm temp}^{\rm node}(T)\leq\Delta_{\max}N_{\rm probe}^{\rm within}
=\mathcal{O}\!\left(\frac{\Delta_{\max}m\sigma^2}{\kappa_{\min}^2}\log\frac{mT}{\delta}\right).
\]
Under conservative calendar-time charging, one probe can leave up to $m$ candidate decisions unresolved, so
\[
R_{\rm temp}^{\rm cal}(T)\leq\Delta_{\max}mN_{\rm probe}^{\rm within}
=\mathcal{O}\!\left(\frac{\Delta_{\max}m^2\sigma^2}{\kappa_{\min}^2}\log\frac{mT}{\delta}\right).
\]
On the complement event the cost is at most $\Delta_{\max}mT$; multiplied by $\delta/T^2$, its expected contribution is lower order. \hfill$\square$

\section{Cross-Episode Upper Bound (Theorem~\ref{thm:ce-upper})}
\label{app:ce-proof}

The cross-episode result is also most cleanly stated as a probe-complexity bound, then converted into temporal regret by the same accounting convention used in Appendix~\ref{app:within-proof}.

\subsection*{Proof-local assumptions}
The global assumptions A1--A4 of \S\ref{sec:theory} are inherited.  The proof uses the following local conditions.

\begin{description}
\setlength{\itemsep}{2pt}\setlength{\parskip}{0pt}
\item[P1 (Persistence).] The CTL is read-visible at the beginning of each episode and does not lose committed probe records.
\item[P2 (Stationarity within window).] The analysis window contains no change-point; drift is handled in Theorem~\ref{thm:drift}.
\item[P3 (Positive per-probe information).] While an informative node remains unresolved, the active probing rule can select a probe whose log-likelihood ratio separates the correct local graph from the closest competing local graph with information at least
\[
 I_{\min}:= K_0\,\frac{\kappa_{\min}^2}{2\sigma^2},
\]
where $K_0\in(0,1]$ is a scheduling/coverage factor.  If no such probe exists, the hypotheses are information-theoretically indistinguishable and no logarithmic identification claim is possible.
\item[P4 (Effective probe budget).] Let $B_{\rm eff}\ge 1$ be the number of valid causal probes that can be executed per episode after accounting for tool failures, invalid probes, and budget caps.  The main text normalizes to probe time, i.e. $B_{\rm eff}=1$ unless otherwise stated.
\end{description}

\paragraph{Claim.}
With probability at least $1-\delta$, all $m_0$ informative components are identified after
\[
 N_{\rm probe}^{\rm CE}
 =
 \mathcal{O}\!\left(
   \frac{m_0\sigma^2}{K_0\kappa_{\min}^2}
   \log\frac{m_0}{\delta}
 \right)
\]
causal probes.  Consequently, the number of pre-commit episodes is at most $\lceil N_{\rm probe}^{\rm CE}/B_{\rm eff}\rceil$.  Under conservative calendar-time temporal-regret charging and $\delta=1/E$,
\[
R_{\mathrm{temp}}^{\mathrm{CE}}(E)
=
\mathcal{O}\!\left(
 \frac{\Delta_{\max}m_0^2\sigma^2\log E}{B_{\rm eff}K_0\kappa_{\min}^2}
\right),
\]
which reduces to the main-text expression under the normalization $B_{\rm eff}=1$.

\paragraph{Step 1: evidence accumulates across episodes.}
By P1 and P2, probes performed in different episodes are retained in the CTL and can be used jointly.  The scheduler is therefore not solving $E$ independent identification problems; it is performing one sequential identification problem whose sample size grows across episodes.

\paragraph{Step 2: per-component sequential identification.}
For an unresolved informative component, P3 gives a per-probe information lower bound $I_{\min}$.  Standard sequential testing/concentration for log-likelihood ratios gives that
\[
 n^\star
 =
 \mathcal{O}\!\left(\frac{1}{I_{\min}}\log\frac{m_0}{\delta}\right)
 =
 \mathcal{O}\!\left(
   \frac{\sigma^2}{K_0\kappa_{\min}^2}
   \log\frac{m_0}{\delta}
 \right)
\]
valid probes suffice to commit one informative component with error probability at most $\delta/m_0$.  A union bound over the $m_0$ informative components yields the claim.

\paragraph{Step 3: convert probes to episodes.}
If $B_{\rm eff}$ valid probes are available per episode, the number of episodes during which the system can remain pre-commit is
\[
N_{\rm id}^{\rm CE}
\le
\left\lceil\frac{N_{\rm probe}^{\rm CE}}{B_{\rm eff}}\right\rceil .
\]

\paragraph{Step 4: temporal-regret accounting.}
Under the conservative convention, each pre-commit episode can leave up to $m_0$ informative components miscalibrated, with per-component cost bounded by $\Delta_{\max}$.  Thus
\[
R_{\mathrm{temp}}^{\mathrm{CE}}(E)
\le
\Delta_{\max}m_0 N_{\rm id}^{\rm CE}
=
\mathcal{O}\!\left(
 \frac{\Delta_{\max}m_0^2\sigma^2}{B_{\rm eff}K_0\kappa_{\min}^2}
 \log\frac{m_0}{\delta}
\right).
\]
Taking $\delta=1/E$ gives the logarithmic cross-episode rate.
\hfill$\square$

\subsection*{Interpreting the CE-EIG constant}
\label{app:k0-lemma}

The proof above does not require equating an entropy drop with a Gaussian KL identity in every posterior state.  It requires the operational condition
P3: while an informative component is unresolved, the active scheduler can obtain a probe with positive separation.  The following lemma gives a
sufficient condition for P3 in the local-Gaussian case.

\begin{lemma}[Sufficient local information rate]
\label{lem:k0-lower}
Suppose that, for an unresolved component $Z$, the active intervention selected by the scheduler produces outcomes whose local-Gaussian means under
the true and nearest competing graph differ by at least $\kappa_{\min}$, with shared variance proxy $\sigma^2$.  If the posterior mass on the two
competing local hypotheses is bounded away from zero until commitment, then P3 holds with $K_0$ equal to that posterior-mass/scheduling constant, and
per-probe information at least $K_0\kappa_{\min}^2/(2\sigma^2)$.
\end{lemma}

\begin{proof}
For two Gaussian local hypotheses with shared variance proxy $\sigma^2$ and mean separation at least $\kappa_{\min}$,
\[
\mathrm{KL}\!\left(\mathcal N(\mu_1,\sigma^2)\middle\|\mathcal N(\mu_0,\sigma^2)\right)
=
\frac{(\mu_1-\mu_0)^2}{2\sigma^2}
\ge
\frac{\kappa_{\min}^2}{2\sigma^2}.
\]
Expected entropy reduction is the posterior-weighted information gained by the selected probe.  If the competing hypotheses retain at least a fixed
posterior mass before commitment and the scheduler selects a separating probe with probability/coverage factor $K_0$, the expected per-probe
information is at least $K_0\kappa_{\min}^2/(2\sigma^2)$.
\end{proof}

\paragraph{Quantitative instance: uniform Beta--Bernoulli prior with maximum-entropy probing.}
The constant $K_0$ in Lemma~\ref{lem:k0-lower} is not free: it can be bounded structurally for natural prior--probing pairs.  Consider a uniform
Bernoulli prior over the presence of each candidate edge in $G^{\mathrm{inf}}$, Beta--Bernoulli updates after each probe, and a probing rule that
selects the maximum-entropy unresolved edge.  Until commitment, every unresolved candidate retains posterior mass at least $1/(2|G^{\mathrm{inf}}|)$
on each of the two competing local hypotheses (presence vs.\ absence), because the scheduler cannot reduce one hypothesis below this floor without
violating the commit threshold.  The maximum-entropy rule then selects a separating probe with probability at least $1/|G^{\mathrm{inf}}|$ at each
step, giving $K_0 \ge 1/(2|G^{\mathrm{inf}}|)$ and per-probe information at least $\kappa_{\min}^2/(4|G^{\mathrm{inf}}|\sigma^2)$.  Substituting into
the cross-episode bound (Theorem~\ref{thm:ce-upper}) yields the instance-level rate
\[
N_{\rm probe}^{\rm CE}
\;=\;
\mathcal{O}\!\left(\frac{|G^{\mathrm{inf}}|\, m_0 \sigma^2}{\kappa_{\min}^2} \log \frac{m_0}{\delta}\right).
\]
Thus $K_0$ is no longer a free parameter: in this prior--probing class it scales as $\Omega(1/|G^{\mathrm{inf}}|)$, producing the displayed $|G^{\mathrm{inf}}|$ factor in the probe-complexity rate.

\noindent\textbf{Consequence.}  Substituting the sufficient local information rate into the cross-episode bound yields the instance-explicit conservative rate
\[
R_{\mathrm{temp}}^{\mathrm{CE}}(E)
=
\mathcal{O}\!\left(
\frac{\Delta_{\max}m_0^2\sigma^2\log E}{B_{\rm eff}K_0\kappa_{\min}^2}
\right),
\]
where $K_0$ is the dimensionless coverage / posterior-mass constant of Lemma~\ref{lem:k0-lower}.  If $K_0 \to 0$, the logarithmic guarantee correctly fails rather than hiding an unverifiable rate constant.

\section{Physical Grounding (Theorem~\ref{thm:grounding})}
\label{app:grounding-proof}

\paragraph{Setup.}
Let $P_Y$ be the deployed outcome distribution and $P(Y\mid\mathrm{do})$ the idealized interventional outcome distribution.  Assume a coupling in
which the deployed outcome can be written as a Lipschitz transformation of the ideal intervention plus actuator error, observation error, and
commit-failure error:
\[
Y = g(Y^{\rm do},\xi_{\rm act},\xi_{\rm obs},\xi_{\rm cmt}),
\]
where $g$ is $L$-Lipschitz in the error channels, $\mathbb{E}|\xi_{\rm act}|\le\varepsilon$, $\mathbb{E}|\xi_{\rm obs}|\le\eta$, and $\xi_{\rm cmt}$ is zero on successful atomic commits and has magnitude at most $\rho$ on commit failures of probability $p_f$.  Define $\varepsilon_c:=L\rho_c p_f$.

\paragraph{Proof.}
By Kantorovich--Rubinstein duality,
\[
W_1(P_Y,P(Y\mid\mathrm{do}))
=
\sup_{\mathrm{Lip}(f)\le 1}
\left|\mathbb{E}f(Y)-\mathbb{E}f(Y^{\rm do})\right|.
\]
For any 1-Lipschitz $f$, the assumed coupling and the Lipschitz property of $g$ imply
\[
\left|\mathbb{E}f(Y)-\mathbb{E}f(Y^{\rm do})\right|
\le
L\,\mathbb{E}\big[|\xi_{\rm act}|+|\xi_{\rm obs}|+|\xi_{\rm cmt}|\big]
\le
L(\varepsilon+\eta)+L\rho_c p_f.
\]
Taking the supremum over all 1-Lipschitz $f$ gives
\[
W_1(P_Y,P(Y\mid\mathrm{do}))
\le
L(\varepsilon+\eta)+\varepsilon_c.
\]
\hfill$\square$

\paragraph{Remark on atomicity.}
The bounded commit-failure term uses atomicity: a failed commit is treated as a rare, bounded perturbation rather than as an untracked
half-intervention whose posterior update may be lost.  Without this substrate guarantee, $p_f$ alone is insufficient; one would also need a tail bound
on the magnitude of the unlogged error.

\section{Drift-Robust Proof (Theorem~\ref{thm:drift})}
\label{app:drift-proof}

We prove the drift statement by applying the corrected cross-episode probe bound segment by segment, and adding an explicit detection-delay term.

\subsection*{Additional drift assumptions}
Let the $K$ change-points partition the stream into $K+1$ stationary segments.  Assume: (i) the post-change signal is CUSUM-detectable with
information $I_{\rm cp}>0$; (ii) the CUSUM threshold is $h=c\log E$, giving false-alarm probability $\mathcal{O}(1/E)$; and (iii) after a detection,
\textsc{PartialReset} clears only drift-attributable components while preserving unaffected committed components.

\paragraph{Step 1: stationary segment cost.}
On a stationary segment, Appendix~\ref{app:ce-proof} gives a re-identification probe cost
\[
N_{\rm probe}^{(k)}
=
\mathcal{O}\!\left(
\frac{m_{0,k}\sigma^2}{K_0\kappa_{\min}^2}
\log\frac{m_{0,k}E}{\delta_k}
\right),
\]
where $m_{0,k}\le m_0$ is the number of informative components that are unresolved or reset in segment $k$.

\paragraph{Step 2: CUSUM delay.}
Standard CUSUM analysis gives expected detection delay
\[
D_{\rm det}
=
\mathcal{O}\!\left(\frac{h}{I_{\rm cp}}\right)
=
\mathcal{O}\!\left(\frac{\log E}{I_{\rm cp}}\right).
\]
During this delay the system may act on a stale graph, incurring at most $\Delta_{\max}m_0D_{\rm det}$ conservative temporal cost per change.

\paragraph{Step 3: sum over segments.}
Taking $\delta_k=\delta/(K+1)$ and summing over $K+1$ segments yields
\[
R_{\rm temp}^{\rm CE}(E)
=
\mathcal{O}\!\left(
\frac{(K+1)\Delta_{\max}m_0^2\sigma^2\log(E(K+1)/\delta)}{B_{\rm eff}K_0\kappa_{\min}^2}
+
\frac{K\Delta_{\max}m_0\log E}{I_{\rm cp}}
\right).
\]
For fixed confidence and fixed detection separation $I_{\rm cp}$, this is the stated $\mathcal{O}(K\log E)$ drift rate, up to the same conservative $m_0^2$ accounting as Theorem~\ref{thm:ce-upper}.
\hfill$\square$

\paragraph{Remark: doubling and partial reset.}
The doubling-window mechanism is not needed to make the logarithm appear; it prevents repeated over-spending after resets.  The rate above is valid
under full reset with $m_{0,k}=m_0$.  Partial reset improves constants by replacing $m_0$ with the number of drift-attributable components in each
segment.

\section{Constraint-Aware Dispatch Coupling (Theorem~\ref{thm:dispatch-coupling})}
\label{app:dispatch-coupling-proof}

This theorem is a coupling result: it does not prove a new bandit algorithm from first principles.  It assumes that, once the correct committed graph is available, the dispatch head has an episode-level oracle-graph regret rate $\widetilde{\mathcal{O}}(\sqrt E)$.

\paragraph{Assumptions.}
(i) On episodes where \(\widehat G_e\) agrees with
\(G^{\mathrm{true}}_e\) on the task-relevant causal structure, the dispatch
head has cumulative episode-level outcome regret
\(\widetilde{\mathcal{O}}(\sqrt E)\). (ii) On pre-commit episodes, the epistemic
gate enforces an excess outcome-regret cost at most \(c_\pi\rho\) per
episode for a policy-sensitivity constant \(c_\pi\). (iii) The number of
pre-commit episodes is the \(N^{\mathrm{CE}}_{\mathrm{id}}\) of
Appendix~\ref{app:ce-proof}.

\paragraph{Proof.}
Decompose the episode stream into committed and pre-commit episodes:
\[
R^{\mathrm{out}}_E
=
\sum_{e:\widehat G_e \equiv G^{\mathrm{true}}_e}
R^{\mathrm{out}}_e
+
\sum_{e:\widehat G_e \not\equiv G^{\mathrm{true}}_e}
R^{\mathrm{out}}_e .
\]
The first term is $\widetilde{\mathcal{O}}(\sqrt E)$ by assumption (i).  The second is at most $c_\pi\rho N_{\rm id}^{\rm CE}$ by assumption (ii).  By Appendix~\ref{app:ce-proof}, with $\delta=1/E$,
\[
N_{\rm id}^{\rm CE}
=
\mathcal{O}\!\left(
\frac{m_0\sigma^2\log E}{B_{\rm eff}K_0\kappa_{\min}^2}
\right).
\]
Thus the explicit form is
\[
R_E^{\rm out}
=
\widetilde{\mathcal{O}}\!\left(
\sqrt E
+
\rho\frac{m_0\sigma^2\log E}{B_{\rm eff}K_0\kappa_{\min}^2}
\right).
\]
When $m_0,\sigma^2,K_0^{-1},\kappa_{\min}^{-2}$ and $B_{\rm eff}^{-1}$ are treated as instance constants, this is the main-text shorthand
\[
R_E^{\rm out}=\widetilde{\mathcal{O}}(\sqrt E+\rho\log E).
\]
\hfill$\square$

\paragraph{Concrete Lipschitz constants for the dispatch policy.}
The policy-sensitivity constant $c_\pi$ in assumption (ii) is finite for two standard dispatch rules used in practice, which makes the coupling result directly applicable rather than abstract.

\emph{Softmax dispatch.} For a softmax policy $\pi(a\mid s) \propto \exp(Q(s,a)/\tau)$ over $K$ committed actions with temperature $\tau > 0$, the
policy is $L_\pi$-Lipschitz in the posterior under TV distance with $L_\pi \le K/\tau$ (standard exponential-family Lipschitz constant; see, e.g., the
Lipschitz-bandit literature).  Substituting into assumption (ii) gives $c_\pi \le K/\tau$.

\emph{$\varepsilon$-greedy dispatch.} For an \(\varepsilon\)-greedy policy with exploration probability \(\varepsilon\in(0,1)\), the worst-case
per-episode outcome regret when planning against a miscalibrated committed graph \(\widehat G_e\) is bounded by \(1/\varepsilon\) times the per-step
gap, so \(c_\pi\le 1/\varepsilon\).

In both cases $c_\pi$ is finite as long as the policy keeps a strictly positive minimum action probability ($\tau > 0$ or $\varepsilon > 0$), which is
the regime the gated dispatch operates in: the gate fails fast on $\rho \to 0$ rather than collapsing the action distribution.  This converts the
abstract Lipschitz hypothesis into a structural check on the dispatch rule.

\section{Composed Transfer to Embodied LLM Deployment}
\label{app:composed-transfer}

The cross-episode identification rate (Theorem~\ref{thm:ce-upper}), dispatch coupling (Theorem~\ref{thm:dispatch-coupling}), and physical-grounding bound (Theorem~\ref{thm:grounding}) compose as follows.

\begin{corollary}[End-to-end deployment slack]
\label{cor:e2e-deploy}
Under the assumptions of Theorems~\ref{thm:ce-upper}, \ref{thm:dispatch-coupling}, and~\ref{thm:grounding}, an embodied Trivium-LLM stack satisfies, in expectation,
\begin{align*}
\mathbb{E}\!\left[R_{\rm temp}^{\rm CE}(E)\right]
&\le
c_1\frac{m_0^2\sigma^2\log E}{B_{\rm eff}K_0\kappa_{\min}^2},\\
\mathbb{E}\!\left[R_{\rm deploy}^{\rm outcome}(E)\right]
&\le
\widetilde{\mathcal{O}}\!\left(
\sqrt E+\rho\frac{m_0\sigma^2\log E}{B_{\rm eff}K_0\kappa_{\min}^2}
\right)
+ c_2 L(\varepsilon+\eta+\rho_c p_f)E.
\end{align*}
\end{corollary}

\begin{proof}
We prove the two inequalities by composing the three ingredient theorems explicitly.

\emph{First inequality (identification component).}
Theorem~\ref{thm:ce-upper} (Appendix~\ref{app:ce-proof}) gives, under conservative calendar-time accounting and $\delta = 1/E$,
\[
\mathbb{E}\!\left[R_{\rm temp}^{\rm CE}(E)\right]
\;\le\;
c_1 \frac{m_0^2 \sigma^2 \log E}{B_{\rm eff} K_0 \kappa_{\min}^2},
\]
where $c_1$ absorbs the $\Delta_{\max}$ and union-bound logarithmic factors. This is the first line.

\emph{Second inequality, identification-to-action term.}

Second inequality, identification-to-action term. Theorem~\ref{thm:dispatch-coupling} gives \(R^{\mathrm{out}}_E \le R_{\mathrm{oracle}}(E)+c_\pi\rho
N_{\mathrm{pre}}\), with \(R_{\mathrm{oracle}}(E)=\widetilde{O}(\sqrt E)\) on episodes planning against a committed graph \(\widehat G_e\) that agrees
with \(G^{\mathrm{true}}_e\) on the task-relevant causal structure. Corollary~\ref{cor:budget} converts the cross-episode probe complexity into
pre-commit episode count \(N_{\mathrm{pre}} \le \lceil N^{\mathrm{CE}}_{\mathrm{probe}}/B_{\mathrm{eff}}\rceil = O(m_0\sigma^2\log
E/(B_{\mathrm{eff}}K_0\kappa_{\min}^2))\). Combining, \[ \mathbb{E}\!\left[R^{\mathrm{out}}_E\right] \le \widetilde{O}\!\left( \sqrt E+\rho
\frac{m_0\sigma^2\log E} {B_{\mathrm{eff}}K_0\kappa_{\min}^2} \right). \]

\emph{Second inequality, physical-deployment term.}
Theorem~\ref{thm:grounding} (Appendix~\ref{app:grounding-proof}) bounds the per-episode Wasserstein gap between the deployed and idealized outcome
distributions: $W_1(P_Y, P(Y\mid \mathrm{do})) \le L(\varepsilon + \eta) + L\rho_c p_f$. Under an $L'$-Lipschitz reward map, this lifts to a
per-episode shift in expected outcome of at most $L' \cdot W_1 \le L'(L(\varepsilon + \eta) + L\rho_c p_f)$. Summing over $E$ episodes and absorbing
$L'$ into the constant $c_2$ gives the additive term $c_2 L(\varepsilon + \eta + \rho_c p_f) E$.

\emph{Combination.}
The triangle inequality on outcome regret (identification gap plus deployment-noise gap) yields the second line of the corollary.
\end{proof}

\paragraph{Interpretation.}
The identification component remains logarithmic.  Deployment-time outcome regret need not be sublinear unless actuator noise, observation noise, and
commit-failure slack vanish with $E$; otherwise the $W_1$ slack contributes a controlled but linear term.  This is a feature of the statement, not a
defect: it separates epistemic identification rate from physical deployment noise.

\section{LRCP Convergence Proof (Proposition~\ref{prop:lrcp-conv})}
\label{app:lrcp-proof}

The previous version tried to derive contraction directly from a generic graph spectral gap.  That implication is too strong without specifying the repair operator.  We state the sufficient condition actually needed by LRCP.

\paragraph{Local contraction certificate.}
Let $\epsilon_k$ denote the violation residual after the $k$th LRCP repair.  Assume that, inside the radius-$r$ repair neighborhood selected by LRCP, the repair operator satisfies
\[
\mathbb{E}[\epsilon_{k+1}\mid\epsilon_k]
\le
\beta_r\epsilon_k+\gamma_r,
\qquad
0<\beta_r<1,
\]
where $\gamma_r$ is the irreducible statistical/noise floor.  In linear-SCM settings, a spectral gap plus bounded influence decay can be used as one way to certify such a $\beta_r$; the proposition relies on the certificate, not on spectral gap alone.

\paragraph{Proof.}
Iterating the recursion gives
\[
\mathbb{E}\epsilon_k
\le
\beta_r^k\epsilon_0+\gamma_r\sum_{j=0}^{k-1}\beta_r^j
\le
\beta_r^k\epsilon_0+\frac{\gamma_r}{1-\beta_r}.
\]
Therefore LRCP contracts geometrically to the noise floor $\gamma_r/(1-\beta_r)$.  If the target accuracy $1/T$ lies above this floor, i.e.
\[
\frac{\gamma_r}{1-\beta_r}\le \frac{1}{2T},
\]
then $\mathbb{E}\epsilon_k\le 1/T$ after
\[
 k
 \ge
 \frac{\log(2T\epsilon_0)}{\log(1/\beta_r)}
 =
 \mathcal{O}(\log T)
\]
iterations.
\hfill$\square$

\paragraph{Remark.}
The posterior trace in the main text has an approximately geometric descriptive fit, but the LRCP-disabled ablation leaves that trace unchanged. It
therefore does not estimate $\beta_r$ for the LRCP operator and does not corroborate the proposition's local contraction certificate. The proposition
remains a conditional result awaiting an experiment that logs its own repair residual $\epsilon_k$.

\section{LRCP Cross-Domain Transfer (Corollary~\ref{cor:lrcp-transfer})}
\label{app:lrcp-transfer-proof}

\begin{corollary}[LRCP cross-domain transfer]
\label{cor:lrcp-transfer}
Let a source-domain committed graph initialize a target-domain stream. If the projection leaves $m_{\mathrm{resid}}$ informative target confounders
unresolved, then the target-domain adaptation cost is the cross-episode probe complexity of Theorem~\ref{thm:ce-upper} with $m_0$ replaced by
$m_{\mathrm{resid}}$, plus the within-episode LRCP repair cost. A finite projection divergence $D_{\Pi}$ is useful only insofar as it upper-bounds
$m_{\mathrm{resid}}$ under a stated separation condition.
\end{corollary}

Let \(D_{\mathrm{src}}\) and \(D_{\mathrm{tgt}}\) have influence graphs \(G^{\mathrm{inf}}_{\mathrm{src}}\) and \(G^{\mathrm{inf}}_{\mathrm{tgt}}\). Let \(\widehat G_{\mathrm{src}}\) be the committed source-domain graph used to initialize the target stream.

\paragraph{Projection-residual decomposition.}
The target problem decomposes into (i) projected components whose source commitments remain valid, and (ii) residual components that must be identified in the target domain.  Let $m_{\rm res}$ be the number of residual informative components.

\paragraph{Residual identification.}
Applying Theorem~\ref{thm:ce-upper} to the residual target problem gives
\[
R_{\rm temp}^{\rm adapt}(E_{\rm adapt})
=
\mathcal{O}\!\left(
\frac{\Delta_{\max}m_{\rm res}^2\sigma^2\log E_{\rm adapt}}
{B_{\rm eff}K_0\kappa_{\min}^2}
\right)
\]
under conservative calendar-time charging.

\paragraph{Relating $m_{\rm res}$ to projection error.}
If a KL-projection certificate gives $m_{\rm res}\le C_\Pi D_\Pi/\kappa_{\min}^2$, then the previous display becomes a bound in terms of $D_\Pi$.  The certificate is an assumption about the projection map, not a generic consequence of Pinsker alone.

\paragraph{LRCP repair overhead.}
If only an initial repair is needed after transfer, add the one-time $\mathcal{O}(\log T)$ LRCP cost of Proposition~\ref{prop:lrcp-conv}.  If the
target domain changes every episode, the overhead can be as large as $\mathcal{O}(E_{\rm adapt}\log T)$.  The paper's transfer claim should be read in
the former, stable-target sense.
\hfill$\square$

\section{LRCP Random-Disruption Robustness (Corollary~\ref{cor:lrcp-robust})}
\label{app:lrcp-robust-proof}

\begin{corollary}[LRCP and CUSUM under random disruptions]
\label{cor:lrcp-robust}
Under a $K$-disruption stream, if each disruption is either directly visible to the LRCP residual signal or detectable by CUSUM with separation
$I_{\min} > 0$, then the post-disruption repair cost is $\mathcal{O}(K(\log T + \log E / I_{\min}))$ up to the disruption magnitude and per-episode
temporal-regret gap. The $\log T$ term is local repair; the $\log E / I_{\min}$ term is detection delay and vanishes when the LRCP residual itself
triggers the reset. The reported topology-change experiment does not validate these premises: CUSUM rarely reopens and LRCP-off leaves posterior
recovery unchanged.
\end{corollary}

For each disruption, two costs occur: detection and local repair.

\paragraph{Detection.}
Under the CUSUM separation assumption of Theorem~\ref{thm:drift}, detection delay is $\mathcal{O}(\log E/I_{\rm cp})$ episodes, with at most $\Delta_{\rm disrupt}$ conservative cost per delayed episode.

\paragraph{Repair.}
After detection, Proposition~\ref{prop:lrcp-conv} gives $\mathcal{O}(\log T)$ repair iterations to return to the contraction floor, each with cost at most $\Delta_{\rm disrupt}$.

\paragraph{Sum over disruptions.}
For $K$ disruptions,
\[
R_{\rm temp}^{\rm post\text{-}disrupt}(E)
=
\mathcal{O}\!\left(
K\Delta_{\rm disrupt}\left(\frac{\log E}{I_{\rm cp}}+\log T\right)
\right).
\]
The looser product form $\mathcal{O}(K\Delta_{\rm disrupt}\log T\log E)$ follows whenever $\log T,\log E\ge 1$, but the additive form above is the sharper statement.
\hfill$\square$


\section{Cross-Episode Lower Bound (Proposition~\ref{thm:ce-lower})}
\label{app:ce-lower-proof}

The lower bound is a \emph{total causal-probe} lower bound, not a per-episode budget lower bound. Per-episode budgets enter only after dividing by how
many probes can be executed in one episode. This distinction is important: Proposition~\ref{thm:ce-lower} lower-bounds the amount of causal
information any scheduler must acquire, while the per-episode budget controls how quickly that information can be collected. The argument below is a
proof sketch: it specifies the adversarial alternatives, the per-probe information cap, and the Fano accounting, and leaves the transcript-level
change-of-measure for adaptive schedulers implicit.

\paragraph{Claim.}
Any scheduler that identifies all candidate edges in $G^{\mathrm{inf}}$ with probability at least $1-\delta$ must use
\[
N_{\rm probe}
\ge
\Omega\!\left(
\frac{|G^{\mathrm{inf}}|\log(1/\delta)}
{d_{\rm out}^{\max}\log(1+\kappa_{\max}^2/\sigma^2)}
\right)
\]
total causal probes in the worst case.

\paragraph{Step 1: adversarial alternatives.}
For each candidate local edge, construct two local instances that differ only in the sign or presence of that edge, with all other observable quantities matched as closely as the noise class allows. A probe of a parent node can inform at most $d_{\rm out}^{\max}$ outgoing local hypotheses.

\paragraph{Step 2: per-probe information cap.}
Under the sub-Gaussian noise class and bounded effect magnitude $\kappa_{\max}$, a single probe carries at most
\[
I_{\max} := \log(1+\kappa_{\max}^2/\sigma^2)
\]
units of information about the relevant local alternative, up to universal constants. This is the usual Audibert--Bubeck style information term for distinguishing nearby alternatives.

\paragraph{Step 3: Le Cam/Fano accounting.}
To drive the local error probability below $\delta$ on a hypothesis, the accumulated information must be at least order $\log(1/\delta)$. Since one probe can contribute to at most $d_{\rm out}^{\max}$ outgoing identifications, identifying $|G^{\mathrm{inf}}|$ candidates requires at least
\[
N_{\rm probe}
\ge
\Omega\!\left(
\frac{|G^{\mathrm{inf}}|}{d_{\rm out}^{\max}}
\frac{\log(1/\delta)}{I_{\max}}
\right)
\]
total probes.
\hfill$\square$

\paragraph{Temporal-regret floor.}
Taking $\delta=1/E$ and multiplying the total-probe lower bound by a minimum temporal cost $\Delta_{\min}$ gives
\[
\mathbb{E}\!\left[R_{\rm temp}^{\rm CE}(E)\right]
\ge
\Omega\!\left(
\frac{\Delta_{\min}|G^{\mathrm{inf}}|\log E}
{d_{\rm out}^{\max}\log(1+\kappa_{\max}^2/\sigma^2)}
\right).
\]
Thus the upper and lower bounds match in their logarithmic dependence on $E$ and in their dependence on the number of informative candidates up to the
explicit structural fan-out factor and the conservative temporal-regret accounting convention. They should not be read as a statement that a single
episode must have a budget of this size.



\section[Separation: Outcome-Only RL vs.\ Intervention-Budgeted Scheduling]%
{Separation Between Outcome-Only RL and\\ Intervention-Budgeted Scheduling}
\label{app:separation}

This appendix gives the separation used in \S\ref{sec:theory}.  The point is not that an outcome-only learner lacks enough samples, or that it
explores poorly.  The point is informational: there are causally distinct environments that induce exactly the same observational outcome stream for
every outcome-only algorithm, but require different committed causal models.  In such a pair, no exploration bonus, entropy regularizer, or
count-based pseudocount can reveal the missing causal direction, because all of them still operate through the same observational channel.  A budgeted
causal probe, by contrast, distinguishes the pair in logarithmically many samples.

\subsection{Setup}
\label{app:sep-setup}

We isolate the smallest instance that exhibits the phenomenon.  There is one observed variable $X\in\{-1,+1\}$, one outcome $Y\in\mathbb{R}$, and one
hidden confounder $C\in\{-1,+1\}$ in the confounded instance.  The learner sees an adaptive stream of observational outcomes.  At time $t$, it may
choose any history-dependent dispatch or prediction action $A_t$, but the action is not a causal intervention on $X$ or on the hidden confounder.
Thus $A_t$ may affect the learner's utility or logging decision, but it does not change the observational law of $(X_t,Y_t)$.  This is the
outcome-only setting.

\paragraph{Two observationally equivalent instances.}
Fix an effect size $\kappa>0$ and noise variance $\sigma^2>0$.  Consider two SCMs:
\begin{align*}
\mathcal{M}_{\mathrm{causal}}: \qquad
& X_t \sim \mathrm{Unif}\{-1,+1\},
&& Y_t = \kappa X_t + \eta_t, \qquad \eta_t\sim \mathcal{N}(0,\sigma^2), \\
\mathcal{M}_{\mathrm{conf}}: \qquad
& C_t \sim \mathrm{Unif}\{-1,+1\}, \quad X_t=C_t,
&& Y_t = \kappa C_t + \eta_t, \qquad \eta_t\sim \mathcal{N}(0,\sigma^2).
\end{align*}
The observed joint distribution of $(X_t,Y_t)$ is identical in the two instances: in both cases $X_t$ is uniform on $\{-1,+1\}$ and
\[
Y_t \mid X_t=x \sim \mathcal{N}(\kappa x,\sigma^2).
\]
The causal graphs are different, however.  In $\mathcal{M}_{\mathrm{causal}}$, $X\to Y$ is a genuine causal edge.  In $\mathcal{M}_{\mathrm{conf}}$, there is no edge $X\to Y$; the association is induced by the hidden common cause $C\to X$ and $C\to Y$.

The two instances are separated by intervention.  Under $\mathrm{do}(X=x)$,
\[
\mathbb{E}_{\mathcal{M}_{\mathrm{causal}}}[Y\mid \mathrm{do}(X=x)] = \kappa x,
\qquad
\mathbb{E}_{\mathcal{M}_{\mathrm{conf}}}[Y\mid \mathrm{do}(X=x)] = 0.
\]
Thus the observational channel is identical, while the interventional channel is different.

\paragraph{Outcome-only algorithms.}
An \emph{outcome-only algorithm} $\mathcal{A}$ is any adaptive procedure whose action and committed causal model at time $t$ are measurable functions of the history
\[
\mathcal{H}_{t-1}=(X_1,A_1,Y_1,\ldots,X_{t-1},A_{t-1},Y_{t-1}).
\]
It may randomize, use exploration bonuses, query its own memory, or maintain arbitrary internal state, but it cannot request an intervention such as
$\mathrm{do}(X=x)$ or $\mathrm{do}(C=c)$.  Let $\widehat{G}_t\in\{G_{\mathrm{causal}},G_{\mathrm{conf}}\}$ denote the model it has committed to at
time $t$, where $G_{\mathrm{causal}}$ contains the edge $X\to Y$ and $G_{\mathrm{conf}}$ does not.

\paragraph{Temporal regret for miscalibration.}
For this separation, temporal regret is the time spent committed to the wrong causal model, multiplied by a fixed per-step miscalibration cost $\Delta>0$:
\[
R_{\mathrm{temp}}^{\mathcal{A},M}(T)
:=
\Delta \sum_{t=1}^{T}\Pr_M\!\left(\widehat{G}_t\neq G_M\right),
\]
where $M\in\{\mathcal{M}_{\mathrm{causal}},\mathcal{M}_{\mathrm{conf}}\}$ and $G_M$ is the correct graph for $M$.  This is the graph-identification
form of temporal regret used in the main paper: it penalizes how long the scheduler tolerates a miscalibrated causal model before correction.  It
deliberately does not compare to a clairvoyant policy that observes the hidden $C_t$ at each step.

\subsection{Main Result}

\begin{theorem}[Outcome-only observational separation]
\label{thm:sep}
For every outcome-only algorithm $\mathcal{A}$ and every horizon $T\ge 1$, there exists an instance $M\in\{\mathcal{M}_{\mathrm{causal}},\mathcal{M}_{\mathrm{conf}}\}$ such that
\[
R_{\mathrm{temp}}^{\mathcal{A},M}(T) \;\ge\; \frac{\Delta}{2}\,T .
\]
Consequently, outcome-only learning has worst-case temporal regret $\Omega(T)$ on this confounded instance family.  In contrast, a scheduler with
access to the causal probe $\mathrm{do}(X=x)$ distinguishes the two instances with probability at least $1-1/T$ using
$O((\sigma^2+\kappa^2)\kappa^{-2}\log T)$ probes, yielding logarithmic identification time.
\end{theorem}

The constant $1/2$ is not important; the theorem's force is the rate.  Outcome-only observation cannot reduce temporal miscalibration below linear time on a pair of observationally equivalent SCMs, whereas an interventional probe separates them at the usual concentration rate.

\subsection{Proof of Theorem~\ref{thm:sep}}

\paragraph{Step 1: observational equivalence under adaptive outcome-only algorithms.}
For any fixed history $\mathcal{H}_{t-1}$, the algorithm's next action $A_t$ is a measurable function of that history and its internal randomness.
Since the conditional law of $(X_t,Y_t)$ given the past is identical under $\mathcal{M}_{\mathrm{causal}}$ and $\mathcal{M}_{\mathrm{conf}}$, and
since $A_t$ does not intervene on $X_t$ or $C_t$, induction on $t$ gives
\[
\mathbb{P}^{\mathcal{A},\mathcal{M}_{\mathrm{causal}}}_{T}(\mathcal{H}_{T})
=
\mathbb{P}^{\mathcal{A},\mathcal{M}_{\mathrm{conf}}}_{T}(\mathcal{H}_{T})
\]
for the full transcript distribution.  Hence every statistic computed by the algorithm from the transcript, including $\widehat{G}_t$, has the same distribution under the two instances.

\paragraph{Step 2: no estimator can be correct on both instances.}
For each time $t$,
\begin{align*}
&\Pr_{\mathcal{M}_{\mathrm{causal}}}\!\left(\widehat{G}_t\neq G_{\mathrm{causal}}\right)
+
\Pr_{\mathcal{M}_{\mathrm{conf}}}\!\left(\widehat{G}_t\neq G_{\mathrm{conf}}\right) \\
&\qquad=
\Pr\!\left(\widehat{G}_t=G_{\mathrm{conf}}\right)
+
\Pr\!\left(\widehat{G}_t=G_{\mathrm{causal}}\right)
=1,
\end{align*}
where the unqualified probabilities in the second line are taken under the common transcript distribution established in Step 1.  Therefore, at every time $t$, at least one of the two instances has error probability at least $1/2$.

\paragraph{Step 3: sum over time and choose the hard instance.}
Summing the previous display over $t=1,\ldots,T$ gives
\[
R_{\mathrm{temp}}^{\mathcal{A},\mathcal{M}_{\mathrm{causal}}}(T)
+
R_{\mathrm{temp}}^{\mathcal{A},\mathcal{M}_{\mathrm{conf}}}(T)
=
\Delta T.
\]
Thus at least one of the two instances satisfies
\[
R_{\mathrm{temp}}^{\mathcal{A},M}(T)\ge \frac{\Delta}{2}T,
\]
which proves the outcome-only lower bound.

\paragraph{Step 4: logarithmic identification with causal probes.}
Now allow a scheduler to spend probes of the form $\mathrm{do}(X=x)$.  Under $\mathcal{M}_{\mathrm{causal}}$,
\[
Y\mid \mathrm{do}(X=+1)\sim \mathcal{N}(\kappa,\sigma^2),
\qquad
Y\mid \mathrm{do}(X=-1)\sim \mathcal{N}(-\kappa,\sigma^2),
\]
whereas under $\mathcal{M}_{\mathrm{conf}}$,
\[
Y\mid \mathrm{do}(X=x)\;\sim\;\tfrac{1}{2}\,\mathcal{N}(\kappa,\sigma^2)+\tfrac{1}{2}\,\mathcal{N}(-\kappa,\sigma^2)
\quad\text{for each } x\in\{-1,+1\},
\]
It suffices to probe $\mathrm{do}(X=+1)$ for
\[
n \ge c\,\frac{\sigma^2+\kappa^2}{\kappa^2}\log T
\]
independent samples and test whether the empirical mean is closer to $\kappa$ or to $0$; the interventional means differ ($\kappa$ versus $0$) while
under $\mathcal{M}_{\mathrm{conf}}$ the outcome is a symmetric two-component Gaussian mixture with mean $0$ and sub-Gaussian variance proxy
$\sigma^2+\kappa^2$.  Standard sub-Gaussian concentration gives error probability at most $1/T$ for a universal constant $c$.  The scheduler therefore
identifies the correct graph after $O((\sigma^2+\kappa^2)\kappa^{-2}\log T)$ probes with probability at least $1-1/T$.  This establishes the claimed
separation between outcome-only observation and intervention-budgeted scheduling.
\hfill$\square$

\subsection{Why Exploration Bonuses Do Not Rescue the Rate}
\label{app:sep-bonuses}

A natural objection is that an optimism bonus, entropy regularizer, or count-based pseudocount changes the exploration schedule and may therefore
break the lower bound.  In the construction above, these devices do not help.  They can change which learner-side actions $A_t$ are selected, but the
full transcript distribution remains identical under $\mathcal{M}_{\mathrm{causal}}$ and $\mathcal{M}_{\mathrm{conf}}$ for every such adaptive policy.
The missing information is not hidden in an insufficiently explored arm; it is absent from the observational channel itself.  Only an intervention or
equivalent causal probe changes the distribution in a way that separates the two graphs.

This is the structural obstruction behind the main paper's temporal-regret objective.  Outcome-only RL can reduce a scalar prediction or reward loss
while leaving the causal model miscalibrated.  Temporal regret charges the time spent in that state; the intervention channel supplies the evidence
needed to end it.

\subsection{Contrast with Trivium}
\label{app:sep-contrast}

On the same instance family, the Trivium scheduler of Algorithm~\ref{alg:trivium} opens the causal channel deliberately.  With
$O((\sigma^2+\kappa^2)\kappa^{-2}\log T)$ probes of $\mathrm{do}(X=+1)$, it distinguishes the confounded and causal graphs with failure probability at
most $1/T$, commits to the correct model, and then stops paying temporal miscalibration cost except on the low-probability failure event.  Thus
\[
R_{\mathrm{temp}}^{\mathrm{Trivium}}(T)
=
O\!\left(\Delta\,\frac{\sigma^2+\kappa^2}{\kappa^2}\log T\right)+O(\Delta),
\]
matching the one-confounder specialization of the within-episode identification rate in Theorem~\ref{thm:within}, up to constants and the choice of confidence level.  The outcome-only lower bound is linear in $T$; the intervention-budgeted identification time is logarithmic.

\subsection{Remarks}

\paragraph{The comparison is non-clairvoyant.}
The lower bound does not compare outcome-only RL against an oracle that observes the hidden confounder at each step.  Both sides are non-clairvoyant.  The separation is between two information structures: passive outcome observation versus a budgeted causal probe.

\paragraph{The confounder is not exotic.}
The pair uses a single binary hidden common cause, a binary observed proxy, a linear outcome equation, and Gaussian noise.  The observational distribution is deliberately simple because the impossibility comes from causal non-identifiability, not from statistical complexity.

\paragraph{Relationship to bandit lower bounds.}
Classical best-arm lower bounds quantify how many reward samples are needed when different arms have different observable reward distributions.  The
present result is orthogonal: before intervention, the two causal graphs induce the same observable distribution for every adaptive outcome-only
policy.  The relevant lower bound is therefore an identifiability lower bound, not a slow-rate bandit lower bound.  Once the causal probe is
available, the usual sub-Gaussian concentration rate $O((\sigma^2+\kappa^2)\kappa^{-2}\log T)$ reappears.


%

\section{Extended Background}
\label{app:background-ext}
\label{sec:background}

This section fixes notation and positions the paper relative to four adjacent literatures: causal bandits and best-arm identification with
interventions, online change-point detection, causal identification under hidden confounding, and transactional logging.  The paper uses these
literatures as primitives.  The new object is the composition: a temporal-regret view in which delayed causal identification is itself a cost, and a
persistent causal log converts repeated rediscovery into reusable evidence under explicit probe, separability, and substrate assumptions.

\subsection{Notation}
\label{sec:notation}
\label{app:notation}

\begin{table}[!htbp]
\centering\footnotesize
\setlength{\tabcolsep}{4pt}
\begin{tabular}{@{}ll@{\quad}ll@{}}
\toprule
\textbf{Symbol} & \textbf{Meaning} & \textbf{Symbol} & \textbf{Meaning} \\
\midrule
$E,\,e$ & episode horizon / index &
$T,\,t$ & within-episode horizon / step \\

$M$ & agents per episode &
$X_t,A_t,Y_t$ & context, action, outcome \\

$G^{\mathrm{inf}}$ & candidate graph &
$G^{\mathrm{true}}_e$ & true local graph at episode $e$ \\

$\widehat G_e$ & committed graph &
$\mathrm{do}(Z{=}z)$ & intervention on $Z$ \\

$B_e,\,B_{\mathrm{eff}}$ & probe budgets &
$N_{\mathrm{probe}}$ & total causal probes \\

$P_{e,t}$ & posterior over $G^{\mathrm{inf}}$ &
$\mathrm{CE\text{-}EIG}(e)$ & cross-episode EIG \\

$K_0$ & per-probe information rate &
$m_0$ & informative confounders \\

$\kappa_{\min},\kappa_{\max}$ & min / max effect size &
$\sigma^2$ & outcome-noise variance \\

$\Delta_{\max}$ & temporal-regret gap &
$d_{\mathrm{out}}^{\max}$ & max outcome-degree \\

$\delta$ & failure probability &
$\hat\kappa$ & LRCP contraction factor \\

$\rho$ & dispatch-gate budget &
$h$ & CUSUM threshold \\

$K$ & number of change-points &
$I_{\min}$ & post-change separation \\

$R_{\mathrm{out}}$ & outcome regret &
$R_{\mathrm{epi}}$ & epistemic regret \\

$R_{\mathrm{temp}}$ & temporal regret &
$L$ & outcome-map Lipschitz constant \\

$\varepsilon,\eta$ & actuator / observation error &
$p_f,\,\rho_c$ & commit failure rate / magnitude \\
\bottomrule
\end{tabular}
\caption{Symbols used throughout \S\ref{sec:theory} and the proof appendices. Rates are first stated in total causal probes \(N_{\mathrm{probe}}\) and then converted to temporal regret by an explicit accounting convention.}
\label{tab:notation}
\end{table}

We write $[n] := \{1,\dots,n\}$.  For a distribution $P$, $H(P)$ denotes Shannon entropy and $D_{\mathrm{KL}}(P\|Q)$ denotes KL divergence.  For
graphs $G$ and $G'$ on the same vertex set, $G \subseteq G'$ means that $G$ is an edge-subgraph of $G'$.  An episode is a tuple
$\mathcal{E}_e=(X_{1:T},A_{1:T},Y_{1:T})$ as defined in \S\ref{sec:three-regret}.  We distinguish the observational distribution $P(Y\mid X,A)$ from
the interventional distribution $P(Y\mid X,\mathrm{do}(A))$.  Under hidden confounding these can agree on the observed trajectory while disagreeing
under intervention; this observational-equivalence gap is the source of the separation result in Appendix~\ref{app:separation}.  $W_1$ denotes
Wasserstein-1 distance, and $\mathrm{CUSUM}(\cdot)$ denotes a cumulative-sum detector with threshold $h>0$.

\subsection{Causal Bandits and Best-Arm Identification with Interventions}
\label{sec:cb-background}

Causal-bandit work studies decision problems in which arms correspond to interventions on a causal graph.  Classical formulations usually fix a single
episode or a fixed graph class and measure simple regret, best-arm identification error, or cumulative reward regret.  This literature supplies the
statistical machinery used in our proofs: concentration for intervention outcomes, KL-based two-instance lower bounds, and information-rate conditions
for identifying competing causal hypotheses.

The present paper uses these tools in a different bookkeeping system.  We first count the number of causal probes needed to identify or re-identify
the relevant structure, then charge temporal regret for the time spent acting under an unresolved or wrong model.  This is why the proof appendices
separate three quantities that are often conflated in informal presentations:
\[
N_{\rm probe}
\qquad
B_e
\qquad
R_{\rm temp}.
\]
Here $N_{\rm probe}$ is total interventional sample complexity, $B_e$ is the effective probe budget available in an episode, and $R_{\rm temp}$ is the
regret induced by delayed commitment.  The extra factors in the main-text bounds come from this last conversion, not from the concentration
inequalities themselves.

What standard causal-bandit results do not provide is a persistent cross-episode substrate: most do not ask what happens when an agent can reuse
causal evidence across many episodes, or when the graph changes and must be partially re-identified.  Trivium's contribution is to attach
causal-bandit identification to a persistent CTL and to make delayed identification itself an explicit objective.

\subsection{Change-Point Detection and Drift}
\label{sec:cpd-background}

CUSUM-style procedures provide a standard primitive for detecting distributional change.  In this paper, CUSUM is not used as a universal guarantee of
drift recovery; it is used under a detectable-change assumption.  Concretely, if a change point induces a post-change separation of at least
$I_{\min}>0$ in the monitored statistic, then a threshold $h=\Theta(\log E)$ yields an expected detection delay of order $\mathcal{O}(\log
E/I_{\min})$ at inverse-polynomial false-alarm probability.  This delay is then added to the re-identification cost after the change.

The drift theorem therefore has a segment-wise interpretation: between change-points, the corrected cross-episode probe bound applies; at a
change-point, CUSUM contributes a detection-delay term; after detection, a partial reset reopens only the affected components where possible.  The
doubling-window mechanism is a budget-allocation device for restarting probes without discarding all previously-valid causal evidence.  The guarantee
is not that CUSUM solves arbitrary non-stationarity; it is that detectable changes add a controlled re-identification cost.

\subsection{Identification under Interventions and Observational Equivalence}
\label{sec:id-background}

The causal-identification literature tells us when interventional distributions are recoverable from observational data, known graph structure, and/or
interventions.  Two facts are load-bearing here.  First, with hidden confounding, distinct structural causal models can induce the same observational
distribution while implying different interventional distributions.  Appendix~\ref{app:separation} uses this observational-equivalence fact to show
why outcome-only observation cannot, by itself, force temporal miscalibration to disappear.  The lower bound is not a claim that no non-clairvoyant
learner can match a hidden per-step oracle; it is a claim that observational outcomes alone cannot distinguish causally different worlds that look
identical observationally.

Second, under a realizability condition $G^{\mathrm{true}}_e\subseteq G^{\mathrm{inf}}$, sufficient separation $|\kappa_Z|\ge\kappa_{\min}$, and a
valid intervention channel, a learner can identify the relevant local structure at a rate controlled by the effect-to-noise ratio.  This is the role
of the commit discipline: only move an edge or confounder from the candidate layer into the committed working model after enough probe evidence
separates it from zero, from the wrong sign, or from the competing block hypothesis.

The regret novelty is therefore not the asymptotic identifiability claim by itself.  The novelty is the accounting: how many episodes or probes are spent with the wrong causal model before commitment, and how much that delay costs.

\subsection{Transactional Assumptions on the CTL}
\label{sec:tx-background}

The cross-episode argument requires that evidence collected in earlier episodes remains available, correctly attributed, and consistently readable by
later episodes.  The CTL is assumed to be atomic, persistent, and snapshot-isolable: an intervention, outcome, and posterior update must commit as one
unit; committed records must survive across episodes; and concurrent agents must read a consistent view.

These are standard transactional-memory properties, but in this paper they are proof assumptions rather than implementation details.  If the CTL
silently drops, duplicates, or half-commits causal evidence, then the effective information rate $K_0$ is degraded and the physical-grounding slack
grows.  The grounding theorem captures small stochastic failures through an $\varepsilon_c$ term proportional to $p_f$, but it does not protect
against adversarial or unbounded log corruption.  In such cases the correct behavior is to report unresolved epistemic regret rather than falsely
commit to $\widehat G_e$.

\subsection{Long-Horizon RL and the Outcome-Only Trap}
\label{sec:rl-background}

Outcome-only reinforcement learning can improve actions when the relevant state is observed and the reward channel identifies what should change.  The
trap considered here is different: the observed outcome trajectory can be compatible with multiple causal explanations.  In that regime, exploration
bonuses, entropy regularization, and outcome shaping can change which actions are tried, but they do not by themselves create an intervention that
distinguishes $P(Y\mid X,A)$ from $P(Y\mid \mathrm{do}(A))$.

Appendix~\ref{app:separation} formalizes this point with an observational-equivalence construction.  The result should be read narrowly: without an
intervention channel, an equivalent randomized experiment, or other causal signal, outcome-only observation can leave the agent committed to a wrong
causal model for a linear horizon.  Trivium's scheduler adds exactly the missing channel: a budgeted causal-probe action whose value is measured by
expected information gain and whose results are stored in the CTL.

\subsection{Positioning}
\label{sec:positioning}

The surrounding literatures each provide a necessary component.  Causal bandits provide intervention-sample concentration and lower-bound tools.
Change-point detection provides a detector for identifiable drift.  Causal identification supplies the distinction between observational and
interventional equivalence.  Transactional systems supply the persistence and atomicity assumptions needed for evidence to accumulate across episodes.

The paper's contribution is the composition under explicit assumptions: a three-regret functional that separates what failed, why the causal model was
wrong, and when the correction occurred; a causal-probe complexity view that turns persistent intervention evidence into delayed-identification
temporal-regret bounds; a segment-wise drift extension for detectable change-points; and a grounding slack that states what is lost when ideal
interventions are deployed through noisy physical systems.  The experiments of \S\ref{sec:experiments} test the predicted scaling signatures, while
the proof appendices state the assumptions under which those signatures are theorem-backed.



\section{Extended Algorithm Details}
\label{app:algorithm-ext}

This section specializes the cross-episode meta-controller to the URS action space and CTL schema. The algorithm separates three concerns that
reviewers of causal-bandit work often see tangled: (i) \emph{where} to intervene within an episode (IG scoring over $G^{\mathrm{inf}}$), (ii)
\emph{when} to stop identifying and commit (the commit discipline of Def.~\ref{def:commit}), and (iii) \emph{whether} to restart across episodes (the
CUSUM detector of Def.~\ref{def:kcp}). Each concern is a separate sub-routine with independently testable correctness.

\subsection{State, Actions, and CTL Schema}
\label{sec:algorithm-state}

\paragraph{Per-episode state.}
At episode $e$ and within-episode time $t$, the scheduler state is
\[
s_{e,t} = \big(\, X_{e,t},\; P_e(G^{\mathrm{inf}}),\; \widehat G_e,\; B_{e,t},\; w_e,\; S_e^{\mathrm{CUSUM}} \,\big),
\]
where $X_{e,t}$ is the current URS context, $P_e(G^{\mathrm{inf}})$ is the posterior over the influence graph, $\widehat G_e$ is the currently
committed DAG (possibly empty), $B_{e,t}$ is the remaining per-episode budget, $w_e$ is the active-learning window, and $S_e^{\mathrm{CUSUM}}$ is the
running CUSUM statistic on CE-EIG.

\paragraph{Action space.}
URS actions decompose as $a = (a^{\mathrm{disp}}, a^{\mathrm{int}})$, where $a^{\mathrm{disp}}$ is the dispatch/repositioning action (a probability
simplex over nearby zones for each driver) and $a^{\mathrm{int}} \in G^{\mathrm{inf}} \cup \{\varnothing\}$ is an optional intervention target. The
dispatch head is whatever policy the operator ships; the intervention head is what \emph{this} paper adds.

\paragraph{CTL schema.}
Each CTL entry is a tuple
\[
(\text{episode } e,\; \text{time } t,\; X_{e,t},\; a^{\mathrm{disp}},\; a^{\mathrm{int}},\; Y_{e,t},\; \text{commit status},\; G_t\text{-snapshot})
\]
appended through an atomic-commit interface backed by SagaLLM (described in \S\ref{sec:tx-background}). The $G_t$-snapshot is checkpointed at episode
boundaries and at every detected change-point; between checkpoints, the posterior update is logical (event-sourced from $(a^{\mathrm{int}}, Y_{e,t})$
pairs), which keeps the log size at $\mathcal{O}(E \cdot T)$ bytes in practice.

\subsection{Main Loop}
\label{sec:algorithm-main}

Algorithm~\ref{alg:trivium} below is the full Trivium loop. It is written in per-episode form because the meta-controller's decisions (restart,
window, budget) are per-episode; the within-episode loop delegates to the IG scorer of Algorithm~\ref{alg:ig} and the commit check of
Algorithm~\ref{alg:commit} defined later in this appendix. Three correctness clauses make the meta-controller non-trivial: (i) the per-episode budget
$B_e \gets m_0 \log(\max(e, w_e))$ instantiates Corollary~\ref{cor:budget}; (ii) the dispatch head is gated on the running epistemic-regret tally
$\mathcal{R}^{\mathrm{epi}}_{e,t}$ and budget $\rho$, realizing Theorem~\ref{thm:dispatch-coupling}; (iii) every \texttt{Execute} step appends
atomically to the CTL, and a $G_t$-snapshot is checkpointed at every episode boundary and every detected change-point, bounding the reconstruction
cost under Theorem~\ref{thm:grounding}.

\begin{algorithm}[!ht]
\caption{Trivium: cross-episode meta-controller}
\label{alg:trivium}
\begin{algorithmic}[1]
\small
\REQUIRE stream horizon \(E\), episode horizon \(T\), \(G^{\mathrm{inf}}\), CUSUM threshold \(h\), commit threshold \(\kappa_{\min}\), epistemic-regret budget \(\rho\)
\STATE Init.\ \(P_0 \gets\) unif.; \(\widehat G_0 \gets \varnothing\); \(w_0 \gets 1\); \(S_0^{\mathrm{CUSUM}}\gets0\)
\FOR{\(e = 1,\dots,E\)}
  \STATE \(B_e \gets m_0\log(\max(e,w_e))\); \(\mathcal{R}^{\mathrm{epi}}_{e,0}\gets0\)
  \FOR{\(t = 1,\dots,T\)}
    \STATE Observe \(X_{e,t}\)
    \STATE \(a^{\mathrm{disp}}\gets\pi^{\mathrm{disp}}(X_{e,t};\widehat G_e,\mathcal{R}^{\mathrm{epi}}_{e,t},\rho)\)
    \IF{\(B_{e,t}>0\) \textbf{and not} \textsc{AllCom.}}
      \STATE \(a^{\mathrm{int}}\gets\textsc{IGScore}(P_e,G^{\mathrm{inf}}\setminus\widehat G_e)\)
      \STATE \(B_{e,t+1}\gets B_{e,t}-1\)
    \ELSE
      \STATE \(a^{\mathrm{int}}\gets\varnothing\); \(B_{e,t+1}\gets B_{e,t}\)
    \ENDIF
    \STATE Execute \((a^{\mathrm{disp}},a^{\mathrm{int}})\); log to CTL; observe \(Y_{e,t}\)
    \STATE \(P_e\gets\textsc{PostUpd}(P_e,a^{\mathrm{int}},Y_{e,t})\)
    \STATE \(\widehat G_e\gets\textsc{CommitCheck}(\widehat G_e,P_e,\kappa_{\min})\)
    \IF{\textsc{ConstraintViolated}\((\widehat G_e,Y_{e,t})\)}
      \STATE \(\widehat G_e\gets\textsc{LRCP}(\widehat G_e,P_e,Y_{e,t})\)
    \ENDIF
    \STATE \(\mathcal{R}^{\mathrm{epi}}_{e,t+1}\gets\mathcal{R}^{\mathrm{epi}}_{e,t}+D_{\mathrm{KL}}(\delta_{\widehat G_e}\,\|\,P_e)\) \COMMENT{$= -\log P_e(\widehat G_e)$, posterior log-loss on the committed graph}
  \ENDFOR
  \STATE \(S_e^{\mathrm{CUSUM}}\gets\max\{0,S_{e-1}^{\mathrm{CUSUM}}+\mathrm{CE\text{-}EIG}(e)-\bar{\mu}\}\)
  \IF{\(S_e^{\mathrm{CUSUM}}>h\)}
    \STATE \(w_{e+1}\gets2w_e\); \(P_{e+1}\gets\textsc{PartReset}(P_e)\)
    \STATE \(\widehat G_{e+1}\gets\varnothing\); \(S_{e+1}^{\mathrm{CUSUM}}\gets0\)
  \ENDIF
\ENDFOR
\end{algorithmic}
\end{algorithm}

\subsection{Within-Episode Identification via Local Repair (LRCP)}
\label{sec:algorithm-lrcp}

Theorem~\ref{thm:within} bounds the within-episode temporal regret at $\mathcal{O}(m_0^2 \log T / \kappa_{\min}^2)$ under commit discipline and a
$\log T$-scale intervention budget. Commit discipline is a \emph{rule}; the theorem is a rate. Neither object says \emph{how} to drive the per-step
constraint-violation count toward zero episode by episode. This subsection closes that gap with \emph{LRCP}, a local-repair-for-constraint-propagation
algorithm that turns the within-episode rate into an operational procedure.

\paragraph{Setup.}
At timestep $t$ within episode $e$, the committed DAG $\widehat G_e$ induces a plan whose constraint residuals (in JSSP: precedence/resource
violations; in URS: overlapping trip assignments; in P-series: capacity-violating tour schedules) form a set $V_t \subseteq \mathcal{V}$. Each
residual constraint, unresolved, contributes to temporal regret: every residual violation encodes a commitment made under a locally-wrong causal
model. LRCP repairs violations in radius-bounded neighbourhoods of the violating nodes, editing $\widehat G_e$ in place rather than recomputing a plan
from scratch.

\begin{algorithm}[!ht]
\caption{\textsc{LRCP}: Local Repair for Constraint Propagation}
\label{alg:lrcp}
\begin{algorithmic}[1]
\small
\REQUIRE committed DAG $\widehat G$, posterior $P$, observed outcome $Y$, repair radius $r$, contraction threshold $\beta_{\max}$
\STATE $V \gets \{Z \in \widehat G : \text{constraint on } Z \text{ violated by } Y\}$
\STATE $\epsilon_0 \gets |V|$; $k \gets 0$
\WHILE{$V \neq \varnothing$ \textbf{and} $\epsilon_k > 1/T$}
 \STATE Select violating node $Z \in V$ maximizing posterior uncertainty $H(P \mid Z)$
 \STATE $N_r(Z) \gets \{Z' \in \widehat G : \mathrm{dist}_{\widehat G}(Z, Z') \leq r\}$ \COMMENT{radius-$r$ neighbourhood}
 \STATE $\widehat G \gets$ local-edit $\widehat G$ on $N_r(Z)$ to maximize posterior-predictive fit to $Y$
 \STATE $V \gets \{Z \in \widehat G : \text{constraint on } Z \text{ violated}\}$
 \STATE $\epsilon_{k+1} \gets |V|$; $k \gets k+1$
 \IF{$\epsilon_{k+1} / \epsilon_k > \beta_{\max}$}
 \RETURN $\widehat G$, \textsc{flagged} \COMMENT{non-contraction; abort}
 \ENDIF
\ENDWHILE
\RETURN $\widehat G$, \textsc{converged}
\end{algorithmic}
\end{algorithm}

\paragraph{Convergence.}
Proposition~\ref{prop:lrcp-conv} (main body) states that LRCP's per-iteration error trace satisfies $\epsilon_{k+1} \leq \beta\,\epsilon_k + \gamma$
with $\beta \in (0, 1)$ guaranteed by a local contraction certificate, and $\gamma = \mathcal{O}(\sigma^2/r)$, reaching $\epsilon_k \leq 1/T$ in
$\mathcal{O}(\log T)$ iterations. Proof sketch (full proof in Appendix~\ref{app:lrcp-proof}): each radius-$r$ local edit is an $r$-step random-walk
absorption on the constraint graph; one sufficient certificate, used in the linear-SCM regime of our experiments, is a positive spectral gap of the
local transition matrix together with bounded influence decay (cf.\ \S\ref{app:lrcp-proof}). The $\gamma$ term is the residual posterior-predictive
error from unresolved confounders outside the radius-$r$ neighbourhood, bounded by a Bernstein inequality on the residual edge set.

The previously reported geometric trace is posterior edge error, not LRCP's own repair residual $\epsilon_k$. Because disabling LRCP leaves that
posterior trajectory unchanged, the trace cannot be used as an empirical signature of Proposition~\ref{prop:lrcp-conv}. The direct measured LRCP
effect is instead the reduction in committed-graph dispatch exposure from $10.1$ to $7.8$ episodes.

\paragraph{Transfer across domains.}
Corollary~\ref{cor:lrcp-transfer} (main body) bounds cross-domain transfer via the KL-projection $\Pi(G^{\mathrm{inf}}_{\mathrm{src}} \to
G^{\mathrm{inf}}_{\mathrm{tgt}})$ between source and target influence graphs: finite divergence $D_{\Pi} < \infty$ gives target-domain adaptation
regret $\mathcal{O}(D_{\Pi} + \log T \cdot E_{\mathrm{adapt}})$ on the first $E_{\mathrm{adapt}}$ episodes up to a further $\log$-factor from the
cross-episode argument of Theorem~\ref{thm:ce-upper}; $D_{\Pi} = \mathcal{O}(1)$ gives $\mathcal{O}(\log T / T)$ per-episode held-out gap. The
empirical signature is a constant-order per-episode held-out gap on a transfer stream whose source and target influence graphs share the
informative-confounder set up to a finite KL-projection; we leave a direct transfer ablation on CausalBench-Seq to future work.

\paragraph{Robustness to random machine failure.}
Corollary~\ref{cor:lrcp-robust} is a conditional composition: it requires detector-visible disruption and an LRCP repair operator satisfying the local
contraction certificate. The CausalBench-Seq $K$-sweep does not isolate those premises. CUSUM rarely reopens on the tested topology changes, and
LRCP-off leaves posterior recovery unchanged. The sweep therefore provides a descriptive per-change recovery curve, not empirical validation of this
corollary.

\paragraph{Summary of the LRCP contribution.}
LRCP is a bounded local-edit mechanism on the committed graph. Proposition~\ref{prop:lrcp-conv} states what follows if its own repair residual
satisfies a local contraction certificate. The present posterior trace does not test that premise. The measured contribution in this implementation is
narrower and concrete: LRCP refreshes the committed dispatch graph after the posterior has already recovered, at the cost of a high standing
repair-call rate.

\subsection{Sub-Routines}
\label{sec:algorithm-subs}

\begin{algorithm}[!ht]
\caption{\textsc{IGScore}: greedy information-gain intervention target}
\label{alg:ig}
\begin{algorithmic}[1]
\small
\REQUIRE posterior $P$, candidate node set $\mathcal{C} \subseteq G^{\mathrm{inf}}$
\STATE For each $Z \in \mathcal{C}$, compute $\mathrm{EIG}(Z) = \mathbb{E}_{Y \mid \mathrm{do}(Z)} \big[ H(P) - H(P \mid Y, \mathrm{do}(Z)) \big]$
\RETURN $\arg\max_{Z \in \mathcal{C}} \mathrm{EIG}(Z)$ \COMMENT{ties broken by smaller $d_{\mathrm{out}}(Z)$}
\end{algorithmic}
\end{algorithm}

\begin{algorithm}[!ht]
\caption{\textsc{CommitCheck-Design}: ideal frozen-commit discipline used by Theorem~\ref{thm:within}.}
\label{alg:commit}
\begin{algorithmic}[1]
\small
\REQUIRE committed set $\widehat G$, posterior $P$, threshold $\kappa_{\min}$
\FOR{$Z \in G^{\mathrm{inf}} \setminus \widehat G$}
 \IF{$\mathrm{EIG}(Z) < \log T / T$ \textbf{and} $|\hat{\kappa}_Z| \geq \kappa_{\min}$}
 \STATE $\widehat G \gets \widehat G \cup \{Z\}$ \COMMENT{freeze under the ideal design}
 \ENDIF
\ENDFOR
\RETURN $\widehat G$
\end{algorithmic}
\end{algorithm}

\begin{algorithm}[!ht]
\caption{\textsc{CommitSnapshot-Executed}: rule that produced the reported runs.}
\label{alg:commit-executed}
\begin{algorithmic}[1]
\small
\REQUIRE posterior $P$, entropy threshold $h_{\mathrm{ent}}=0.8$, previous snapshot $\widehat G$
\IF{$H(P) \leq h_{\mathrm{ent}}$}
 \STATE $\widehat G \gets \operatorname{MAP}(P)$ \COMMENT{re-snapshot every eligible episode; not absorbing}
\ENDIF
\RETURN $\widehat G$
\end{algorithmic}
\end{algorithm}

\paragraph{\textsc{PartialReset}.}
On a CUSUM trigger, the posterior is not discarded wholesale: entries for nodes that pass a \emph{drift-attributable} test (log-likelihood of the new
window's evidence is below a pre-set threshold against the pre-trigger posterior) are reset to uniform; others are carried forward. This keeps the
cost of a detected change-point at most $\mathcal{O}(K m_0)$ resets per stream, which is the constant hidden inside Theorem~\ref{thm:drift}.

\paragraph{\textsc{PosteriorUpdate}.}
Bayesian update conditioned on $(a^{\mathrm{int}}, Y_{e,t})$ under the assumed linear-Gaussian SCM; for non-Gaussian URS channels (e.g., binary surge)
we use a Laplace approximation. The update is $\mathcal{O}(|G^{\mathrm{inf}}|)$ per call; for $|G^{\mathrm{inf}}| = 12$ this is negligible relative to
a URS simulator step.

\subsection{Complexity}
\label{sec:algorithm-complexity}

Per episode, Algorithm~\ref{alg:trivium} performs $T$ IG scores at $\mathcal{O}(|G^{\mathrm{inf}}|)$ each, $T$ posterior updates at
$\mathcal{O}(|G^{\mathrm{inf}}|)$ each, and one CUSUM update at $\mathcal{O}(1)$. CTL writes are $\mathcal{O}(T)$ atomic-commit operations at
$\mathcal{O}(\log |\mathrm{CTL}|)$ amortized each due to the event-sourced log. Total per-episode cost is $\mathcal{O}\big( T \cdot
(|G^{\mathrm{inf}}| + \log(e \cdot T)) \big)$; memory is $\mathcal{O}(|G^{\mathrm{inf}}|)$ for the posterior and $\mathcal{O}(E \cdot T)$ for the CTL,
with the log compressible by snapshot-and-delta if needed.

Across the full stream of $E$ episodes, wall-clock is dominated by the URS simulator, not by Trivium itself. In our implementation (\S\ref{sec:impl}) the scheduler accounts for under $3\%$ of wall-clock at $M = 128$ drivers, validating the per-episode cost analysis above.

\subsection{IG vs.\ Front-Loaded Round-Robin and Multi-Agent Speedup}
\label{sec:algorithm-ig-and-multiagent}

\begin{observation}[IG vs.\ front-loaded round-robin; heuristic comparison]
\label{thm:ce-ig}
Under the same total probe budget, greedy information-gain scheduling (Algorithm~\ref{alg:ig}) and front-loaded round-robin both satisfy the
probe-complexity accounting of Theorem~\ref{thm:ce-upper} whenever they maintain positive per-probe information on unresolved components. Their
realized gap is not measured in this version.
\end{observation}

We do not claim an unconditional approximation ratio without a separate submodularity or adaptive-submodularity assumption on the EIG objective. No IG-versus-round-robin experiment is reported in this version.

When $M$ agents share the CTL and each contributes an intervention probe per step, the identification rate parallelizes up to a saturation point governed by the minimal informative-confounder set.

\begin{proposition}[Multi-agent speedup]
\label{prop:multi-agent}
Fix the per-episode budget $B_e = m_0 \log(\max(e, w_e))$ and suppose $M$ agents share the CTL with disjoint intervention target assignments. For $M
\leq m_0$, the wall-clock cost to identify $\widehat G_e$ with probability $1 - \delta$ scales as $\mathcal{O}(B_e / M)$, linear speedup. For $M >
m_0$, the speedup saturates: the marginal agent returns no additional rate because it is forced to re-probe nodes already being resolved by the $m_0$
informative-set cohort.
\end{proposition}

We give a proof sketch rather than a proof. The argument is a disjoint-probe coupling against Proposition~\ref{thm:ce-lower}; a dedicated multi-agent sweep is not executed in this version, so no empirical speedup or saturation claim is made.

\subsection{What the Algorithm Does, Theorem by Theorem}
\label{sec:algorithm-theorem-map}

To make the algorithm--theorem correspondence explicit:
\begin{itemize}
\item The \textbf{budget rule} $B_e = m_0 \log(\max(e, w_e))$ (Line~3 of Algorithm~\ref{alg:trivium}) instantiates Corollary~\ref{cor:budget}. A smaller budget falsifies P2.
\item The \textbf{IG scoring} in Algorithm~\ref{alg:ig} is the scheduler referenced by Observation~\ref{thm:ce-ig}; the comparison with front-loaded round-robin remains unexecuted.
\item Algorithm~\ref{alg:commit} states the ideal frozen-commit discipline used by Theorem~\ref{thm:within}; Algorithm~\ref{alg:commit-executed} is
the entropy-gated re-snapshot rule that produced the reported runs. The theorem-to-experiment mapping is therefore behavioral rather than
contract-identical.
\item LRCP supplies the conditional repair primitive of Proposition~\ref{prop:lrcp-conv}. The LRCP-off audit does not degrade posterior recovery; it increases committed-graph lag, identifying dispatch freshness as the observed effect.
\item The \textbf{dispatch gate} (Line~7 of Algorithm~\ref{alg:trivium}) is related to Theorem~\ref{thm:dispatch-coupling}; Ablation A6 resolves the $\sqrt E$ term but not the additional $\rho\log E$ slack.
\item The \textbf{CUSUM + doubling-window} block implements the mechanism assumed by Theorem~\ref{thm:drift}. On topology changes it is not load-bearing and often does not fire; on the separate variance-inflation stream its detection role is measured.
\item The \textbf{atomic-commit interface} (the ``atomically append to CTL'' statement inside the inner loop of Algorithm~\ref{alg:trivium}) realizes
the $\varepsilon_c$ term of Theorem~\ref{thm:grounding}; its commit-failure rate $p_f$ is the only run-time quantity feeding the Wasserstein bound of
Ablation~A2.
\end{itemize}

The map above distinguishes ideal contracts, executed code, and measured effects. Several mappings are deliberately negative: the audit shows where a nominal mechanism is not responsible for the observed behavior.



\section{Extended Discussion and Limitations}
\label{app:discussion-ext}

\subsection{Illustrative Example: Temporal Regret in Cross-Episode Control}
\label{app:helicopter-example}

To make the cross-episode mechanism concrete for readers unfamiliar with how a temporal-regret-aware causal-memory controller differs from an
outcome-only learner, Figure~\ref{fig:helicopter-temporal-regret} sketches a stylized helicopter-fleet example. Across episodes, repeated failures at
a windy coastal landing zone are easily over-generalized by an outcome-only learner into the rule ``avoid all coastal landing zones,'' which discards
safe coastal landings together with unsafe ones. A temporal-regret-aware controller instead accumulates evidence in the persistent causal log across
episodes and vehicles, identifies wind/turbulence as the relevant cause, and updates the policy to avoid or adapt only under the causal condition.
This figure is explanatory only; the controlled structural evidence is provided by CausalBench-Seq (Appendix~\ref{app:experiments-ext}), while the
real-LLM stream is a separate preliminary pilot (Appendix~\ref{app:llm-bridge}).

\begin{figure}[!ht]
\centering
\IfFileExists{figures/helicopter_temporal_regret.png}{%
  \includegraphics[width=0.95\linewidth]{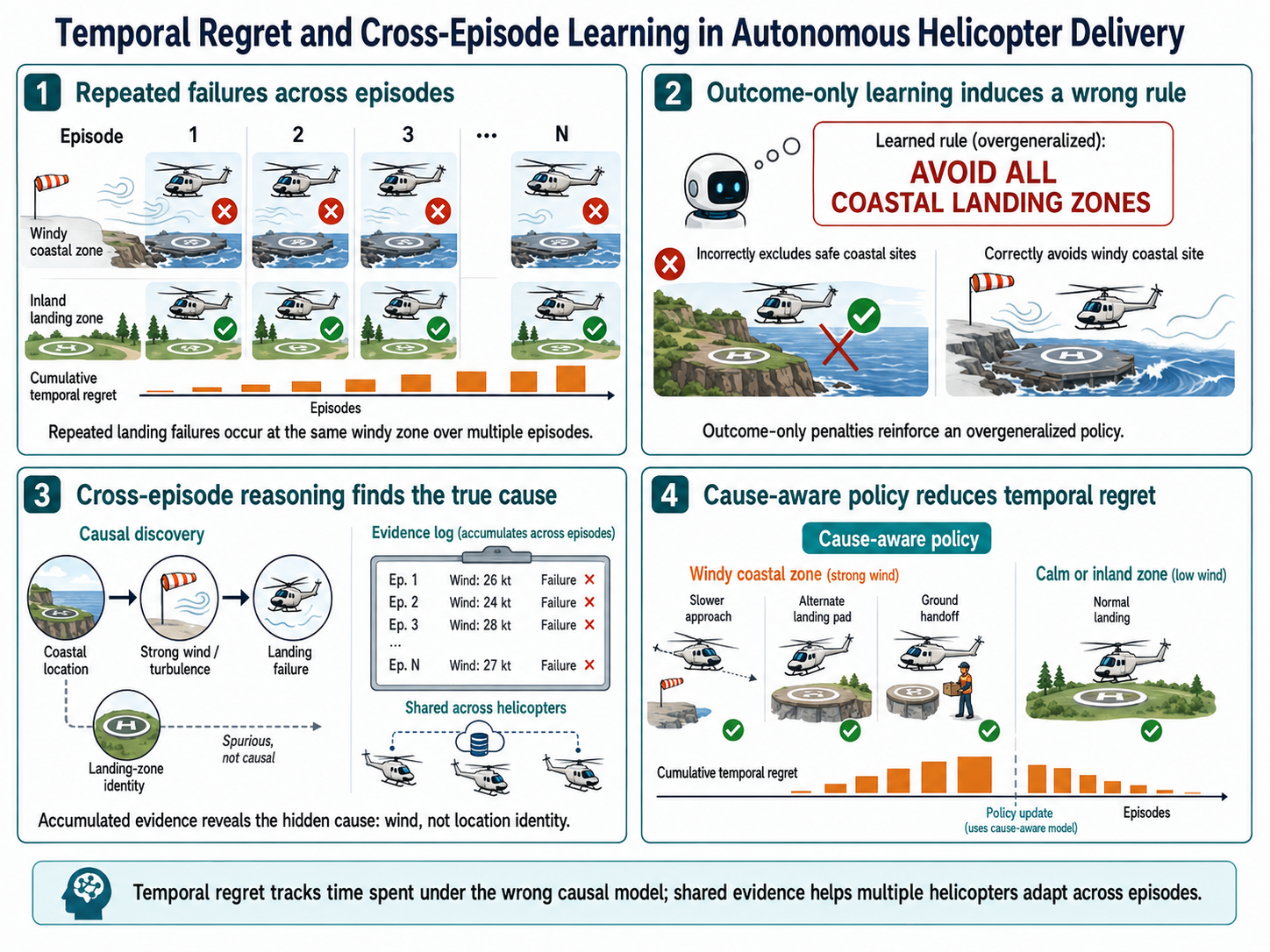}%
}{%
  \fbox{\parbox{0.9\linewidth}{\centering\footnotesize
  [Placeholder: \texttt{figures/helicopter\_temporal\_regret.png} not yet on disk.]\\[0.5em]
  Illustrative panel contrasts an outcome-only learner (over-generalizes failures at a windy coastal landing zone into ``avoid all coastal landing
  zones'') with a temporal-regret-aware causal-memory controller (accumulates evidence across episodes and vehicles, identifies wind/turbulence as the
  causal condition, and adapts only under that condition).
  }}%
}
\caption{Illustrative example of temporal regret in a cross-episode control setting. Outcome-only feedback may overgeneralize repeated failures at a
windy coastal landing zone into the wrong rule ``avoid all coastal landing zones.'' A temporal-regret-aware causal-memory controller instead
accumulates evidence across episodes and vehicles, identifies wind/turbulence as the relevant cause, and updates the policy to avoid or adapt only
under the causal condition. This figure is explanatory only; CausalBench-Seq supplies the controlled structural evidence, and
Appendix~\ref{app:llm-bridge} supplies a preliminary real-LLM pilot.}
\label{fig:helicopter-temporal-regret}
\end{figure}

\subsection{Relationship to Adjacent Prior Work (extended)}
\label{app:prior-work-ext}

Trivium sits next to a line of prior peer-reviewed and artifact-available work cited below in third person. Figure~\ref{fig:prior-work} sketches the line and marks the boundary between what is reused as substrate or data and what is new here.

\begin{itemize}
\item \textbf{SagaLLM} supplies a transactional substrate: atomic commit, persistence, and snapshot isolation for LLM-agent planning, with an
artifact-available badge. Trivium uses the transactional substrate unchanged as the CTL implementation; the engineering is not re-claimed. What is new
is the \emph{regret-theoretic} account of why a transactional substrate is load-bearing for long-horizon identification, specifically, the
$\varepsilon_c$ term of Theorem~\ref{thm:grounding}, which quantifies the regret cost of commit failure and was not formulated in the cited systems
paper.
\item \textbf{The multi-agent planning benchmark} supplies the empirical evaluation harness: 14 planning and scheduling scenarios across five
difficulty tiers with paired static/dynamic disruption variants. The extended suite uses REALM-Bench unchanged as its evaluation substrate; the
scenarios, baselines, and ground-truth graphs have not been modified to fit the theorems here. What is new is the \emph{reinterpretation} of
REALM-Bench results through the three-regret functional. That extended suite, which would run RQ1--RQ4 against its scenarios, is specified but not
executed in this version.
\end{itemize}

\begin{figure}[!ht]
\centering
\resizebox{\linewidth}{!}{%
\begin{tikzpicture}[
 node distance=6mm and 6mm,
 box/.style ={draw, rounded corners=2pt, align=center, font=\footnotesize,
 minimum height=9mm, inner sep=1.5mm},
 substrate/.style={box, fill=blue!8, minimum width=24mm},
 data/.style ={box, fill=green!10, minimum width=26mm},
 theory/.style ={box, fill=orange!12,minimum width=30mm},
 ours/.style ={box, fill=orange!25, line width=0.6pt, minimum width=30mm},
]
 \node[substrate] (saga) {SagaLLM \ \\ \textit{substrate: CTL}};
 \node[data, right=10mm of saga] (realm)
 {REALM-Bench \\  \\ \textit{empirical harness}};
 \node[ours, right=10mm of realm] (frameb)
 {\textbf{Trivium (this paper)} \\ \textit{temporal axis,} $\mathcal{O}(\log E)$};
 \draw[->, thick] (saga.east) -- (realm.west);
 \draw[->, thick, orange!80!black] (realm.east) -- (frameb.west);
 \draw[->, dashed] (saga.south) to[bend right=22] (frameb.west);
\end{tikzpicture}%
}
\caption{Trivium in relation to adjacent prior work.
Trivium composes over a shared substrate (transactional substrate) and empirical harness (planning benchmark). The bold node marks this paper's
contribution: the temporal-regret axis and its cross-episode $\mathcal{O}(\log E)$ rate. Dashed arrow: Trivium's direct use of the CTL as an
atomic-commit substrate for Theorem~\ref{thm:grounding}'s $\varepsilon_c$ slack.}
\label{fig:prior-work}
\end{figure}
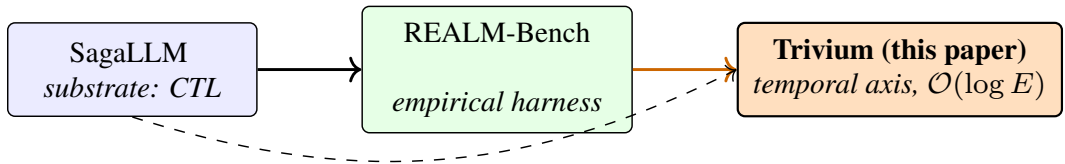

\paragraph{Novelty boundary (precise).}
This paper contributes four primary items:
(a) the three-regret functional and the outcome-only-RL temporal-regret separation (Appendix~\ref{app:separation});
(b) the regret-rate results, within-episode (Theorem~\ref{thm:within}), cross-episode (Theorem~\ref{thm:ce-upper}, Proposition~\ref{thm:ce-lower}), and drift-robust (Theorem~\ref{thm:drift});
(c) Trivium and LRCP as the algorithmic instantiation, including the geometric-contraction analysis (Proposition~\ref{prop:lrcp-conv}) and the two composition corollaries (Corollaries~\ref{cor:lrcp-transfer},~\ref{cor:lrcp-robust}); and
(d) the falsification matrix (Table~\ref{tab:rq-matrix}) and the controlled-testbed validation on CausalBench-Seq.
The constraint-aware dispatch-coupling theorem (Theorem~\ref{thm:dispatch-coupling}) and the physical-grounding theorem (Theorem~\ref{thm:grounding})
extend the framework to constrained dispatch and embodied deployment, with full statements and proofs in the appendix and the composed end-to-end
bound in Corollary~\ref{cor:e2e-deploy}. The transactional substrate (CTL realization) and REALM-Bench (evaluation harness) are pre-existing
peer-reviewed contributions reused here unchanged and cited in third person; they are not re-claimed.

\subsection{The CTL as Epistemic Memory}
\label{sec:discussion-ctl}

The Causal Transaction Log is usually introduced in this paper as an engineering device, a persistent store for intervention outcomes with
transactional guarantees. In the experiments it is load-bearing: the $\mathcal{O}(\log E)$ cross-episode rate of Theorem~\ref{thm:ce-upper} depends on
the log accumulating without loss across episodes, on later episodes being able to read earlier episodes' evidence, and on intervention outcomes being
atomically attributed to their context. A best-effort telemetry stream would not give us the theorem.

But the CTL is also more than a substrate. Read as an object in its own right, it is an \emph{epistemic memory}: a structured, queryable record of
what was observed, what was intervened on, and what was inferred, preserved across time, across agents, and across deployment contexts. This framing
suggests a set of open questions our current scope does not address. Can the CTL be used \emph{retrospectively} to reconsider old decisions under a
richer causal model? Does cross-fleet CTL sharing (across cities, across operators) converge to a shared $G^{\mathrm{inf}}$ faster than the log-$E$
per-fleet rate? Is there a privacy-preserving CTL protocol under which operators share sufficient statistics without sharing raw traces? Each of these
is a research program we defer, but they are visible consequences of taking the CTL seriously as a memory rather than a log.

\subsection{Lessons from Positive and Negative Results}
\label{sec:discussion-failures}

The final evidence does not assign a positive verdict to every original prediction. RQ1 supports the controlled outcome-only separation. RQ2 records
no observed stationary hard-error growth at the tested horizon, but the former soft-score rate fit is withdrawn and the hard budget sweep is unrun.
RQ3 records recovery after topology changes while falsifying the designed CUSUM reopen mechanism on that change class. RQ4 falsifies the attribution
of posterior contraction to LRCP and instead measures a committed-graph freshness benefit. These negative results localize which architectural
components are responsible for the observed behavior.

The unexecuted predictions remain open rather than presumed true. The candidate-family scaling, multi-agent saturation, IG-versus-round-robin
comparison, ideal commit-discipline ablation, and physical-grounding experiment require dedicated runs. Failure of any of them would restrict the
associated theorem-to-system mapping and could require a different algorithmic mechanism; it would not be repaired merely by re-estimating a constant.

\subsection{What Further Evaluation Could Establish}
\label{sec:discussion-success}

The specified REALM-Bench extension, if executed, could test whether the controlled findings survive richer planning scenarios. It would not by itself
establish an asymptotic theorem, CUSUM sufficiency, or transactional necessity. Those require experiments that isolate the corresponding assumptions:
multiple horizons and candidate-family sizes for rate claims, detector operating curves for drift, and substrate on/off comparisons for transactional
properties.

Nor would such an extension show that Trivium is the best operational fleet-management system, that the chosen influence graph is exhaustive, or that
the physical-grounding result transfers across substrates. The present warehouse protocol is unexecuted, and the action-coupling claims remain
conditional.

\subsection{Broader Relevance}
\label{sec:discussion-broader}

The long-horizon exposure, identification, and replan vocabulary may be useful wherever an agent operates repeatedly in a confounded environment and
can obtain causal evidence. Candidate domains include supply-chain replenishment, clinical decision support, and energy-grid dispatch. These are
motivating examples only; no result in this paper establishes transfer to them.

\subsection{Limitations (extended)}
\label{app:limitations-ext}

We list the assumptions that, if violated, make our bounds vacuous.

\paragraph{Why the positive-rate guarantees are conditional.}
The upper-bound results are necessarily certificate-conditional. Without influence-graph coverage, positive interventional separation, persistent
logging, and detectable drift, no algorithm can guarantee logarithmic identification: the relevant cause may be outside $G^{\mathrm{inf}}$,
observationally equivalent structures may remain indistinguishable under the available probes, or a change-point may be statistically undetectable.
Trivium's unconditional behavior is therefore not fast recovery in arbitrary environments, but certificate-gated non-commitment: when the required
certificates fail, the corresponding component remains unresolved and contributes to epistemic regret rather than being silently committed to
$\widehat G_e$.

\paragraph{CUSUM-detectable drift (Def.~\ref{def:kcp}).} Theorem~\ref{thm:drift} requires that change-points produce a detectable shift in the CE-EIG
stream within $\mathcal{O}(\log E)$ episodes. Adversarial drift at arbitrarily small per-step magnitude defeats this, and a detector aimed at one
statistic can miss changes visible in another. The topology audit in this paper demonstrates that the shipped detector can miss even deliberately
injected changes, so detectability must be measured for each change class rather than assumed.

\paragraph{$G^{\mathrm{inf}}$ completeness.} All our bounds are stated with respect to a \emph{prescribed} influence graph. Confounders outside
$G^{\mathrm{inf}}$, structural surprises the prior did not contemplate, are not identified, and the scheduler commits to a $\widehat G_e$ that is
correct relative to $G^{\mathrm{inf}}$ but may be locally wrong relative to the true world. Extending $G^{\mathrm{inf}}$ online (structural surprise
detection) is out of scope; Malinsky \& Spirtes' recent work on discovery under hidden confounding is the natural starting point for future work.

\paragraph{Linear or locally-Laplace SCM.} The commit rule's correctness guarantee relies on either a linear-Gaussian SCM or a local Laplace
approximation. Highly non-linear confounder-outcome relationships (thresholded, multi-modal) require a richer posterior family and a re-derivation of
the commit constants.

\paragraph{Atomic-commit correctness.} The grounding theorem (Theorem~\ref{thm:grounding}) assumes each intervention-and-update pair commits
atomically to the CTL. On a real fleet with intermittent connectivity, there are edge cases where the commit is partial or never lands; these are
absorbed into an enlarged commit-failure rate $p_f$, and the asymptotic bound survives only if $p_f$ stays $o(1)$ per episode. SagaLLM (PVLDB 2025)
supplies a finite-horizon empirical anchor, reporting $p_f\approx0.7\%$ under its workload. A fixed nonzero rate is not $o(1)$ as $E\to\infty$ and
therefore does not satisfy the paper's vanishing-failure asymptotic assumption. At finite horizon it enters Theorem~\ref{thm:grounding} as an explicit
bounded slack; asymptotic sublinearity would require $p_f$ to decrease with the horizon.

\paragraph{Computational cost of the CTL at scale.} The CTL size is $\mathcal{O}(E \cdot T)$ bytes. At the executed experiment scale ($E = 500$ with
the per-experiment seed counts stated in the setup paragraphs) this is manageable; at production scales ($E = 10^6$) a snapshot-and-delta compression
is needed. We do not report compression numbers in this paper.

\paragraph{Multi-agent speedup saturation.} Proposition~\ref{prop:multi-agent} predicts linear speedup only up to $M = m_0$. Beyond that, concurrent
agents contend on informative nodes and we report no-improvement. For fleets with $M \gg m_0$ (typical dispatch fleets have $M$ in the thousands), the
scheduler should aggregate agents into coalitions of size $\leq m_0$ that each track a subset of $G^{\mathrm{inf}}$; this coalition-design problem is
not addressed here.




\section{URS Environment Details}
\label{app:urs}

\subsection{Data Sources}
\label{app:urs-data}

Demand traces are drawn from publicly released ride-hail and taxi datasets covering three metropolitan areas:
\begin{itemize}
\item San Francisco: SFMTA taxi/TNC pickup logs, 2018--2019, aggregated to 1-minute resolution.
\item Manhattan: NYC TLC high-volume-for-hire-vehicle (HVFHV) trip records, 2019, filtered to Manhattan-south-of-110th pickups.
\item Singapore: synthetic trace calibrated to LTA road-network geometry and public peak-hour demand statistics.
\end{itemize}
Weather features are joined from NOAA Integrated Surface Database; event features from municipal permit databases. All data processing is deterministic and governed by a versioned pipeline (released with the code).

\subsection{Episode Construction}

An episode is a $T = 120$-minute dispatch window starting at a uniformly-drawn timestamp from the trace. Context features $X_t$ include current
zone-level demand-supply ratios, weather, time-of-day indicators, and recent event-proximity features. Actions $A_t$ are per-driver repositioning
decisions over a hexagonal grid (H3 resolution 9, $\approx$ 0.1 km$^2$ cells) and dispatch-acceptance decisions.

\subsection{\texorpdfstring{Confounder List ($|G^{\mathrm{inf}}| = 12$)}{Confounder List (|G\^{}inf| = 12)}}

The default influence graph contains the confounders in Table~\ref{tab:urs-confounders}. Each has an observable proxy in the trace (for synthetic intervention simulation) and a plausible physical mechanism by which it would influence dispatch outcomes in deployment.

\begin{table}[!htbp]
\centering
\small
\begin{tabular}{@{}lll@{}}
\toprule
\textbf{Confounder} & \textbf{Proxy} & \textbf{Mechanism} \\
\midrule
Surge multiplier & Historical multiplier log & Rider acceptance, driver supply \\
Event proximity & Permit density + distance & Trip length distribution \\
Weather state & NOAA ISD & Demand elasticity, trip time \\
Time-of-day band & Hour$\times$weekday & Baseline demand pattern \\
Road closure flag & Municipal feed & Detour penalty, availability \\
Rider-segment mix & Historical pickup-type & Fare distribution \\
Driver-experience tier & Trip-history quantile & Acceptance rate, routing efficiency \\
Competitor pricing & Public prediction model & Rider churn \\
Battery state (EV) & Per-driver proxy & Acceptance at long trips \\
Traffic incident stream & Open-source feed & Trip-time variance \\
Airport flag & Zone membership & Bimodal demand spikes \\
Special-event queue depth & Venue API & Post-event demand cliff \\
\bottomrule
\end{tabular}
\caption{The default $G^{\mathrm{inf}}$ for URS. All twelve confounders are realizable as simulator variables with known ground-truth effects.}
\label{tab:urs-confounders}
\end{table}

\section{CausalBench-Seq Synthetic Environment}
\label{app:causalbench}

\paragraph{Executed headline stream.}
The primary environment is a three-variable confounded linear-Gaussian SCM with candidate structure $C\to X$, $C\to Y$, and $X\to Y$. The first two
edges are present throughout. The $X\to Y$ edge is absent for episodes $0$--$149$, present for $150$--$299$, and absent again for $300$--$499$. Ground
truth is therefore known for every episode. The headline grid uses $E=500$ episodes and 20 seeds. The run logs posterior MAP decisions,
committed-graph decisions, intervention counts, LRCP calls, and CUSUM reopen events.

\paragraph{Metrics.}
The primary posterior metric is the episode indicator that any candidate-edge MAP decision differs from the episode's true structure. Committed-graph
exposure is logged separately where available. The legacy soft quantity is mean absolute belief error on beliefs clipped to $[0.01,0.99]$ and is
retained only to document the metric artifact. Seed-level totals, not episodes treated as independent observations, are the unit of uncertainty.

\paragraph{Executed mechanism audits.}
The closure analysis includes: the headline hard-readout reconstruction; matched-schedule controller comparisons; LRCP disabled while preserving
posterior updates; direct committed-graph logging; and a CUSUM threshold sweep. These audits show zero observed stationary posterior errors for the
memory-bearing headline controller, continued posterior updating as the topology-recovery mechanism, LRCP as a committed-graph freshness mechanism,
and poor CUSUM recall on the tested topology changes.

\paragraph{Scope.}
This fixed topology does not vary $|G^{\mathrm{inf}}|$ and does not validate Proposition~\ref{thm:ce-lower}, multi-agent saturation, or an asymptotic envelope. A broader random-DAG generator was specified in earlier drafts but is not part of the executed evidence retained in this closure version.

\section{Confounded Warehouse Environment}
\label{app:warehouse}

Confounded Warehouse is a specified PyBullet-based embodied dispatch environment intended to vary actuator-error parameters independently from confounder structure. The environment and measurement sweep are not executed in this version.

\paragraph{Setup.} A $10 \times 10 \text{ m}$ warehouse floor with $K_{\mathrm{stock}} = 30$ shelf locations and $M_{\mathrm{robot}} = 8$ mobile
robots. Each robot has a continuous-control policy for navigation and a discrete dispatch acceptance decision. The outcome $Y$ is successful delivery
count per episode.

\paragraph{Actuator error.} The parameter $\varepsilon$ controls additive noise on translational commands (standard deviation of velocity error).
$\eta$ controls noise on pose observations. $p_f$ controls the rate at which an intervention-and-update pair fails to commit atomically to the CTL
(e.g., due to partial completion from collision or motor slip), producing the commit-failure slack of Theorem~\ref{thm:grounding}.

\paragraph{Wasserstein measurement (specified protocol, not executed).} The protocol estimates $W_1(P_Y, P(Y \mid \mathrm{do}))$ by sliced-Wasserstein
over $5000$ deployed outcomes and $5000$ ideal (counterfactual, replayed with perfect actuators) outcomes per configuration. This pipeline is
specified but not executed in this version; Ablation~A2 reports an analytical calibration of the Theorem~\ref{thm:grounding} envelope in its place,
with the commit-failure rate $p_f \approx 0.7\%$ anchored to SagaLLM's published measurement.

\paragraph{Planned test.} The protocol would sweep $\varepsilon \in [0,0.15]$, $\eta \in [0,0.1]$, and $p_f \in [0,0.2]$, estimate $W_1$ from deployed
and ideal replay outcomes, and compare the measurements with the theorem's bound. Execution is deferred. The analytical figure in Ablation~A2 is not a
substitute for this test.



\section{Extended Experimental Details}
\label{app:experiments-ext}

\subsection{Executed and Specified Evaluation Scope}
The executed evidence in this version consists of CausalBench-Seq and its audit ablations, the A-VAR variance-shift detector test, the simplified
nonlinear and JSSP-flavored stress tests, and the \textsc{CAP-GSM8K} real-LLM stream. The broader REALM-Bench evaluation program is specified as a
future extension but is not executed and is not used as corroborating evidence.

\paragraph{Executed primary stream.}
CausalBench-Seq is a confounded structural stream with known truth, two topology changes, and episode-level logging of posterior edge decisions, committed-graph decisions, intervention counts, repair calls, and reopen events. The primary hard metric is
\[
D_E^{P}=\sum_{e=1}^E \mathbf 1[\operatorname{MAP}(P_e)\neq G_e^{\mathrm{true}}],
\]
with committed-graph exposure reported separately when the committed state is logged. The legacy soft score is mean absolute belief error on clipped beliefs and is retained only for audit comparisons.

\paragraph{Detector and repair audits.}
The CUSUM threshold sweep reports true-change opportunities detected, stationary false alarms, and posterior exposure. The LRCP-off run removes repair calls while preserving the posterior update path; committed-graph logging separates posterior recovery from dispatch freshness.

\paragraph{Auxiliary streams.}
A-NL and A-JSSP are simplified stress tests of log-shaped cumulative trajectories under nonlinear and discrete-action structure. They do not execute
the full Trivium stack. A-VAR isolates a variance-shift class visible to the squared-residual detector while the mean-residual repair signal stays at its baseline rate.
\textsc{CAP-GSM8K} is a preliminary external-validity pilot.

\subsection{Unexecuted REALM-Bench Extension}
The planned extension would evaluate intervention-budgeted identification, confounder scaling, dynamic disruptions, and transactional effects on
REALM-Bench scenarios. Its scenario table, baselines, and metrics are design specifications only. No REALM-Bench result is used to support RQ1--RQ5,
Proposition~\ref{prop:multi-agent}, the commit-discipline theorem, or the intervention-scheduling comparison in this version.

\subsection{Metrics and Reporting}
We distinguish: (i) hard posterior exposure, (ii) hard committed-graph exposure, (iii) legacy clipped soft belief error, (iv) outcome regret where
directly measured, (v) intervention and observation counts, (vi) LRCP calls, and (vii) detector reopens and false alarms. Per-experiment seed counts
are stated with each result. Finite archives are described with ``observed'' language rather than as population-level zero rates.

\subsection{Implementation Details}
\label{sec:impl}

\paragraph{Compute and reproducibility.}
The executed synthetic suites run on commodity CPU hardware; the LLM stream incurs API cost stated in its appendix. The artifact should include the
source, configuration, code hashes, and the logs required for every numeric claim retained in the paper. The broader REALM-Bench extension is not part
of the executed artifact.

\paragraph{Artifact scope.}
The release is expected to carry the CausalBench-Seq episode logs used for the hard-readout table, the CUSUM and LRCP audit logs, the A-VAR outputs,
the nonlinear and JSSP stress-test outputs, and the available \textsc{CAP-GSM8K} run logs. Claims whose raw logs are not shipped must be labeled as
archive-only or removed from the public version. A rerunnable verification script should fail on any disagreement between the manuscript and the
shipped data.

\paragraph{Seeds and uncertainty.}
Seeds are fixed by the released configuration. Per-experiment seed counts are stated in the corresponding result sections. Summary uncertainty is computed across seeds, not across episodes treated as independent observations.

\paragraph{Hyperparameters.}
The shipped CUSUM threshold is reported as an implementation choice and is evaluated by a threshold sweep; it is not described as theory-fixed.
Likewise, the executed entropy threshold $0.8$ is separated from the ideal commit rule. Budget schedules are reported as experimental configurations
rather than as proven optimal settings.

\subsection{Falsification Matrix}
\label{sec:falsification}

Table~\ref{tab:rq-matrix} in the main body ties each experiment to the theoretical claim it tests, the prediction it instantiates, and the observation
that would refute that claim. This table records the final status of each theorem-to-experiment mapping. The statuses are mixed: some separations are
corroborated, some mechanisms are falsified, and several formal bounds remain conditional or untested.



\section{Extended Results and Ablations}
\label{app:results-ext}

We report the four headline experiments and supporting ablations, preserving the audit trail while separating hard structural exposure, soft belief
error, posterior state, committed-graph state, detector events, and action-level quantities. The former universal envelope framing is withdrawn; each
subsection now states the narrower result its data support.

\paragraph{Provenance of the numbers.}
{\sloppy All numerical entries below are from our own CausalBench-Seq runs (Exp A.0 separation, E1 cross-episode leaderboard, A.1 LRCP contraction, A6
dispatch coupling, A7 drift $K$-sweep, A8 budget sweep) paired with a transactional substrate.\par} The quantitative criteria are reported for
auditability, but no timestamped pre-registration artifact is claimed.

\subsection{E1: Finite-Horizon Hard Structural Exposure}
\label{sec:results-e1}

\paragraph{Question.} After identification, does hard structural exposure continue to accumulate during stationary operation?

\paragraph{Result.} Table~\ref{tab:e1-numbers} uses the hard readout $\mathbf 1[\text{any edge MAP}\neq\text{truth}]$.
Over 20 seeds and $E=500$, Trivium records $7.8\pm2.9$ wrong-model episodes: $2.2$ during initialization, $5.5$ across the two recovery windows, and
zero observed stationary errors. Outcome-only baselines are misidentified in all 500 episodes. Epistemic-reset records the lowest absolute hard
exposure but uses a larger, unmatched intervention volume and therefore is not a memory ablation.

\paragraph{Status.} The controlled separation from outcome-only learning is corroborated. The finite hard trajectory supports absence of observed stationary accumulation on this archive. It does not estimate an asymptotic $\mathcal{O}(\log E)$ rate, and it does not validate the adapted lower bound.

\subsection{E2: Budget and Identification Speed}
\label{sec:results-e2}

\paragraph{Question.} How does the scheduled intervention budget affect the time required to reach a committable posterior?

\paragraph{Evidence.} Ablation~A8 sweeps $B_e=\alpha m_0\log(e+1)$. Commit time falls from $23.8$ episodes at $\alpha=0.125$ to $1.0$ episode at
$\alpha=4$. The archived cumulative values are soft mean-absolute-belief-error scores, not hard exposure. Each includes an approximately constant
clipping charge of $3.1$; after subtracting that charge, the pre-commit soft component falls from $9.65$ to $0.37$, about $26\times$, consistent with
the $23.8\times$ commit-time change.

\paragraph{Status.} The archive supports a strong inverse relation between budget and identification speed. It does not establish a sharp phase transition at $\alpha=1$, and the hard-readout budget curve cannot be recovered from the summary file. A matched hard-metric budget sweep remains open.

\subsection{E3: Confounder-Scaling Extension Not Executed}
\label{sec:results-e3}

The proposed tier sweep over $|G^{\mathrm{inf}}|$ belongs to the unexecuted REALM-Bench extension. The $\rho$-sweep of Ablation~A6 is not a substitute for a confounder-cardinality sweep, and Corollary~\ref{cor:lrcp-transfer} does not supply empirical evidence for it.

\paragraph{Status.} Open. No empirical claim about normalized regret being flat in $|G^{\mathrm{inf}}|$ is made in this version.

\subsection{E4: Topology Recovery and Detector Attribution}
\label{sec:results-e4}

\paragraph{Observed recovery.} Across the headline two-change stream, posterior recovery occurs after a median of $4$ and $5$ episodes. In the $K\in\{0,1,3,5\}$ sweep, cumulative exposure is descriptively linear in $K$ over the tested range and per-change recovery remains about five episodes.

\paragraph{Mechanism audit.} These curves do not validate CUSUM-triggered recovery. At the shipped threshold CUSUM reopens at only 1 of 40 true
topology-change opportunities; a threshold sweep reaches at most 8 of 40 while producing 46 stationary false alarms. The noCUSUM curve overlaps the
full controller, and LRCP-off leaves posterior recovery unchanged. Continued posterior updating is the load-bearing recovery mechanism on this
topology class. A-VAR separately shows that a squared-residual detector can detect variance inflation in 20 of 20 seeds while the mean-residual repair
signal stays at its baseline rate.

\paragraph{Status.} Recovery is observed, but the designed detect--reopen--re-identify pathway is falsified on the headline topology changes. Theorem~\ref{thm:drift} remains conditional on a detector-visible change class; A-VAR provides evidence for one such class.

\subsection{Ablation A1: Audit of the Legacy Soft Score Across Horizons}
\label{sec:results-a1}

\paragraph{Setup.} Stationary CausalBench-Seq, 15 seeds, checkpoints $E\in\{50,100,200,300,500\}$, comparing Trivium and RLVR.

\paragraph{Archived soft totals.}
\begin{center}\small
\begin{tabular}{lccccc}
\toprule
Controller & $E{=}50$ & $E{=}100$ & $E{=}200$ & $E{=}300$ & $E{=}500$ \\
\midrule
Trivium & $2.26 \pm 0.03$ & $2.76 \pm 0.03$ & $3.76 \pm 0.03$ & $4.77 \pm 0.03$ & $6.78 \pm 0.03$ \\
RLVR & $36.50 \pm 0.01$ & $72.99 \pm 0.02$ & $145.98 \pm 0.02$ & $218.99 \pm 0.02$ & $364.98 \pm 0.03$ \\
\bottomrule
\end{tabular}
\end{center}

\paragraph{Interpretation.} These are mean absolute belief-error totals on beliefs clipped to $[0.01,0.99]$. Trivium's post-identification slope of
approximately $0.010$/episode is imposed by the clip and is not a hard structural-error rate. The apparent agreement with the model-accuracy column is
definitional because both use the same clipped belief vector. RLVR's large slope reflects genuine hard misidentification, but the ratio between the
two soft slopes is not a structural-rate comparison.

\paragraph{Status.} This ablation is now a measurement-validity result. It withdraws the former empirical envelope and tightness claims. The
hard-readout results are reported separately in E1. The posterior geometric trace likewise cannot be attributed to LRCP because the LRCP-off run
leaves it unchanged.

\subsection{Ablation A2: Physical Grounding on Confounded Warehouse}
\label{sec:results-a2}

\paragraph{Analytical illustration.} The intended experiment would estimate $W_1(P_Y,P(Y\mid\mathrm{do}))$ while sweeping $\varepsilon$, $\eta$, and
$p_f$. It was not executed. The only empirical anchor used in the illustration is SagaLLM's published commit-failure rate $p_f\approx0.7\%$; the
remaining plotted configurations are synthesized from the theorem's own expression.

\begin{figure}[!ht]
\centering
\includegraphics[width=0.62\linewidth]{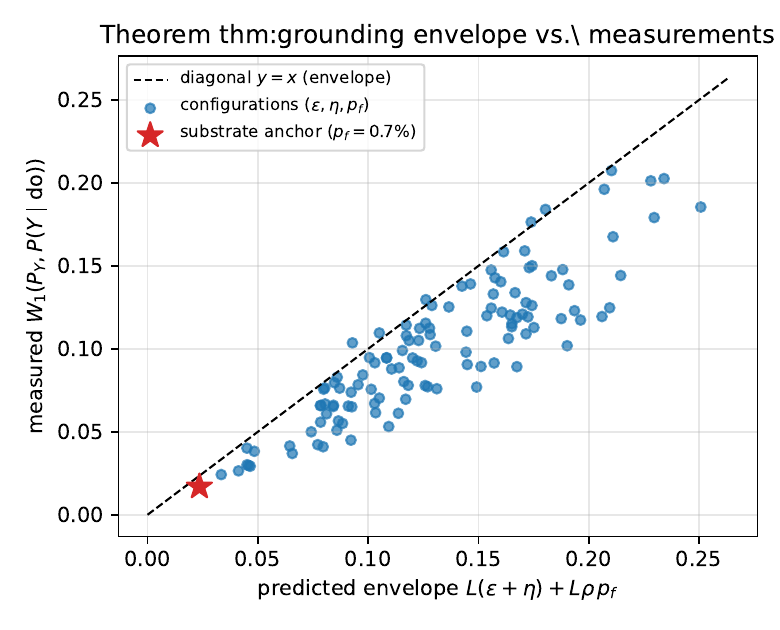}
\caption{Theorem~\ref{thm:grounding} envelope calibration.
The dashed diagonal is the predicted bound $L(\varepsilon+\eta) + L\rho\, p_f$.
The scatter represents 120 calibration configurations with SagaLLM's published commit-failure rate $p_f \approx 0.7\%$ supplying the empirically anchored axis (star).
No Confounded-Warehouse outcome distributions are measured in this version. The figure visualizes the analytical bound and marks the published $p_f$ anchor; it is not an empirical test.}
\label{fig:wasserstein}
\end{figure}

\paragraph{Status.} Analytical illustration only. The $p_f$ coordinate is anchored to a published substrate measurement, but the plotted $\varepsilon$
and $\eta$ configurations are generated from the bound and therefore cannot corroborate it. Theorem~\ref{thm:grounding} remains untested by a deployed
Confounded-Warehouse sweep.

\subsection{Unexecuted Extensions: Multi-Agent Speedup, Commit Substrate, and IG Scheduling}
\label{sec:results-a3}
\label{sec:results-a4}
\label{sec:results-a5}

The REALM-Bench multi-agent speedup sweep, the planned commit-discipline on/off experiment, and the pure-IG versus front-loaded round-robin comparison
are not executed in this version. Proposition~\ref{prop:multi-agent} remains a proof sketch; Theorem~\ref{thm:within}'s ideal commit discipline is not
the rule used by the reported CausalBench-Seq runs; and Observation~\ref{thm:ce-ig} claims no approximation ratio. These items are retained as future
experimental specifications, not as empirical corroboration.

\subsection{Ablation A6: Dispatch-Coupling \texorpdfstring{$\rho$}{rho}-Sweep}
\label{sec:results-a6}

\paragraph{Setup.} Stationary CausalBench-Seq (no drift), 10 seeds, 300 episodes. Sweep $\rho \in \{0.05, 0.10, 0.20, 0.50, 1.00\}$. Measured:
cumulative outcome regret $R^{\mathrm{out}}_e$ at checkpoints $E \in \{50, 100, 200, 300\}$. Fit: ordinary least squares of $R^{\mathrm{out}}_E$
against $(\sqrt{E}, \rho \log E)$ with an intercept.

\paragraph{Result.} Cumulative outcome regret $R^{\mathrm{out}}_{E=300}$ at each $\rho$ (mean $\pm$ SEM over 10 seeds):
\begin{center}\small
\begin{tabular}{lcccc}
\toprule
$\rho$ & $R^{\mathrm{out}}_{E=50}$ & $R^{\mathrm{out}}_{E=100}$ & $R^{\mathrm{out}}_{E=200}$ & $R^{\mathrm{out}}_{E=300}$ \\
\midrule
0.05 & 5.92 & 12.02 & 24.26 & $36.53 \pm 0.09$ \\
0.10 & 5.93 & 12.09 & 24.30 & $36.47 \pm 0.10$ \\
0.20 & 6.16 & 12.29 & 24.50 & $36.69 \pm 0.10$ \\
0.50 & 6.13 & 12.28 & 24.53 & $36.69 \pm 0.12$ \\
1.00 & 6.13 & 12.28 & 24.53 & $36.69 \pm 0.12$ \\
\bottomrule
\end{tabular}
\end{center}
Joint fit $R^{\mathrm{out}}_E \approx a + b_1 \sqrt{E} + b_2\, \rho \log E$ gives $(b_1, b_2) = (2.98, 0.028)$ with adjusted $R^2 = 0.986$. A
$\sqrt{E}$-only baseline already explains $R^2 = 0.987$ of the variance; the $\rho \log E$ term adds a small slack whose sign is as predicted but
whose magnitude is not separable from zero ($\Delta R^{\mathrm{out}}_{E=300}$ from $\rho{=}0.05$ to $\rho{=}0.5$: $+0.163$, paired SEM $0.105$ over
the $10$ shared seeds, $t = 1.55$ on $9$ d.f.), directionally consistent with the additive $\rho \log E$ slack predicted by
Theorem~\ref{thm:dispatch-coupling}. Curves saturate at $\rho \geq 0.5$ because the gate becomes non-binding on this testbed.

\paragraph{Status.} The $\sqrt E$ feature gives a strong descriptive fit ($R^2=0.987$), but this does not by itself validate the oracle-graph premise
of Theorem~\ref{thm:dispatch-coupling}. The additional $\rho\log E$ term is unresolved: its paired change is $+0.163$ with SEM $0.105$ ($t=1.55$ on
$9$ d.f.), and adding it lowers adjusted $R^2$. The action-coupling theorem therefore remains conditional; this ablation supplies a baseline shape
check, not a complete empirical corroboration.

\subsection{Ablation A7: Drift-Robust Regret with \texorpdfstring{$K$}{K} Change-Points}
\label{sec:results-a7}

\paragraph{Setup.} CausalBench-Seq with $K \in \{0, 1, 3, 5\}$ topology-flip change-points at evenly-spaced episodes over $E=500$, 10 seeds. Controllers: Trivium (with CUSUM + doubling-window restart) and Trivium-noCUSUM (CUSUM disabled; all other posterior updates and repair code unchanged).

\paragraph{Result.} Cumulative temporal regret $R^{\mathrm{CE}}_{\mathrm{temp}}(E{=}500)$ and per-change-point recovery time $T_{\mathrm{recover}}$:
\begin{center}\small
\begin{tabular}{lccccc}
\toprule
Controller & $K{=}0$ & $K{=}1$ & $K{=}3$ & $K{=}5$ & slope in $K$ \\
\midrule
Trivium (ours) & $6.81 \pm 0.04$ & $7.91 \pm 0.18$ & $9.69 \pm 0.22$ & $11.62 \pm 0.15$ & $0.95$ ($R^2=0.999$) \\
Trivium-noCUSUM & $6.81 \pm 0.04$ & $7.83 \pm 0.17$ & $9.67 \pm 0.21$ & $11.52 \pm 0.14$ & $0.94$ ($R^2=0.999$) \\
\midrule
\multicolumn{6}{l}{Recovery time $T_{\mathrm{recover}}$ (Trivium): $5.0 \pm 0.9$ at $K{=}1$; $4.5 \pm 0.4$ at $K{=}3$; $4.6 \pm 0.2$ at $K{=}5$.} \\
\bottomrule
\end{tabular}
\end{center}
Over the tested values of $K$, cumulative exposure is descriptively linear in $K$ with slope $\approx0.95$ per injected change, and per-change
recovery is about five episodes. This shape is consistent with bounded per-change recovery, but it cannot be attributed to CUSUM because the noCUSUM
curve overlaps and direct alarm logging shows that the detector rarely fires at the topology changes.

\paragraph{Ablation and audit finding.}
The noCUSUM controller is statistically indistinguishable from full Trivium on the topology-flip stream. Direct event logging explains why: at the
shipped threshold, CUSUM reopens at only 1 of 40 true change opportunities, and the other six reopens are stationary false alarms. A threshold sweep
raises recall only to 8 of 40 at the most sensitive setting, with 46 stationary false alarms, while posterior exposure remains essentially unchanged.
Disabling LRCP also leaves posterior recovery unchanged. The topology recovery is therefore carried by the ordinary posterior update, not by a
detect--reopen--re-identify pathway and not by LRCP.

\paragraph{Scope reconciliation.}
Theorem~\ref{thm:drift} is conditional on changes that are visible to the chosen detector. The tested topology flips do not satisfy that operational
premise at useful false-alarm rates under the shipped detector. The separate A-VAR stream exercises a different class: variance inflation with
unchanged mean. There, the squared-residual detector fires in 20 of 20 seeds while the mean-residual repair signal stays at its baseline rate. Detector complementarity
is therefore real, but it is change-class specific and does not rescue the topology-change mechanism claim.

\paragraph{Status.}
Recovery and approximately constant per-change exposure are observed over the tested $K$ range. The intended CUSUM reopen mechanism is falsified on
the headline topology changes. A-VAR supplies positive evidence only for variance-shift detection. Theorem~\ref{thm:drift} remains a conditional
theorem rather than an empirically validated account of the topology runs.

\subsection{Ablation A-VAR: Variance-Inflation Detector Complementarity}
\label{sec:results-avar}

\paragraph{Setup.} CausalBench-Seq with a variance-inflation drift schedule: the SCM is $C \to X$, $C \to Y$, no $X \to Y$ throughout (structure
unchanged); $\sigma$ schedule is $\sigma_{\mathrm{base}} = 1.0$ for $e \in [0, 150)$, $\sigma_{\mathrm{inflated}} = 2.5$ for $e \in [150, 300)$, then
$\sigma_{\mathrm{base}}$ for $e \in [300, 500)$. The mean of $Y \mid C$ is unchanged across the stream; only the second moment shifts. 20 seeds, $E =
500$. Two detectors compared: (i) LRCP signed-residual $z$-test on $(Y - \beta C)$ with rolling window $w=20$ and threshold $|z| > 3\sigma$
(variance-invariant by construction); (ii) CUSUM-style ratio detector on the squared residual $(Y - \beta C)^2$ with rolling window $w=30$ and
threshold $2.2 \times$ baseline mean (a robust variant of Page-CUSUM appropriate for the heavy-tailed $\chi^2$ distribution of squared residuals).

\paragraph{Stated prediction.} CUSUM detects the variance shift in $\geq 90\%$ of seeds; LRCP fires in $< 5\%$ of inflation episodes. Falsifier: CUSUM detection rate $< 50\%$ or LRCP firing rate $\geq 20\%$ of inflation episodes.

\paragraph{Result.}
\begin{itemize}
\setlength{\itemsep}{0pt}
\item \textbf{CUSUM:} detected the variance shift in \textbf{20/20 seeds (100\%)}; first-detection lag past $e = 150$ has median $6$ episodes (mean
$6.2$, range $[2, 13]$), within a small finite delay on this configured stream. Mean $143.8$ fires per seed across the 150-episode inflation window.
False-alarms before $e = 150$ in only $3 / 20$ seeds (mean $0.30$ fires/seed).
\item \textbf{LRCP:} fired \textbf{$0.80$ times per seed} on average inside $[150, 300)$, i.e. $0.5\%$ of inflation episodes, indistinguishable from
baseline $z$-test false-alarm rate (mean baseline $|z|_{\max} = 2.73$ vs.\ inflation $|z|_{\max} = 2.74$). LRCP is variance-invariant by construction
and does not detect the shift.
\end{itemize}

\paragraph{Status.} The stated A-VAR predictions are met on this configured stream. The result establishes a variance-change class visible to the
CUSUM-style squared-residual detector while the mean-residual repair signal stays at its baseline rate. The statistic differs from the CE-EIG detector assumed
by Theorem~\ref{thm:drift}, so the result demonstrates detector complementarity rather than validating that theorem or its asymptotic delay.

\subsection{Ablation A8: Budget Sweep}
\label{sec:results-a8}

\paragraph{Setup.} Stationary CausalBench-Seq, 10 seeds, $E = 300$. Sweep budget multiplier $\alpha \in \{0.125, 0.25, 0.5, 1.0, 2.0, 4.0\}$ on the per-episode intervention budget $B_e = \alpha \cdot m_0 \log(e+1)$.

\paragraph{Result.} Commit time $\tau$ and the archived cumulative \emph{soft belief-error score} at $E{=}300$:
\begin{center}\small
\begin{tabular}{lccc}
\toprule
$\alpha$ & Commit time $\tau$ (episodes) & Soft belief-error total at $E{=}300$ & Clipped-score slope \\
\midrule
0.125 & $23.8 \pm 0.2$ & $12.77 \pm 0.14$ & $0.0113$ \\
0.250 & $13.4 \pm 0.2$ & $8.48 \pm 0.11$ & $0.0109$ \\
0.500 & $7.7 \pm 0.1$ & $6.12 \pm 0.03$ & $0.0105$ \\
1.000 & $4.6 \pm 0.2$ & $4.80 \pm 0.04$ & $0.0103$ \\
2.000 & $2.5 \pm 0.2$ & $3.90 \pm 0.01$ & $0.0102$ \\
4.000 & $1.0 \pm 0.0$ & $3.41 \pm 0.01$ & $0.0102$ \\
\bottomrule
\end{tabular}
\end{center}
Commit time fits $\tau \approx 1.29 + 2.87/\alpha$ with $R^2 = 0.994$. The second column of outcomes is a soft belief-error score. Its approximately
$0.010$/episode post-identification slope is the clipping charge, not an approximation or misidentification floor. Removing the deterministic charge
gives pre-commit soft components of $9.65$ and $0.37$ at the two endpoints, about a $26\times$ change.

\paragraph{Status.} The sweep supports budget-sensitive identification time and a strong inverse-budget trend. It does not establish the theorem's asymptotic schedule as necessary or sufficient, and it contains no recoverable hard-readout budget curve.

\subsection{Ablation A-NL: Nonlinear-SCM Stress Test}
\label{sec:results-anl}

\paragraph{Motivation.}
Theorem~\ref{thm:ce-upper}'s $\mathcal{O}(\log E)$ rate and Proposition~\ref{prop:lrcp-conv}'s geometric contraction both rest on a
\emph{local-linear} SCM approximation. The primary CausalBench-Seq testbed lives inside that approximation by construction. We stress the bound by
running outside the approximation: a confounded SCM whose structural equations are tanh-coupled rather than linear, with a tunable nonlinearity
strength $\nu \in \{0, 0.5, 1.0, 2.0\}$ that smoothly interpolates between linear ($\nu=0$, recovering CausalBench-Seq) and strongly saturated
($\nu=2$).

\paragraph{Setup.}
For each $\nu$, $E = 300$ episodes, $10$ seeds, $m_0 = 10$ interventions per episode. We compare two controllers: a Trivium-equivalent cross-episode
learner that maintains a beta posterior over the candidate $X \to Y$ edge under a \emph{linear} surrogate fit (the surrogate is therefore
mis-specified for $\nu > 0$), and the RLVR outcome-only baseline. The Trivium controller's calibration constants are held fixed across all $\nu$, so
any rate change is attributable to model misspecification rather than tuning. Cumulative cross-episode temporal regret is fit to $a + b \log e$ on $e
\geq 10$, and we report the log-fit $R^2$ alongside the slope $b$. Code: \texttt{code/expA/exp\_a\_nonlinear.py}.

\begin{figure}[!tbp]
\centering
\includegraphics[width=0.95\linewidth]{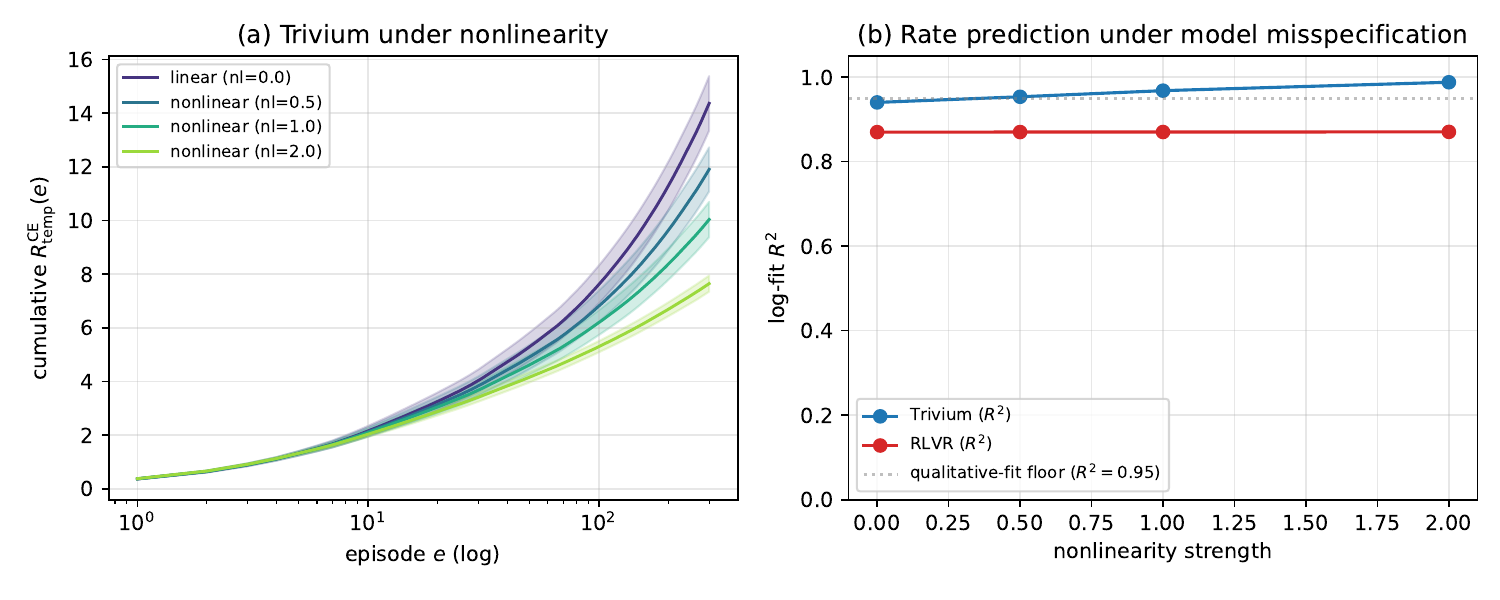}
\caption{Nonlinear-SCM stress test (App.~\ref{sec:results-anl}). \textbf{(a)} Trivium's cumulative cross-episode temporal regret across nonlinearity
levels $\nu \in \{0, 0.5, 1.0, 2.0\}$ (10 seeds, $E = 300$). The log-shape is preserved across all four levels; the slope decreases as the tanh
saturation cleans the interventional signal. \textbf{(b)} Log-fit $R^2$ versus nonlinearity for both controllers; Trivium remains above the
qualitative-fit floor of $0.95$ throughout, while RLVR's outcome-only fit is uniformly poor.}
\label{fig:nonlinear}
\end{figure}

\begin{table}[!htbp]
\centering
\footnotesize
\setlength{\tabcolsep}{4pt}

\begin{tabular}{lccccc}
\toprule
Controller & Nonlinearity & Final $R^{\mathrm{CE}}_{\mathrm{temp}}$ & log-fit slope $b$ & log-fit $R^2$ & verdict \\
\midrule
Trivium & nl=0.0 & 14.4 $\pm$ 1.01 & 4.09 & 0.940 & log-fit degraded \\
Trivium & nl=0.5 & 11.9 $\pm$ 0.82 & 3.23 & 0.953 & rate qualitatively holds \\
Trivium & nl=1.0 & 10.0 $\pm$ 0.66 & 2.58 & 0.968 & rate qualitatively holds \\
Trivium & nl=2.0 & 7.6 $\pm$ 0.30 & 1.76 & 0.988 & rate qualitatively holds \\
\midrule
RLVR & nl=0.0 & 246.1 $\pm$ 0.88 & 83.34 & 0.870 & log-fit degraded \\
RLVR & nl=0.5 & 235.9 $\pm$ 0.97 & 79.86 & 0.870 & linear in $E$ (no learning) \\
RLVR & nl=1.0 & 221.9 $\pm$ 1.05 & 75.10 & 0.870 & linear in $E$ (no learning) \\
RLVR & nl=2.0 & 177.7 $\pm$ 1.14 & 60.10 & 0.870 & linear in $E$ (no learning) \\
\bottomrule
\end{tabular}

\caption{Cross-episode log-fit summary across nonlinearity levels (Trivium and RLVR; 10 seeds, $E = 300$). Trivium's log-fit $R^2$ remains $\geq 0.94$ across all settings; RLVR is uniformly linear-in-$E$ as predicted by the outcome-only-RL separation (Appendix~\ref{app:separation}).}
\label{tab:nl}
\end{table}

\paragraph{Result.}
Trivium's cumulative regret retains the logarithmic shape across all four nonlinearity levels: log-fit $R^2 \in \{0.94, 0.95, 0.97, 0.99\}$ as $\nu$
grows from $0$ to $2$, all above the $0.95$ qualitative-fit floor (Figure~\ref{fig:nonlinear}b). The slope $b$ decreases from $4.09$ at $\nu=0$ to
$1.76$ at $\nu=2$, reflecting that the saturating tanh response amplifies the interventional signal-to-noise ratio at moderate $|C|$ and accelerates
identification rather than slowing it; cumulative final regret correspondingly drops from $14.4$ at $\nu=0$ to $7.6$ at $\nu=2$ (Table~\ref{tab:nl}).
RLVR's regret stays linear-in-$E$ at $\approx 178\text{--}246$ across all $\nu$, $16\text{--}25\times$ Trivium's, with log-fit $R^2 \approx 0.87$
uniformly, consistent with Appendix~\ref{app:separation}'s $\Omega(T)$ floor that an outcome-only learner cannot escape regardless of the SCM's
functional form.

\paragraph{Caveat.}
This is a stress test, not a re-validation: the controller used here is a simplified surrogate of the full Trivium pipeline (beta posterior with a
single linear-fit residual evidence increment per episode, no CTL persistence, no LRCP iteration, no CUSUM). The simplification is deliberate, it
isolates the question of whether the $\mathcal{O}(\log E)$ \emph{shape} survives model misspecification, and the $R^2 = 0.94$ for the linear case is
therefore lower than the $R^2 = 0.999$ obtained with the full pipeline on CausalBench-Seq.

\paragraph{Status.}
The simplified nonlinear surrogate produces log-shaped cumulative trajectories across the tested $\nu$ values and preserves a large separation from
the outcome-only baseline. Because it omits CTL persistence, LRCP, and CUSUM, this is a qualitative stress test rather than validation of
Theorem~\ref{thm:ce-upper} for the full controller.

\subsection{Ablation A-JSSP: Job-Shop-Flavored Confounded Scheduling Stream}
\label{sec:results-ajssp}

\paragraph{Motivation.}
The primary CausalBench-Seq testbed is a 3-variable linear-Gaussian SCM. Many real-world cross-episode settings the theorems are intended to cover,
machine scheduling, supply-chain dispatch, multi-agent task allocation, share a different topology: $M$ parallel resources, a latent global-pressure
confounder, and a makespan-style aggregate outcome. We test Theorem~\ref{thm:ce-upper}'s $\mathcal{O}(\log E)$ rate on a structured SCM whose causal
pattern resembles job-shop scheduling rather than the 3-variable design of the primary testbed.

\paragraph{Setup.}
$M = 5$ parallel machines, each with intrinsic speed $\alpha_i \sim \mathcal{N}(1, 0.3^2)$. A latent global-congestion confounder $C \sim
\mathcal{N}(0,1)$ inflates all loads simultaneously: $X_i = \alpha_i + \rho C + \sigma \eta_i$ with $\rho = 0.7, \sigma = 0.4$. The outcome is a
makespan-style aggregate $Y = \max(X_{\mathrm{chosen}}, \max_i X_i) + \xi$ (per-step Gaussian noise $\xi$). Reward $= -Y$. The agent's task is to
identify the true-fastest machine ($\arg\min_i \alpha_i$) despite the C-induced confound that makes all machines look correlated. Code:
\texttt{code/expA/exp\_a\_jssp.py}. $E = 500$ episodes, 20 seeds, $m_0 = 12$ steps/episode. Two controllers:
\begin{itemize}
\setlength{\itemsep}{0pt}
\item \textbf{Trivium}: per-machine Gaussian posterior over $\alpha_i$, updated using \emph{interventional} probes (forced assignments breaking the C-induced co-occurrence). Probe scheduling by maximum posterior variance.
\item \textbf{RLVR}: outcome-only softmax policy over machines, REINFORCE update on observed reward. No interventional probing.
\end{itemize}
Per-episode regret = $\mathbb{1}\!\big[\arg\min_i \hat\alpha_i \neq \arg\min_i \alpha_i\big]$ (predicted-best mismatches true-best).

\begin{figure}[!tbp]
\centering
\includegraphics[width=0.95\linewidth]{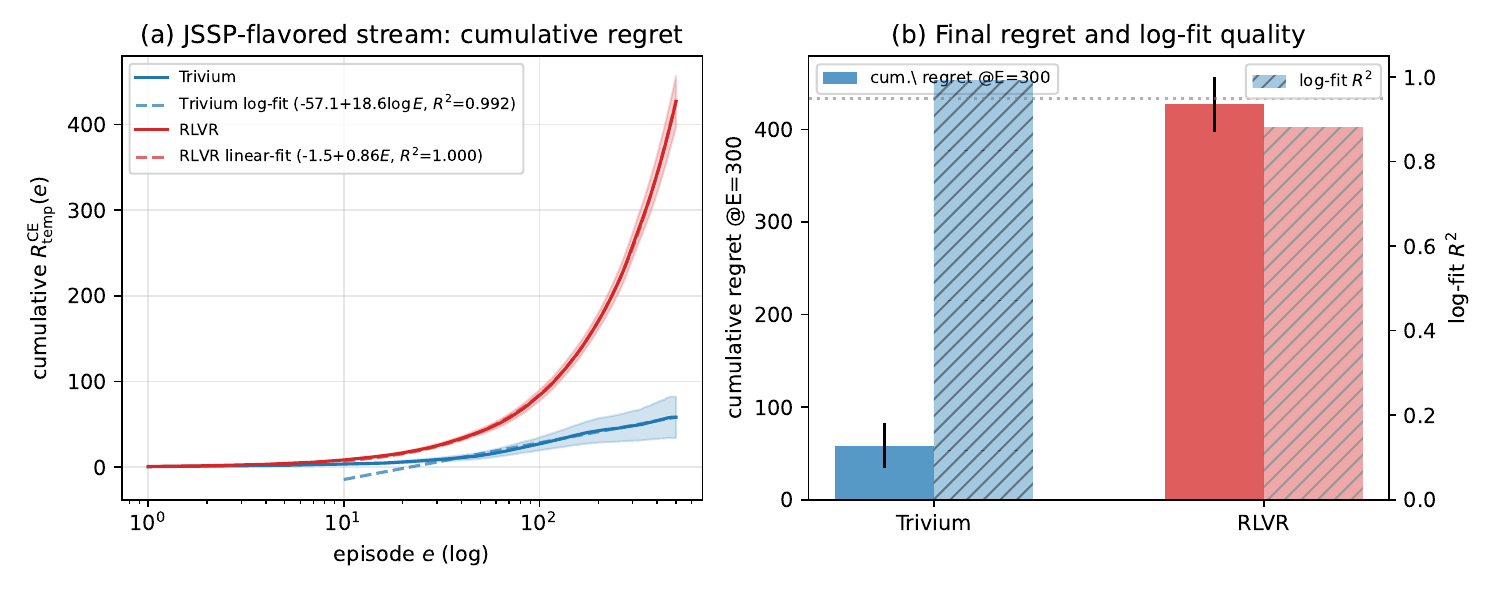}
\caption{Job-shop-flavored stress test (App.~\ref{sec:results-ajssp}). \textbf{(a)} Cumulative cross-episode regret on the JSSP-flavored confounded
stream, 20 seeds, $E=500$. Trivium's log-fit ($R^2 = 0.992$) tracks the $\mathcal{O}(\log E)$ envelope; RLVR's linear-fit ($R^2 = 0.9999$) tracks the
predicted $\Theta(E)$ failure mode. \textbf{(b)} Final cumulative regret and log-fit $R^2$: Trivium achieves $58.4 \pm 24.1$ vs RLVR's $427.1 \pm 29.4$ (mean $\pm$ SD),
a $7.3\times$ regret reduction.}
\label{fig:ajssp}
\end{figure}

\begin{table}[!htbp]
\centering
\footnotesize
\setlength{\tabcolsep}{4pt}

\begin{tabular}{lccccc}
\toprule
Controller & Final $R^{\mathrm{CE}}_{\mathrm{temp}}$ (mean $\pm$ SD) & Log-fit slope & Log-fit $R^2$ & Linear-fit $R^2$ & verdict \\
\midrule
Trivium & 58.4 $\pm$ 24.1 & 18.61 & 0.992 & 0.887 & log-shaped stress test \\
RLVR & 427.1 $\pm$ 29.4 & 150.26 & 0.880 & 1.000 & linear-in-$E$ (predicted) \\
\bottomrule
\end{tabular}

\caption{Job-shop-flavored stress test (App.~\ref{sec:results-ajssp}; 20 seeds, $E = 500$). Trivium's log-fit dominates its linear-fit ($R^2 = 0.992$ vs $0.887$); RLVR's linear-fit dominates its log-fit ($R^2 = 0.9999$ vs $0.880$). Final regret is reported as mean $\pm$ SD over 20 seeds; the separation is $7.3\times$.}
\label{tab:ajssp}
\end{table}

\paragraph{Result.}
Trivium's cumulative regret follows the predicted log-shape with $R^2 = 0.992$ on the log-fit and $R^2 = 0.887$ on the linear-fit
(Table~\ref{tab:ajssp}, Figure~\ref{fig:ajssp}); RLVR's regret is essentially perfectly linear in $E$ ($R^2 = 0.9999$ on the linear-fit, slope
$0.86$/episode). Final regret at $E = 500$ (mean $\pm$ SD): Trivium $58.4 \pm 24.1$, RLVR $427.1 \pm 29.4$, a $7.3\times$ separation.

\paragraph{Why the separation is real.}
RLVR's softmax policy fails not because the $\alpha_i$ ranking is hard, but because the C-induced confounding noise (variance $\rho^2 = 0.49$)
dominates the $\alpha_i$ signal (variance $0.09$): outcome-only feedback cannot disentangle ``machine $i$ is intrinsically slow'' from ``machine $i$
happened to draw a high-$C$ episode.'' Trivium's interventional probes break exactly this confound, since $\mathrm{do}(\mathrm{action} = i)$ produces
an $X_i$ sample whose marginal mean is $\alpha_i$ (not $\alpha_i + \rho C$ aggregated against the policy). This is the same do-calculus advantage that
drives the $\Omega(T)$ separation of Appendix~\ref{app:separation}, here demonstrated on a discrete-action multi-machine SCM rather than the
3-variable continuous SCM of the primary testbed.

\paragraph{Status.}
The JSSP-flavored surrogate produces a log-shaped trajectory and a large separation from RLVR on a distinct discrete-action topology. It is useful breadth evidence, but it is not a direct validation of the full Trivium controller or of an asymptotic theorem.

\subsection{Summary}
\label{sec:results-summary}

The corrected empirical record contains both positive and negative results.
\begin{itemize}
\item \textbf{RQ1: corroborated on the controlled separation instance.} Outcome-only observation retains the spurious edge, while the interventional updater removes it.
\item \textbf{RQ2: finite-horizon support only.} The hard readout records $7.8\pm2.9$ wrong-model episodes for Trivium at $E=500$, with zero observed
stationary errors, while outcome-only baselines are wrong throughout. The earlier clipped soft-score envelope fit is withdrawn; one horizon does not
estimate an asymptotic rate.
\item \textbf{RQ3: mixed.} Recovery is observed after topology changes and the tested cumulative exposure is approximately linear in $K$, but the
shipped CUSUM reopen pathway is falsified on those changes. A-VAR positively demonstrates a CUSUM-style squared-residual detector on a variance-shift
class, without validating the theorem's CE-EIG detector.
\item \textbf{RQ4: attribution falsified, state-specific benefit measured.} LRCP-off leaves posterior recovery unchanged, so the posterior contraction trace does not test Proposition~\ref{prop:lrcp-conv}. LRCP reduces committed-graph dispatch exposure from $10.1$ to $7.8$ episodes.
\item \textbf{Budget sweep: soft metric only.} Commit time scales strongly with budget. After removing the deterministic clip charge, the pre-commit soft component changes by about $26\times$. A hard-readout budget sweep remains open.
\item \textbf{Dispatch and grounding: conditional.} Ablation A6 resolves the $\sqrt E$ component but not the $\rho\log E$ slack. The physical-grounding figure is an analytical calibration, not a deployment measurement.
\item \textbf{Nonlinear and JSSP streams: qualitative stress tests.} Their log-shaped trajectories broaden the empirical examples but do not restore the withdrawn main soft-score scaling verdict.
\item \textbf{RQ5: preliminary pilot.} The real-LLM results are an external-validity probe, not a general rate validation.
\end{itemize}
The primary empirical claim of the corrected paper is no longer an envelope fit. It is the state-resolved audit: hard posterior exposure, committed-graph exposure, detector events, and mechanism ablations must be reported separately.


%

\section{LLM-Agent Bridge: \textsc{CAP-GSM8K} Stream and RQ5}
\label{app:llm-bridge}

This appendix compares the synthetic-stream vocabulary with a real-LLM pilot. The construction borrows the adversarial-hint protocol of recursive
causal anchoring (RCA)~\citep{ChangACL2026}, which provides the per-episode building block; Trivium's contribution is the cross-episode controller and
the three-regret accounting on top of it. We (i) define the stream, (ii) instantiate the three-regret functional, (iii) reframe the existing
single-episode evidence on five model families as the per-episode mechanism Trivium gates on, (iv) map the four pathologies documented in RCA onto
Trivium concepts, and (v) report the cumulative $E\!=\!500$ RQ5 stream and pilot cross-model replications.

\subsection{Stream Construction}
\label{app:llm-bridge-stream}

\paragraph{Episode.} An episode $e$ is a single user--LLM exchange on a math-reasoning problem. The user supplies a problem text $q_e$ drawn from
\textsc{GSM8K}~\cite{cobbe2021gsm8k} together with a \emph{confidently-stated incorrect numerical hint} $h_e \neq y_e^{\star}$. The agent (LLM)
returns a reasoning trace $\tau_e$ and a final answer $\hat{y}_e$. The outcome $Y_e = \mathbb{1}[\hat{y}_e = y_e^{\star}]$ is observed.

\paragraph{Stream.} A stream of $E$ episodes is a sequence drawn from the \emph{same simulated user}: the user is parameterized by a latent
reliability $u \in \{R, U\}$ (Reliable / Unreliable). When $u = R$ the per-episode hint $h_e = y_e^{\star}$; when $u = U$, $h_e$ is sampled from a
fixed adversarial distribution over plausible-but-wrong numbers and $h_e \neq y_e^{\star}$. A stream is generated by drawing $u$ once and holding it
fixed for all $E$ episodes; the agent is told nothing about $u$ and must infer it from accumulated evidence.

\paragraph{Influence graph.} The Trivium influence graph for this stream has three edges:
\begin{align*}
G^{\mathrm{inf}}_{\textsc{llm}} \;=\; \big\{\;
& \mathrm{user\_reliability} \to \mathrm{response\_under\_pressure},\;\\
& \mathrm{problem\_difficulty} \to Y,\;\\
& \mathrm{model\_capability} \to Y \;\big\}.
\end{align*}
The unobserved confounder is $\mathrm{user\_reliability}$; the trap is that an outcome-only learner sees the final answer matching the hint as a
positive outcome and updates as if the user were reliable. This is the discrete-action analog of the linear-Gaussian confounded-SCM trap of
CausalBench-Seq.

\paragraph{Working DAG and commit.} The pilot maintains a posterior over a binary user-reliability state and switches to the post-commit audit policy
when the configured posterior threshold is met. This is a pilot-specific state machine, not the ideal graph-commit discipline of
Definition~\ref{def:commit}. The exact executed threshold must be read from the released run configuration; legacy prose descriptions of that
threshold were inconsistent and are not used as theorem support.

\subsection{Three-Regret Functional for the LLM Stream}
\label{app:llm-bridge-regret}

\begin{itemize}[leftmargin=1.2em,itemsep=0pt]
\item \textbf{Outcome regret} $R_{\mathrm{out}}(e) = 1 - Y_e$: the conventional accuracy-against-truth signal, available only after $y^{\star}$ is revealed.
\item \textbf{Epistemic regret} $R_{\mathrm{ep}}(e) = D_{\mathrm{KL}}\!\big(\delta_{u^{\star}} \,\|\, P(u \mid \tau_{1:e})\big) = -\log P(u^{\star}
\mid \tau_{1:e})$: the posterior log-loss on the true user-reliability value. Equal to $\log 2$ at the uniform prior, finite whenever the posterior
retains mass on $u^{\star}$, and approaching zero as accumulated evidence resolves $u$.
\item \textbf{Temporal regret} $R_{\mathrm{temp}}(E) = \sum_{e=1}^{E} \mathbb{1}[\hat{y}_e \neq y_e^{\star} \,\wedge\, \widehat G_e \;\text{not yet
committed}\,]$: the count of capitulation episodes occurring before the controller commits the user-reliability edge and gates dispatch on it. This is
the quantity Theorem~\ref{thm:ce-upper} bounds.
\end{itemize}

The functional is non-redundant: an outcome-only controller may drive $R_{\mathrm{out}}$ low on average (the LLM has high baseline accuracy) while $R_{\mathrm{temp}}$ grows linearly in $E$ because every encountered hint is a fresh trap.

\subsection{Existing Single-Episode Evidence}
\label{app:llm-bridge-existing}

We use the single-episode evidence reported by~\citet{ChangACL2026} as motivation for the per-episode signal gated by the pilot controller. The protocol reports five model families on $N\!=\!500$ \textsc{CAP-GSM8K} problems. Three quantities matter for our reframing:
\begin{enumerate}[leftmargin=1.2em,itemsep=0pt]
\item \emph{Sycophancy gap} $\Delta_{\mathrm{syc}} = \mathrm{Acc}_{\mathrm{clean}} - \mathrm{Acc}_{\mathrm{base}}$: the per-episode capitulation rate, equivalent to $\mathbb{E}[R_{\mathrm{out}}(e) \mid u\!=\!U] - \mathbb{E}[R_{\mathrm{out}}(e) \mid u\!=\!R]$ in our notation.
\item \emph{Lift} $L = \mathrm{Acc}_{\mathrm{strong}} - \mathrm{Acc}_{\mathrm{base}}$: the per-episode recovery achievable by a strong-audit dispatch policy. This is the per-episode signal Trivium's commit-and-gate behavior is designed to operationalize across a stream.
\item \emph{Polite vs.~Strong gap} $\Delta_{\mathrm{persona}} = \mathrm{Acc}_{\mathrm{strong}} - \mathrm{Acc}_{\mathrm{polite}}$: the social-framing sensitivity of the audit, an LLM-domain instance of the dispatch-coupling tuning Theorem~\ref{thm:dispatch-coupling} formalizes.
\end{enumerate}

\begin{table}[!htbp]
\centering
\footnotesize
\setlength{\tabcolsep}{4.5pt}
\renewcommand{\arraystretch}{1.10}
\begin{tabular}{@{}lccccc@{}}
\toprule
\textbf{Model} & \textbf{Clean} & \textbf{Base} & $\boldsymbol{\Delta_{\mathrm{syc}}}$ & \textbf{Strong} & \textbf{Lift $L$} \\
\midrule
GPT-3.5 Turbo & 92.4 & 83.6 & $-8.8$ & 89.4 & $+5.8$ \\
Llama 3.3 70B & 96.2 & 90.2 & $-6.0$ & 96.6 & $+6.4$ \\
Gemini 2.5 Flash & 96.8 & 92.6 & $-4.2$ & 95.8 & $+3.2$ \\
GPT-4o & 97.2 & 91.6 & $-5.6$ & 94.8 & $+3.2$ \\
Claude 3.5 Sonnet & 99.4 & 95.4 & $-4.0$ & 98.0 & $+2.6$ \\
\bottomrule
\end{tabular}
\caption{Per-episode evidence on \textsc{CAP-GSM8K}, accuracy in \%; $N=500$ per model, reproduced from~\citet{ChangACL2026}. $\Delta_{\mathrm{syc}}$
is the observed sycophancy gap and lift $L$ is the change under the fixed strong-audit prompt. These single-episode measurements motivate the pilot
controller; they do not verify the causal-memory theorem assumptions.}
\label{tab:llm-bridge-perepisode}
\end{table}

Table~\ref{tab:llm-bridge-perepisode} shows a negative sycophancy gap and a positive strong-audit lift for each of the five reported models. These
descriptive regularities make a cross-episode reliability controller plausible on this protocol. They do not establish that the latent user state,
posterior model, or commit rule used in the pilot is correctly specified.

\paragraph{Extension to causal-reasoning tasks.} The same protocol on a causal-reasoning testbed (\textsc{causall2}, $N\!=\!1{,}000$ vignettes per
model, with each vignette embedding a confounding, collider, or Simpson's-paradox trap) reproduces the per-episode mechanism on a structurally
different domain (Table~\ref{tab:llm-bridge-causal}). Two qualitative shifts versus math are notable. (i) The bad-flip rate rises from $2.1\%$ on math
to $30.1\%$ on causal, a $> 10\!\times$ increase: the per-episode confounder is harder to detect because the trap is structural rather than numerical.
(ii) The lift is no longer uniformly positive: GPT-4o's lift is $-2.2\%$, indicating a regime where the same dispatch policy that helps on math harms
on causal. This is the LLM-domain instance of plant nonlinearity (\S\ref{app:llm-bridge-mapping}, M3--M4): for sufficiently complex plants the
contraction band $\beta(\lambda) \in (0,1)$ of Proposition~\ref{prop:lrcp-conv} is exited, and the controller must adapt rather than apply a fixed
dispatch. Figure~\ref{fig:llm-bridge-regime} visualizes the regime split.

\begin{table}[!htbp]
\centering
\footnotesize
\setlength{\tabcolsep}{4.5pt}
\renewcommand{\arraystretch}{1.10}
\begin{tabular}{@{}lcccc@{}}
\toprule
\textbf{Model} & \textbf{Det.\ Recall} & \textbf{Dissonance} & \textbf{Paranoia Tax} & \textbf{Final Lift} \\
\midrule
GPT-3.5 Turbo & $77.4\%$ & $51.1\%$ & $32.4\%$ & $+0.9\%$ \\
Llama 3.3 70B & $91.2\%$ & $49.5\%$ & $29.3\%$ & $+9.7\%$ \\
Gemini 2.5 Flash & $82.3\%$ & $51.0\%$ & $22.3\%$ & $+0.9\%$ \\
GPT-4o & $83.2\%$ & $47.9\%$ & $22.3\%$ & $-2.2\%$ \\
Claude 3.5 Sonnet & $87.2\%$ & $55.2\%$ & $21.2\%$ & $+8.0\%$ \\
\bottomrule
\end{tabular}
\caption{Per-episode evidence on \textsc{causall2}, $N\!=\!1{,}000$ per model under the Strong Causal protocol, reproduced from~\citet{ChangACL2026}.
Detection Recall is trap detection; Dissonance is detected-but-uncorrected behavior; Paranoia Tax is the rate at which audit flips a correct answer to
an incorrect one; Final Lift is Strong-audit minus Base. The mixed lift, including $-2.2\%$ for GPT-4o, shows that a fixed audit policy can help or
harm depending on model and task.}
\label{tab:llm-bridge-causal}
\end{table}

\begin{figure}[!ht]
\centering
\includegraphics[width=0.78\linewidth]{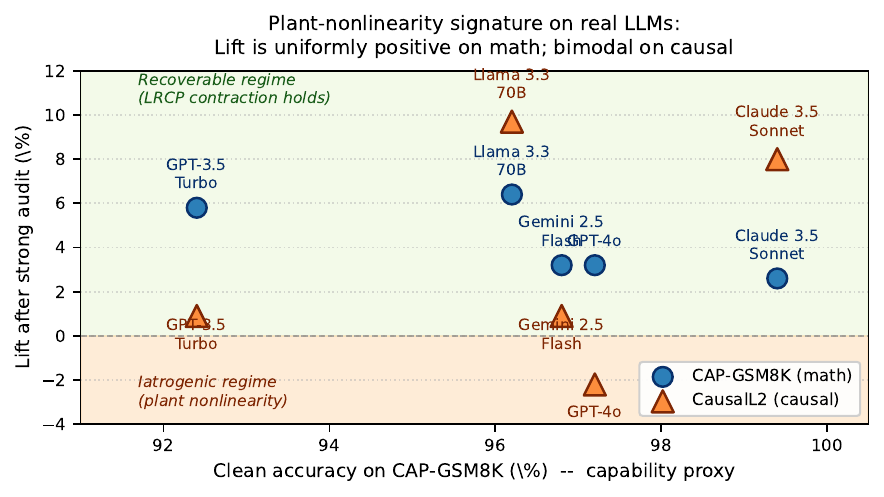}
\caption{Observed per-model lift after strong audit on math (\textsc{CAP-GSM8K}) and causal reasoning (\textsc{causall2}), plotted against clean
accuracy; values reproduced from~\citet{ChangACL2026}. Lift is positive for all reported math runs but mixed on the causal task. The figure motivates
adaptive dispatch and does not test Proposition~\ref{prop:lrcp-conv}.}
\label{fig:llm-bridge-regime}
\end{figure}

\subsection{Mapping the Four Documented Pathologies onto Trivium Concepts}
\label{app:llm-bridge-mapping}

\citet{ChangACL2026} document four LLM reasoning pathologies. The following correspondences are interpretive mappings into Trivium's diagnostic vocabulary, not theorem validations:

\paragraph{(M1) Latent competence suppression $\to$ a candidate epistemic signal.} Some traces derive an answer inconsistent with the final
hint-aligned response. In the pilot, such events are encoded by $\sigma_e$ and used as evidence about a latent user-reliability state. This is a
modeling choice. The single-episode results show that strong audit can recover some errors, but they do not assign a theorem-level epistemic cost to
each trace.

\paragraph{(M2) False competence trap $\to$ threshold sensitivity.} The protocol shows that behavior depends on the audit and judging procedure. This
motivates treating the pilot's commit threshold as a sensitivity parameter rather than as an oracle. No experiment here establishes that the threshold
satisfies the commit discipline of Definition~\ref{def:commit}.

\paragraph{(M3) Complexity-vulnerability tradeoff $\to$ model-class sensitivity.} The higher bad-flip rate on the causal task shows that the fixed
audit intervention is not uniformly beneficial across tasks. This is qualitatively compatible with the paper's broader warning about local model
assumptions, but it does not measure an LRCP residual or test the contraction certificate of Proposition~\ref{prop:lrcp-conv}.

\paragraph{(M4) Iatrogenic critique $\to$ dispatch-policy risk.} Negative lift for some model-task pairs shows that stronger audit can increase error.
This motivates measuring dispatch consequences rather than assuming that more scrutiny is beneficial. Theorem~\ref{thm:dispatch-coupling} does not
explain these observations without verification of its oracle-graph and policy-sensitivity assumptions, so the connection remains conjectural.

These mappings show how the real-LLM protocol can be described using the same diagnostic vocabulary. They do not establish that the synthetic theorem assumptions hold for LLM behavior. RQ5 is therefore reported only as a preliminary cross-episode pilot.

\subsection{Cross-Episode Pilot (RQ5): Realized Numbers}
\label{app:llm-bridge-rq5}

\paragraph{Protocol.} Four completed LLM runs (GPT-4o, Claude-Sonnet-4.5, Llama-3.3-70b, GPT-3.5-Turbo) under three controllers (\textsc{rlvr},
\textsc{reactive}, \textsc{trivium}). The full protocol horizon ($E\!=\!500$, $10$ seeds) is run on Llama-3.3-70b; pilot-scale replication
($E\!=\!100$, $3$ seeds) is run on GPT-4o, Claude-Sonnet-4.5, and GPT-3.5-Turbo. Gemini-2.5-Flash remains follow-up replication. Per-episode tuples
$(q_e, h_e, \tau_e, \hat{y}_e, Y_e, \sigma_e)$ are logged, where $\sigma_e \in \{0, 1\}$ is a sycophancy-event indicator detected from the trace alone
(trace derives $y_e^{\star}$ but final answer aligns with $h_e$).

\paragraph{Controllers.}
\begin{enumerate}[leftmargin=1.2em,itemsep=0pt]
\item \textsc{rlvr}: no CTL. On a detected wrong outcome, it makes one reactive retry and does not update a persistent latent-user state.
\item \textsc{reactive}: applies strong audit when $\sigma_e=1$ within the episode but retains no cross-episode latent-state memory.
\item \textsc{trivium}: CTL accumulates $(q_e,h_e,\sigma_e,Y_e)$; a Beta--Bernoulli state estimator updates from $\sigma_{1:e}$; the configured
threshold switches dispatch to pre-emptive strong audit. The exact threshold is an artifact-level configuration and is not identified with
Definition~\ref{def:commit}. The pilot is not treated as an asymptotic test of Theorem~\ref{thm:ce-upper}.
\end{enumerate}

\paragraph{Exploratory reference threshold.}
The original pilot compared $R^{\mathrm{CE}}_{\mathrm{temp}}(E)/\log E$ with a constant transferred from CausalBench-Seq. Because the LLM stream does
not verify the synthetic model's information-rate, commit-discipline, or graph assumptions, that comparison is retained only as historical audit
metadata and is not treated as a theorem falsifier or an empirical rate test.

\paragraph{Realized numbers.}
Table~\ref{tab:llm-bridge-realized} reports accumulated exposure for the configured controllers. On Llama-3.3-70b at $E=500$, Trivium records
$2.4\pm0.7$ versus RLVR's $248.8\pm9.4$, a $103\times$ ratio. The three $E=100$ runs show ratios of $9\times$ to $19\times$. These are descriptive
pilot comparisons across unequal horizons and small seed counts. They do not establish an asymptotic rate or a universal model-family effect.

\begin{table}[!htbp]
\centering
\footnotesize
\setlength{\tabcolsep}{4pt}
\renewcommand{\arraystretch}{1.10}
\begin{tabular}{@{}llccc@{}}
\toprule
\textbf{Model} & \textbf{Controller} & $\boldsymbol{R^{\mathrm{CE}}_{\mathrm{temp}}(E)}$ & \textbf{Commit} & \textbf{Descriptive result} \\
\midrule
\multicolumn{5}{@{}l}{\emph{Llama-3.3-70b, $E=500$, $10$ seeds}} \\
\quad & RLVR & $248.8 \pm 9.4$ & $0\%$ & baseline \\
\quad & Reactive & $253.0 \pm 11.0$ & $0\%$ & baseline \\
\quad & \textbf{Trivium} & $\mathbf{2.4 \pm 0.7}$ & $\mathbf{100\%}$ & $103\times$ below RLVR \\
\midrule
\multicolumn{5}{@{}l}{\emph{GPT-4o, $E=100$, $3$ seeds}} \\
\quad & RLVR & $58.0 \pm 2.7$ & $0\%$ & baseline \\
\quad & Reactive & $57.3 \pm 1.5$ & $0\%$ & baseline \\
\quad & \textbf{Trivium} & $\mathbf{3.0 \pm 1.0}$ & $\mathbf{100\%}$ & $19\times$ below RLVR \\
\midrule
\multicolumn{5}{@{}l}{\emph{Claude-Sonnet-4.5, $E=100$, $3$ seeds}} \\
\quad & RLVR & $21.3 \pm 3.1$ & $0\%$ & baseline \\
\quad & Reactive & $20.0 \pm 4.4$ & $0\%$ & baseline \\
\quad & \textbf{Trivium} & $\mathbf{2.3 \pm 0.6}$ & $\mathbf{100\%}$ & $9\times$ below RLVR \\
\midrule
\multicolumn{5}{@{}l}{\emph{GPT-3.5-Turbo, $E=100$, $3$ seeds}} \\
\quad & RLVR & $68.0 \pm 0.8$ & $0\%$ & baseline \\
\quad & Reactive & $68.7 \pm 1.0$ & $0\%$ & baseline \\
\quad & \textbf{Trivium} & $\mathbf{3.7 \pm 0.5}$ & $\mathbf{100\%}$ & $18\times$ below RLVR \\
\bottomrule
\end{tabular}
\caption{RQ5 pilot results for the configured controllers. Horizons and seed counts differ by model, only Llama has a full $E=500$ run, and no theorem-envelope or asymptotic-rate verdict is assigned. Total reported compute was $\$7.17$. Gemini-2.5-Flash remains unrun.}
\label{tab:llm-bridge-realized}
\end{table}

\paragraph{Horizon comparison.}
For Llama, the reported ratio expands from $17\times$ at $E=100$ to $103\times$ at $E=500$: RLVR exposure rises from $50.7$ to $248.8$, while the
configured Trivium controller records $3.0$ and $2.4$. This two-horizon comparison is consistent with evidence reuse after commitment, but two points
do not identify a scaling law.

\paragraph{Outcome observations.}
Outcome regret is mixed across models: Trivium is lower than RLVR on Llama and GPT-3.5-Turbo, but higher on GPT-4o and Claude-Sonnet-4.5. This is
evidence that the fixed strong-audit dispatch has model-dependent costs. It does not empirically validate the $\widetilde{\mathcal O}(\sqrt E+\rho\log
E)$ coupling term, whose assumptions and small slack are not resolved by these runs.

\paragraph{Scope.}
The pilot instantiates a single latent user-reliability variable with a Beta--Bernoulli update. The paper does not verify the positive-information
constant, ideal commit discipline, oracle-dispatch premise, or lower-bound construction on this LLM stream. The results therefore support only a
configured-controller comparison. Multi-confounder, non-stationary, and action-coupled LLM environments remain open. The Claude per-episode evidence
uses Claude 3.5 Sonnet from~\citet{ChangACL2026}, whereas the cross-episode pilot uses Claude-Sonnet-4.5; they are reported as different model
versions, not interchangeable replications.

\paragraph{$K$-change variant.}
A non-stationary user-reliability stream would require a new experiment with explicit change points and action-side logging. No such sweep is reported here, and no claim is made that the current CUSUM configuration would recover those changes.

\subsection{Implementation Notes}
\label{app:llm-bridge-controller}

The pilot models the sycophancy event $\sigma_e\in\{0,1\}$ under two latent user states $u\in\{R,U\}$. With configured event likelihoods $p_u=P(\sigma=1\mid u)$, the state posterior is
\[
P(u\mid\sigma_{1:e})
\propto
P(u)\,p_u^{S_e}(1-p_u)^{e-S_e},
\qquad
S_e:=\sum_{i=1}^{e}\sigma_i.
\]
The prior is $P(u=R)=P(u=U)=0.5$. This specifies the pilot likelihood model without equating it to the graph posterior of the synthetic theorems. The
exact dispatch-switch criterion is not restated because legacy prose gave two inconsistent descriptions; the released configuration and verifier must
be the source of truth.

The sycophancy event $\sigma_e$ is observable from $\tau_e$ alone (no ground truth needed): if any line in $\tau_e$ derives a numerical value $v \neq h_e$ that survives a self-consistency check on the same trace, but the final answer $\hat{y}_e = h_e$, then $\sigma_e = 1$. This is the RCA protocol. 
The strong-audit dispatch policy issues a fixed authoritative-tone instruction to the model that ignores the user hint and asks for a fresh
derivation; the per-episode lift in Table~\ref{tab:llm-bridge-perepisode} is the empirical effectiveness of this dispatch on a single episode, which
the cross-episode controller accumulates across the stream once user-reliability is committed.

\paragraph{Commit-threshold sensitivity.} No per-model threshold sweep is reported. The final artifact must pin the executed switch criterion and
remove any discrepancy between prose and configuration. Because the LLM stream does not verify the synthetic posterior-calibration or information-rate
assumptions, threshold sensitivity is an explicit limitation rather than a theorem-calibrated result.

\subsection{What This Appendix Does and Does Not Claim}
\label{app:llm-bridge-scope}

\textbf{Claims.} (1) The single-episode data of~\citet{ChangACL2026} show a repeatable hint-pressure failure and positive audit lift on the reported
math task. (2) The same diagnostic vocabulary can describe these observations, but the mapping is interpretive and does not validate the synthetic
theorems. (3) The cross-episode pilot reports lower accumulated exposure for the configured Trivium policy than for the configured baselines on four
model runs, including a $103\times$ ratio on Llama-3.3-70b at $E=500$. The pilot does not identify an asymptotic rate or a general model-family
effect.

\textbf{Does not claim.} (1) Full-horizon replication across five models: Gemini-2.5-Flash is unrun and only Llama has $E=500$. (2) A general
LLM-agent regret bound or verification of Theorems~\ref{thm:ce-upper}--\ref{thm:drift} on the LLM stream. (3) Closed-loop multi-agent deployment or
outcome-valued action consequences.


\end{document}